%% file: paper.tex
\documentclass[a4paper]{article}

\usepackage{url}
\usepackage{graphicx} 
\usepackage{amsmath} 
\usepackage{amssymb}  
\usepackage{amsthm}
\usepackage{multirow}
\usepackage{algorithm}
\usepackage{algpseudocode}
\usepackage{xcolor}
\usepackage{hyperref}
\usepackage{subcaption}
\usepackage{wrapfig}
\usepackage{cancel}
\usepackage[verbose, hyperpageref]{backref}

\usepackage[includeheadfoot]{geometry}
\usepackage[title]{appendix}
\usepackage[font=normalsize,labelfont=bf]{caption}
\definecolor{hibiscus}{RGB}{176,48,96}
\definecolor{emeraldgreen}{RGB}{0,155,119}
\definecolor{violet}{RGB}{238,130,238}
\definecolor{inkscapePurple}{RGB}{128,0,128}

\DeclareMathOperator*{\argmax}{arg\,max}

\newcommand{\bydef}{\triangleq}
\newcommand{\probd}{\mathbb{P}}
\newcommand{\prob}{\mathrm{P}} 

\newcommand{\lb}{\mathrm{LB}}
\newcommand{\ub}{\mathrm{UB}}

\renewcommand{\qed}{$\blacksquare$}
\theoremstyle{plain}
\newtheorem{thm}{\protect\theoremname}
\theoremstyle{remark}

\theoremstyle{plain}
\newtheorem{lem}{\protect\lemmaname}
\theoremstyle{definition}

\providecommand{\claimname}{Claim}
\providecommand{\definitionname}{Definition}
\providecommand{\lemmaname}{Lemma}
\providecommand{\theoremname}{Theorem}

\begin{document}

\title{\LARGE \bf Risk Aware Adaptive Belief-dependent Probabilistically Constrained Continuous POMDP Planning\thanks{This work was partially supported by the Israel Science Foundation (ISF).}}

\author{Andrey Zhitnikov$^{1}$, \quad Vadim Indelman$^2$ \\
	$^1$Technion Autonomous Systems Program (TASP) \\
	$^2$Department of Aerospace Engineering\\
	Technion - Israel Institute of Technology, Haifa 32000, Israel\\
	\small{\texttt{andreyz@campus.technion.ac.il, vadim.indelman@technion.ac.il}}}
	
\maketitle

\begin{abstract}
	Although risk awareness is fundamental to an online operating agent, it has received less attention in the challenging continuous domain and under partial observability. This paper presents a novel formulation and solution for risk-averse belief-dependent probabilistically constrained continuous POMDP. We tackle a demanding setting of belief-dependent reward and constraint operators.   The probabilistic confidence parameter makes our formulation genuinely risk-averse and much more flexible than the state-of-the-art chance constraint. 
	Our rigorous analysis shows that in the stiffest probabilistic confidence case, our formulation is very close to chance constraint. However, our probabilistic formulation allows much faster and more accurate adaptive acceptance or pruning of actions fulfilling or violating the constraint. In addition, with an arbitrary confidence parameter, we did not find any analogs to our approach. We present algorithms for the solution of our formulation in continuous domains. We also uplift the chance-constrained approach to continuous environments using importance sampling. Moreover, all our presented algorithms can be used with parametric and nonparametric beliefs represented by particles. Last but not least, we contribute, rigorously analyze and simulate an approximation of chance-constrained continuous POMDP. The simulations demonstrate that our algorithms exhibit unprecedented celerity compared to the baseline, with the same performance in terms of collisions.   
\end{abstract}
\newpage
\tableofcontents
\newpage
\section{Introduction}
\begin{wrapfigure}{r}{0.5\textwidth}
	\centering 
	\includegraphics[width=0.48\textwidth]{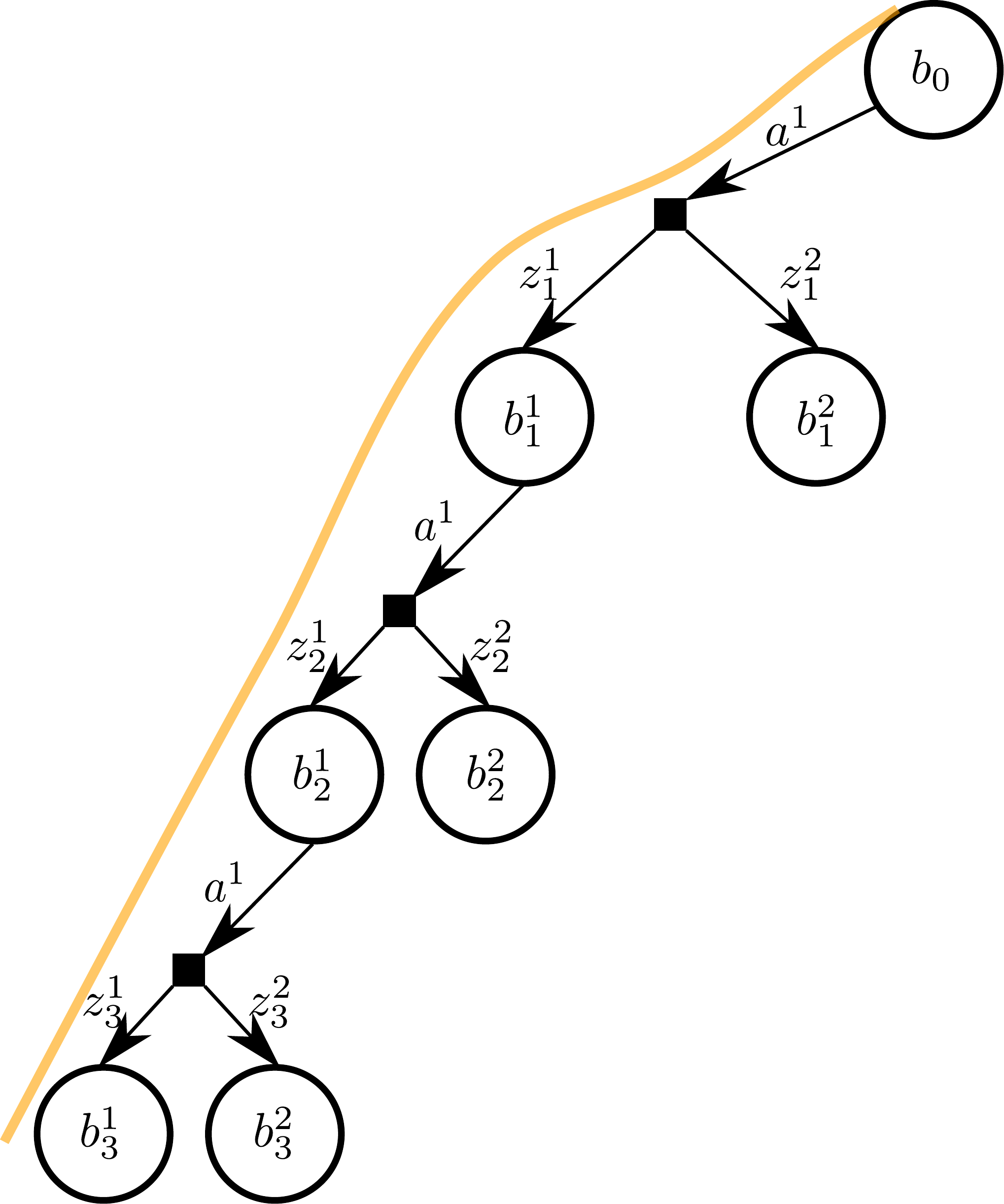}
	\caption{ Illustration  of a belief tree. The yellow lace depicts a sequence of beliefs $(b_0, b_1^1,b_2^1, b_3^1)$ generated by a corresponding sequence of  observations under some policy (Section \ref{sec:Framework}). }
	\label{fig:lace}
\end{wrapfigure} 
Decision making under uncertainty in partially observable domains is a key capability for reliable autonomous agents. Commonly, the basis of the state-of-the-art algorithms for decision making under uncertainty is the Partially Observable Markov Decision Process (POMDP). The robot does not have access to its state. Instead, it maintains a distribution, named the belief, over the state given all its current information, namely, the history of actions and the observations alongside the prior belief. The decision maker shall maintain and reason about the evolution of the belief within the planning phase. At the same time, the robot's online goal is to find an optimal action for its current belief. Unfortunately, an exact solution of POMDP is unfeasible \cite{Papadimitriou87math}.
A critical limitation of the classical POMDP formulation is the assumption that the belief-dependent reward is nothing more than the expectation of state-dependent reward with respect to belief \cite{Kochenderfer22book}. Another limiting assumption in many state-of-the-art algorithms is the discrete domain, e.g., discrete state and observation spaces \cite{Silver10nips}, \cite{Ye17jair}. In contrast, we tackle the continuous domain in terms of state and observation spaces.

Augmenting POMDP with general belief-dependent rewards is a long-standing problem. Unravelling it would allow information theoretic rewards, which are extremely important in numerous problems in Artificial Intelligence (AI) and Robotics, such as autonomous exploration, Belief Space Planning (BSP) \cite{Indelman15ijrr}, and active Simultaneous Localization and Mapping (SLAM) \cite{Placed22arxiv}. The belief-dependent reward formulation is known as $\rho$-POMDP \cite{Araya10nips}, \cite{Fehr18nips}.  Earlier techniques focused on offline solvers and extended $\alpha$-vectors approach to piecewise linear and convex \cite{Araya10nips}, \cite{Dressel17icaps} or Lipschitz-continuous rewards \cite{Fehr18nips}. These extended solvers are also limited to discrete domains in terms of states and observations.

Continuous spaces and general belief-dependent rewards render many off-the-shelf POMDP solvers not applicable. Another way to incorporate general belief-dependent rewards is to reformulate POMDP as Belief-MDP (BMDP) and use more recent online solvers designed for MDP. These algorithms are suitable for continuous domains and challenging nonparametric beliefs represented by particles.
Seminal approaches in this category are Sparse Sampling (SS) \cite{Kearns02jml}, Monte Carlo Tree Search (MCTS) \cite{Sunberg18icaps}, and its efficient, simplified variant \cite{Sztyglic21arxiv_b}. Such an MCTS running on BMDP is called a Particle Filter Tree with Double Progressive Widening (PFT-DPW) \cite{Sunberg18icaps}. 

The described limitation of discrete  state and observation spaces also applies to recently appeared chance constrained approaches, gaining attention. The motivation to add the chance constraints is to introduce the notion of \emph{risk} into the problem. Initially, the planning community focused on collision avoidance, formulating it as a chance constraint. For example, \cite{Bry11icra} is limited to Gaussian parameterized beliefs. Applying Kalman filter on the linearized system, the \cite{Bry11icra} tracks the distribution (modeled as a Gussian) over possible state estimates alongside the (Gaussian) posterior beliefs. Each posterior, conditioned on the history of actions and observations, is then regarded as conditional on the sole state estimate and multiplied by the probability density of the state estimate. This way, the joint distribution of state and state estimates is obtained. The chance constraint is calculated on the marginal of such a joint distribution which is the probability of the state propagated solely with actions and without the observations. In this way, on the selected by \cite{Bry11icra} optimal action sequence, the probability of collision is the probability of the robot's final state obtained solely with actions and without the observations.

 Such a formulation does not apply to general beliefs represented by particles, nor is it a truly risk-averse constraint. We extensively debate this claim in the paper.

More recent works examine a discrete domain and a belief-dependent constraint being the first moment of the state-dependent constraint \cite{Lee18nips}. Moreover, \cite{Lee18nips} provides only a local solution. Another line of work considers chance constraints \cite{Santana16aaai}. The paper  \cite{Santana16aaai} introduces the algorithm RAO* which uses admissible heuristics for the action value function ($Q$-function) in the belief space. This aspect is problematic with {\bf general belief-dependent rewards}. However, this is not a  main point of this work.

Moreover,  the RAO* algorithm prunes not feasible actions using a necessary condition of the feasibility of chance constraint  (See Appendix \ref{sec:ExecutionRiskBound}) .  In other words, the chance constraint may be violated, but the action has not been pruned. In fact, this is one of our main points in this paper. Since the condition is only necessary, after pruning it is still required to verify feasibility of each kept action. This significantly complicates the algorithmic and inflict unnecessary computational burden. It would be much easier if after pruning we would know that remained actions satisfy the constraint.   

Similar to the situation with belief-dependent rewards reformulation POMDP as Belief-MDP (BMDP) can possibly allow employing approaches designed for Probabilistically constrained MDP \cite{Han22arxiv}, \cite{Frey20arxiv}.  However, the theory presented in these papers does not apply to parametric nor nonparametric Belief-MDP due to various assumptions made by the authors.  This is one of the gaps we aim to fill in this work. 

In this paper, we formulate a novel two-staged approach to enhance continuous belief-dependent POMDP with belief-dependent probabilistic constraint. Typically algorithms designed for general beliefs represent the belief as a set of particles and use a particle filter \cite{Thrun05book} for nonparametric Bayesian updates. 
In this work, we assume the setting of nonparametric beliefs, although our formulations also support a parametric setting.

\subsection{Comparison to Chance Constraint}
Most works that consider constrained online planning under uncertainty utilize the chance-constrained formulation. This formulation is regarded as state-of-the-art. The chance constraints are formulated with respect to agent trajectories, such that only the safe state spaces are pushed forward into the future. By design, the chance constraint is averaging the safe trajectories mixing together the trajectories originating from different future observations (posterior beliefs). Due to this averaging, the chance constraint is oblivious to the distribution of future observations (posterior beliefs). This drawback is present in \cite{Bry11icra} but fixed to some extent in \cite{Santana16aaai} by enforcing the chance constraint starting from each nonterminal belief in the belief tree. Therefore the Chance Constraint is accessing the beliefs in the tree solely with a resolution of subtrees (Fig.~\ref{fig:ChanceConstraint}).  
We redevelop the Chance Constraint from the BMDP point of view and prove this claim.

Moreover, since the origin of chance constrained formulation are safe trajectories, it needs to be clarified how to extend chance constraint to a general belief-dependent operator. Last but not least, by examining the chance constraint on the level of posterior beliefs, we observe that only the safe portion of the belief is pushed forward to the future time with action and the observation. In other words, we have different distributions of future observations for rewards and the chance constraint.   This fact significantly complicates the algorithmics.

In contrast, we formulate our probabilistic constraint on the level of posterior beliefs. This allows us to utilize general belief-dependent operators. However, in our setting, the unsafe portion of the belief is also updated with action and observation, such that unsafe states can be pushed forward if such an action is not discarded. Instead of a safe trajectory of the future states, we have a safe trajectory of future beliefs. This way, we have an identical distribution of future observations for belief-dependent rewards and constraints. As we further see, this is highly beneficial in terms of time efficiency.

\subsection{Contributions}
The previous discussion leads us to the contributions of this work. 
\begin{itemize}
	\item  Firstly,  we formulate a risk-averse belief-dependent probabilstically constrained continuous POMDP. Averaging the state-dependent reward/constraint to obtain the belief-dependent reward/constraint is a severe hindrance that we relax.  We are unaware of prior works addressing POMDP with risk-averse or belief-dependent constraints (even with expectation). Our constraint is more general compared to previous approaches:  
	The proposed constraint is probabilistic and {\bf future distribution aware}, whereas the state-of-the-art constrained formulations devise the constraint as an expectation with respect to observations. In particular, our probabilistic belief-dependent constraint supports risk-averse operators, such as Conditional Value at Risk (CVaR),  and leads to a novel safety constraint formulation.   
	\item Secondly, we contribute algorithms for online solution of Probabilisticaly Constrained belief dependent POMDP in continuous domains. Our algorithms are adaptive given the budget of observation laces to expand in the belief tree. In other words, we provide a way to guide the belief tree construction while planning. Our framework is universal for challenging continuous domains and can be applied in nonparametric and parametric settings.
	\item   On top of our probabilistic formulation, we contribute a novel, efficient actions-pruning mechanism.  Previous pruning techniques proposed by \cite{Santana16aaai} constitute only a necessary condition such that it is possible that after  pruning, actions violating the chance constraint are kept in the belief tree.  Therefore, in these techniques, the feasibility of chance constraint has to still be inspected for each action after the pruning. On the contrary,  our pruning is necessary and sufficient. No addional checks needed after the pruning is complete. 
	\item We present a rigorous analysis of our probabilistic formulation versus chance constrained. 
	\item We uplift a chance constrained solver to continuous domains and general belief dependent rewards through Importance Sampling.
	\item We suggest an approximation to continuous chance constrained problem  accompanied by guarantees under rather mild assumptions.
	\item We present a detailed and comprehensive study of nonparametric collision avoidance. 
\end{itemize}
\subsection{Paper Layout}
The rest of this paper is organized as follows. We start from preliminaries in section \ref{sec:Preliminaries}. We then define our novel framework in section \ref{sec:Framework}, and give relevant examples of possible constraints in section \ref{sec:PossibleConstraints}. Next, in section \ref{sec:OuterConstrEvalAndAlgorithms} we adaptively evaluate the probabilistic constraint while constructing the belief tree and present an online algorithms (section \ref{sec:PCSS}) for our novel formulation.  
Further, we rigorously analyze the conventional chance constraint in section \ref{sec:RelationToChance} and compare to our probabilistic .  Finally, in section \ref{sec:ChanceConstrainedContPOMDP} we introduce the state-of-the-art online solver for chance-constrained POMDP in continuous setting augmented with belief dependent rewards. Section \ref{sec:ChanceApprox} devoted to its efficient approximation. 
Eventually, section \ref{sec:Results} show simulations and results. The conclusions and final remarks are presented in section \ref{sec:Conclusions}. To allow fluid reading, we placed the proofs for all theorems and lemmas, and additional in-depth discussions in the appendix.
\section{Preliminaries} \label{sec:Preliminaries}
The $\rho$-POMDP is a tuple $(\mathcal{X}, \mathcal{A}, \mathcal{Z}, T, O, \rho, \gamma, b_0)$ where $\mathcal{X}, \mathcal{A}, \mathcal{Z}$ denote state, action,
and observation spaces with $x \! \in \! \mathcal{X}$, $a \! \in \! \mathcal{A}$, $z \! \in  \! \mathcal{Z}$ the momentary state, action, and observation, respectively. $T(x' , a, x)  =  \probd_T(x' | x, a)$ is a stochastic transition model from the past state $x$ to the subsequent $x'$ through action
$a$, $O(z , x)  =  \probd_Z(z|x)$ is the stochastic observation model, $\gamma \in (0, 1]$ is the discount factor, $b_0$ is the belief over the
initial state (prior), and $\rho$ is the belief-dependent reward operator. Let $h_k$ be a history, of actions and observations alongside the prior belief, obtained 
by the agent up to time instance $k$. The posterior belief $b_k$ is a shorthand for  
the probability density function of the state given all information up to current time index $b_k(x_k) \bydef  p(x_k |h_k)$.
The policy is a, indexed by the time instances, mapping from belief to action to be executed $\pi_k : \mathcal{B} \mapsto \mathcal{A}$, where $\mathcal{B}$ is the space of all the beliefs taken into account in the problem. The policy for $L$ consecutive steps ahead is denoted by $\pi_{k:k+L-1}$. Sometimes we will omit the time indices for clarity and write $\pi$. We hope the time indices will be evident from the context.
When an information theoretic reward, for instance, information gain, is introduced to the problem, the reward can assume the following form $\rho(b, a, z', b')\! =\! (1 \! - \! \lambda)r^x(b, a)\! + \! \lambda r^I(b, a, z', b')$, in this case it is a function of two subsequent beliefs, an action, and an observation. Note that in the setting of nonparametric beliefs, we shall resort to sampling approximations using $m_x$ samples of the belief. Such a reward is comprised of the expectation over the state and action dependent reward
\begin{align} \label{eq:StateDepRew}
	&  r^{x}(b, a)= \mathbb{E}_{x \sim b}{[r^x(x, a)}] \approx \frac{1}{m_x}\sum_{i=1}^{m_x} r^x(x^i, a),
\end{align}
and the information-theoretic reward $r^I(\cdot)$ weighted by $\lambda$, which in general can be dependent on consecutive beliefs and the elements relating them (e.g. information gain).
The online decision making goal is to find an action to execute, maximizing the action value function 
\begin{align}
	Q^{\pi}(b_k, a_k) =  \underset{z_{k+1}}{\mathbb{E}}\bigg[\rho(b_k, a_k, z_{k+1}, b_{k+1}) +  V^{\pi}(b_{k+1})\bigg| b_k, a_k\bigg], \label{eq:Qfunc}
\end{align} 
where $\pi$ is the policy and the value function 
\begin{align}
	V^{\pi}(b_{k}) = \underset{z_{k+1:k+L}}{\mathbb{E}}\Bigg[\sum_{\ell=k}^{L-1}\rho(b_{\ell}, a_{\ell}, z_{\ell+1}, b_{\ell+1})\Bigg| b_k, \pi \Bigg],
\end{align}
is expected cumulative reward under the particular policy $\pi$.

When the agent performs an action and receives an observation, it shall update its belief from $b$ to $b'$. Let us denote the update operator by $\psi$ such that $b'=\psi(b, a, z')$. In our context, it will be a particle filter since we focus on the setting of nonparametric beliefs. Moreover, we define a propagated belief $b'^{-}$ as the belief $b$ after the robot performed an action $a$ and before it received and observation. 

In this paper, we present a new risk-averse decision making problem. We augment the $\rho$-POMDP \cite{Araya10nips}  objective with a novel, probabilistic general belief-dependent constraint. To our knowledge, all previous chance constrained formulations such as \cite{Bry11icra}, \cite{Santana16aaai} suffer from limiting assumptions. To be specific, they perform averaging over the state trajectories as we unveil in this work and therefore are not distribution-aware formulations. Moreover, a general belief-dependent constraint was not studied nor proposed. Nevertheless, such a  constraint is of the highest importance. For instance, as we discuss in \cite{Zhitnikov23arxiv}, such a formulation can be used to determine when to stop exploration, e.g.~in an active SLAM context, which is an open problem currently \cite{Cadena16tro}, \cite{Placed22arxiv}. 

Risk-averse planning has been actively investigated \cite{Chow15nips}, \cite{Zhitnikov22ai}, but risk aversion was not considered for the constraint to the best of our knowledge.	
\section{Continuous $\rho$-POMDP enhanced with Probabilistic Belief Dependent Constraints} \label{sec:Framework}
Further we formulate the problem and in due course give examples of possible belief dependent constraints. 
\subsection{Problem Formulation}
In this work we augment the classical belief dependent formulation described above with a general probabilistic belief-dependent constraint. We introduce a new problem with the following objective
\begin{align}
	&a^*_{k} \in \arg\max_{a_k \in \mathcal{A}} \{ Q^{\pi^*}(b_k, a_k)\} \text{ \ \ subject to} \label{eq:ConstrObj} \\
	&\prob\left(c(b_{k:k+L}; \phi,\delta)=1\mid b_k, \pi^*_{k+1:k+L-1}, a_k  \right) \geq  1-\epsilon, \label{eq:OuterConstr}
\end{align}
where $c$ is a Bernoulli random variable. By $\pi^*$ we denote the belief tree policy defined by the planning algorithm.

Further, we will regard the probabilistic constraint \eqref{eq:OuterConstr} as the outer constraint. It requires two parameters,  $\epsilon$, and $\delta$. The former, $\epsilon$, is the probability margin within we permit to the future, rendered by possible future observations generating the beliefs (see Fig.~\ref{fig:lace}), violate the constraint, in other words, to be unprofitable or unsafe. The parameter $\delta$ is the margin for some particular sequence of the beliefs $b_{k:k+L}$. With a high probability of at least $1-\epsilon$, we want the received sequence of future posterior beliefs to fulfill the constraint.  

The inner constraint can be of  two forms. The first form is cumulative
\begin{align} 
	c(b_{k:k+L}; \phi, \delta) \bydef \mathbf{1}_{\left\{\left(\sum_{\ell=k}^{k+L-1} \phi(b_{\ell+1}, b_{\ell})\right) >  \delta\right\}}, \label{eq:InnerConstr1} 
\end{align}
and the second is multiplicative
\begin{align} 
	c(b_{k:k+L};\phi,\delta) \bydef \prod_{\ell=k}^{k+L} \mathbf{1}_{\left\{\phi(b_{\ell}) \geq \delta \right\}}, \label{eq:InnerConstr2} 
\end{align} 
where $\phi$ denotes a general belief-dependent operator. 
Let us interpret the two forms,  \eqref{eq:InnerConstr1} and \eqref{eq:InnerConstr2}.
The first form \eqref{eq:InnerConstr1} is formulated with respect to a cumulative value of the operator $\phi$ along a sequence of  beliefs generated by a sequence of possible future observations. In this form we permit immediate value of the operator $\phi$ to deviate but the cumulative value shall fulfill the inequality \eqref{eq:InnerConstr1}.  In contrast, \eqref{eq:InnerConstr2} states that every value of $\phi$ in the sequence of the beliefs shall fulfill the inequality \eqref{eq:InnerConstr2}, meaning to be larger than or equal to $\delta$. Both formulations are novel, to the best of our knowledge. Furthermore, the form of \eqref{eq:InnerConstr1} is motivated by the long standing question of stopping exploration \cite{Cadena16tro}. The form of \eqref{eq:InnerConstr2} is motivated by  safety, e.g, collision avoidance.

From now on, for clarity, in the constraint, we will use  $\pi^*$ instead of  $\pi^*_{k+1:k+L-1}$. When the problem \eqref{eq:ConstrObj} is augmented with the probabilistic constraint \eqref{eq:OuterConstr}, ideally, every selection of the action following the policy shall take into account the constraint at the root of the belief tree.  

\subsection{Possible Constraints} \label{sec:PossibleConstraints}
Subsequent to the formulation of the problem, in due course, we focus on possible operators $\phi$ as a constraint.   
One important example is a  safety constraint, e.g., collision avoidance or energy consumption.
We propose the following formulation,   
\begin{align}
	&\prob\big(\mathbf{1}_{\left\{ b_{k:k+L} \in \mathcal{B}^{\text{safe}}_{k:k+L}   \right\}}|b_k, \pi_{k+1:k+L-1}, a_k  \big)  \geq  1-\epsilon, \label{eq:ProbSafety}
\end{align}
where $\mathcal{B}^{\text{safe}}_{k:k+L}$ is the space of safe belief sequences starting at time index $k$ and of length $L$.
To relate to \eqref{eq:OuterConstr}, in \eqref{eq:ProbSafety}:  $c(b_{k:k+L} ) \bydef \mathbf{1}_{\left\{  b_{k:k+L} \in \mathcal{B}^{\text{safe}}_{k:k+L}  \right\}}$.
Further, we show explicitly why this formulation is advantageous over previous formulations of safety constraint.
The safeness of a sequence of beliefs $b_{k+1:k+L}$ can be defined in various ways. One possibility is 
\begin{align}
	& \mathbf{1}_{\left\{  b_{k:k+L} \in \mathcal{B}^{\text{safe}}_{k:k+L} \right\}} \triangleq \textstyle\prod_{i=k}^{k+L}\mathbf{1} \left\{ \prob\big(\mathbf{1}_{\left\{x_{i}\in \mathcal{X}^{\text{safe}}_{i}\right\}}| b_{i}\big) \geq \delta\right\}, \label{eq:ProbSafeTraj}
\end{align}
where $\mathcal{X}^{\mathrm{safe}}_{i}$ is the safe space, which generally can be time-dependent, e.g., due to moving obstacles in the context of collision avoidance.

Another possibility is to use the Conditional Value at Risk (CVaR) operator for collision avoidance as $-\phi$ (minus sign is needed merely to maintain $\geq$ in \eqref{eq:InnerConstr2}). We define the deviation of the robot’s position
from the safe region $\mathcal{Y}^{\ell}$ considering the obstacle $\ell$ as follows
$	\mathrm{dist}(x,\mathcal{Y}^{\ell}) = \min_{y \in \mathcal{Y}^{\ell}} \|x - y\|_2$.
Note that $\bigcap_{\ell=1}^M \mathcal{Y}^{\ell} = \mathcal{X}^{\mathrm{safe}}$ for $M$ obstacles. 
The following belief-dependent constraint 
$\mathrm{CVaR}_{\alpha}[\mathrm{dist}(x,\mathcal{Y}^{\ell})] \leq \delta $ 
assures the safety,
where $x \sim b. $
Now the event to be safe is
\begin{align} \label{eq:cvarcoll}
	&\!\!\mathbf{1}_{\left\{  b_{k:k+L} \in \mathcal{B}^{\text{safe}}_{k:k+L} \right\}}\!\! \bydef \!\! \prod_{i=k}^{k+L} \! \prod_{\ell=1}^M \! \mathbf{1}_{\left\{-\mathrm{CVaR}_{\alpha}[\mathrm{dist}(x_i,\mathcal{Y}^{\ell}_i)] \geq \delta. \right\}},
\end{align}  
Note that the CVaR operator cannot be represented as the expectation over the observations. Therefore this is a general belief-dependent constraint and \emph{not supported} by existing constrained POMDP approaches. Let us explain the meaning of such a constraint. Let $\zeta = \mathrm{dist}(x,\mathcal{Y}^{\ell})$.
By definition
$\mathrm{CVaR}_{\alpha}[\zeta] \bydef \mathbb{E}[\zeta|\zeta\geq \mathrm{VaR}_{\alpha}(\zeta)]$,
where the value-at-risk $\mathrm{VaR}_{\alpha}(\zeta)$ at confidence level $\alpha$ is the   $1-\alpha$ quantile of  $\zeta$, namely,  
$$\mathrm{VaR}_{\alpha}[\zeta] \bydef \min \{ \xi | \prob (\zeta\leq \xi) \geq 1-\alpha\}.$$
The value $\mathrm{VaR}_{\alpha}[\mathrm{dist}(x,\mathcal{Y}^{\ell})]$ is the minimal value such that with probability at least $1- \alpha$ the deviation from the safe space considering one obstacle is smaller than or equal it. The $\mathrm{CVaR}$ is taking the average of the unsafe  tail. Meaning if the unsafe tail is extremely unsafe but with low probability, such a constraint will catch that. The constraint formulated with probabilities of \eqref{eq:ProbSafeTraj} as well as conventional \cite{Santana16aaai} is unable to distinguish such a behavior. The distribution over the unsafe part of the beliefs is unaccessible. We note that such a constraint was suggested by \cite{Hakobyan20icra}, in the setting of randomly moving obstacles. However  \cite{Hakobyan20icra} assumes a fully observable state and linear models, and not the general POMDP setting considered herein. 

Another example of a general belief-dependent constraint is  Information Gain ($\mathrm{IG}$),  defined as follows 
\begin{align}
	\mathrm{IG}(b, a, z', b') = -\mathcal{H}(b') + \mathcal{H}(b), \label{eq:IG}
\end{align}
where $\mathcal{H}(\cdot)$ denotes differential entropy. Utilizing this constraint with the form \eqref{eq:InnerConstr1} allows one to reason if the cumulative information gain along a planning horizon is significant enough (above threshold $\delta$) with the probability of at least $1-\epsilon$. 
Such a capability has a number of implications. For instance, in the context of informative planning and active SLAM, instead of prompting the agent to maximize its information gain, we can require that it does so only if it is able to decrease  uncertainty in some tangible amount. This is a new concept made possible by our general formulation, which therefore can be used to identify, e.g., when to stop exploration \cite{Zhitnikov23arxiv}.  

Let us discuss one more important constraint, the probability of reaching the goal (see, e.g., \cite{Bry11icra}).  Throughout the manuscript, for clarity, we assumed that the operator $\phi$ is identical for all time indices. We now relax that assumption and redefine the constraint of the first form as follows\footnote{We denote $f \equiv g$ for two operators, if we have $f(x)=g(x) \quad \forall x$.} 
\begin{align} 
	c(b_{k:k+L};\phi_{k:k+L},\delta) \bydef \mathbf{1}\left\{\left(\sum_{\ell=k}^{k+L-1} \phi_{\ell+1}(b_{\ell+1}, b_{\ell})\right) \geq  \delta\right\}. \label{eq:Constr1TimeInd} 
\end{align}
Further, let $\phi_{\ell+1}(\cdot) \equiv 0 \quad \forall \ell \in k: k+L-2$ and 
\begin{align}
	\phi_{k+L}(b_{k+L}) = \prob (x_{k+L} \in \mathcal{X}^{\text{goal}}| b_{k+L}), \label{eq:ReachingTheGoalProb}
\end{align}
where \eqref{eq:ReachingTheGoalProb} defines the task of reaching the goal.

\section{Approach to Continuous $\rho$-POMDP with Belief Dependent Probabilistic Constraints}\label{sec:OuterConstrEvalAndAlgorithms}
Having presented our problem formulation and the examples of possible belief-dependent operators to serve as an inner constraint, we are keen to proceed into the adaptive approach to precisely evaluate the sample approximation of our probabilistic constraint.
\subsection{Coupled Outer Constraint Evaluation and Belief Tree Construction} 
In this section, we delve into the evaluation of our novel formulation of the probabilistic  constraint \eqref{eq:OuterConstr}.  We start by presenting a helpful lemma.
\begin{lem}[Representation of our outer constraint] \label{lem:Represent}
	\begin{align}
		&\prob\left(c(b_{k:k+L};\phi, \delta)|b_k, \pi_{k+1:k+L-1}, a_k  \right)  = \underset{z_{k+1:k+L}}{\mathbb{E}}\left[c(b_{k:k+L}; \phi, \delta) \bigg|b_k, \pi_{k+1:k+L-1}, a_k\right]. \label{eq:ConstraintExp}
	\end{align}
\end{lem}
\noindent The reader can find the proof in Appendix \ref{proof:Represent}. From  Lemma~\ref{lem:Represent} we behold how to obtain the best sample approximation of the outer constraint, since the theoretical expectation \eqref{eq:ConstraintExp} is out of the reach. In practice, we approximate expectation in \eqref{eq:ConstraintExp} with a finite number of samples, such that the outer constraint becomes
\begin{align}
	\frac{1}{m} \sum_{i=1}^{m} c(b^i_{k:k+L};\phi, \delta) \geq 1-\epsilon, \label{eq:SampleOuterConstraint}
\end{align}
where $m$ is the number of the observation sequences $z_{k+1:k+L}$ expanded from action $a_k$ \eqref{eq:ConstrObj} at the root of the belief tree. If the belief tree is given, we can traverse it from the bottom up and calculate value of $c(b^i_{k:k+L};\phi, \delta)$ for $i \in 1 \ldots m$ along the way such that when we reach the root, we have everything to evaluate \eqref{eq:SampleOuterConstraint}. In general, since the parameter $m$ has to be known,  this applies to approaches that decouple belief tree construction from the solution, e.g.,~SS  algorithm \cite{Kearns02jml}. 

However, we would like to guide the belief tree construction such that if the action does not fulfill the outer constraint we will spend on it as less effort as possible.  We shall regard another interesting aspect of \eqref{eq:ConstraintExp}. Because $c$ is a Bernoulli random variable, by definition  $1 \geq \frac{1}{m} \sum_{i=1}^{m} c(b^i_{k:k+L};\phi, \delta)$. This imply that, under the condition $\epsilon=0$, to satisfy \eqref{eq:SampleOuterConstraint}, we shall require     $\big(\sum_{i=1}^{m} c(b^i_{k:k+L};\phi, \delta)\big) = m $. In other words, in this setting we will not be able to early accept action (before expanding $m$ future belief laces). However, as we further see we will be able to do a highly efficient pruning. Further, we describe an adaptive constraint evaluation mechanism for general $\epsilon$ and after that focus on the case of $\epsilon = 0$. 

\subsubsection{Accurate Adaptive Constraint Inquiry with $0<\epsilon \leq 1$} \label{sec:Adaptive}
Having presented \eqref{eq:SampleOuterConstraint}, we are now ready to address a complete belief tree construction. We bound the expression of the sample approximation of outer constraint \eqref{eq:SampleOuterConstraint} from each end using the already expanded part of the belief tree. 
Suppose the online algorithm at the root for each action expands upon termination $m$ laces appropriate to the drawn observations $\{z^i_{k+1:k+L}\}_{i=1}^{m}$. Each lace $i$ corresponds to a particular realization of the return. 

Suppose the algorithm already expanded $n \leq m$ laces with some order. We denote expanded laces by a sub-sequence $j=1 \ldots n$, such that $i^j$ is the index of the observation sequence, i.e, $i^j \in 1\dots m$. The lower bound $\lb(b_k,\pi)$ on the \eqref{eq:SampleOuterConstraint} is 
\begin{align}
	&1-\epsilon  \overbrace{\leq}^{?} \underbrace{\frac{1}{m} \sum_{j=1}^{n} c(b^{i^{j}}_{k:k+L};\phi,\delta)}_{\lb(b_{k},\pi)} \leq  \frac{1}{m} \sum_{i=1}^{m} c(b^{i}_{k:k+L};\phi,\delta).  \label{eq:Lower}
\end{align}
Whereas the upper bound $\ub(b_k,\pi)$ reads
\begin{align}
	&  \frac{1}{m}\! \sum_{i=1}^{m}\! c(b^{i}_{k:k+L};\phi,\delta) \!\leq\! \underbrace{\frac{m\!-\!n}{m}\!+\!\frac{1}{m} \!\sum_{j=1}^{n} \! c(b^{i^{j}}_{k:k+L};\phi, \delta)}_{\ub(b_k,\pi)}  \overbrace{<}^{?} 1-\epsilon.  \label{eq:Upper}
\end{align}
By the question mark  we denote the inequalities that shall be fulfilled online to check either the sample approximation of the outer constraint \eqref{eq:SampleOuterConstraint} is met \eqref{eq:Lower} or violated \eqref{eq:Upper}. These bounds allow to evaluate the constraint {\bf adaptively} before expanding the $m$ laces of the belief sequences $b_{k:k+L}$. Such a technique is a applicable for both settings: open and closed-loop. Note that both bounds advance with the step size $1/m$. Moreover, each added observation lace only one of the bounds is contracting the lower of the upper. If the expanded lace results in $c(b^{i_{j}}_{k:k+L}; \phi, \delta) = 1$ the lower bound makes a step. This event is happening with probability $\prob\left(c(b_{k:k+L}; \phi,\delta)=1\mid b_k, \pi  \right)$. Conversely, if the expanded lace results in $c(b^{i_{j}}_{k:k+L}; \phi, \delta) = 0$ the upper bound makes a step. This  event is happening with probability $\prob\left(c(b_{k:k+L}; \phi,\delta)=0\mid b_k, \pi  \right)$.

One example of an adaptive usage of \eqref{eq:Lower} and \eqref{eq:Upper} is to save time in open loop planning or alternatively spend more time on the action sequences which fulfill the constraint. Envisage a static action sequence to be checked. After each expanded lace $c(b^{i_{j}}_{k:k+L}; \phi, \delta)$ of \eqref{eq:SampleOuterConstraint} we are probing \eqref{eq:Lower}, if fulfilled, we know that the sample approximation of outer constraint is satisfied, and we can stop dealing with the constraints for this candidate action sequence. Else we are trying  \eqref{eq:Upper}; if fulfilled, we know that the current action sequence violates the sample approximation of outer constraint \eqref{eq:SampleOuterConstraint}. The third possibility is to add one more lace and check again. In such a way, we adaptively expand the lower possible number of inner constraint laces to be evaluated and {\color{emeraldgreen} {\bf validate}} or {\color{hibiscus} {\bf invalidate}} the action sequence depending on whether the probabilistic constraint is fulfilled on not. The presented adaptivity mechanism is exact and guaranteed to satisfy or discard our probabilistic constraints. To our knowledge no analogs to this  exists in the literature, e.g. \cite{Santana16aaai}.

Another example is the closed loop setting, where we deal with policies. Further, we focus attentively on the closed loop setting in this paper.

\subsubsection{Efficient Exact Adaptive Constraint Pruning under Condition of $ \epsilon =0 $  }
The constraint confidence $\epsilon$ controls the stiffness of the condition that the distribution of belief-dependent constraint shall fulfill.  The maximal stiffness is reached when $\epsilon =0 $. Focus on multiplicative flavor of inner constraint \eqref{eq:InnerConstr2}, in this case we have an interesting behavior summarized in following theorem.  
\begin{thm}[Necessary and sufficient condition for feasibility of probabilistic constraint]
	\label{thm:pruning}
	Fix $\epsilon =0$ and $\delta$. Let the inner constraint comply to \eqref{eq:InnerConstr2}. The fact that  
	\begin{align}
		\frac{1}{m} \sum_{i=1}^{m} c(b^i_{k:k+L};\phi, \delta) \geq 1	- \cancelto{0}{\epsilon}
	\end{align}
	implies that $\forall i,\ell \quad \phi(b^i_{\ell})\geq \delta$.  Moreover, if 
		\begin{align}
		\frac{1}{m} \sum_{i=1}^{m} c(b^i_{k:k+L};\phi, \delta) < 1,	
	\end{align} 
	so  $\exists i,\ell \quad \phi(b^i_{\ell}) < \delta$
\end{thm}
We placed proof in the Appendix~\ref{proof:pruning}. Immediate result of Theorem~\ref{thm:pruning} is soundness of our pruning technique.  By the arriving to a belief $b_{\ell}$, we prune all the actions in the belief tree resulting in  
$\phi(b_{\ell+1}) < \delta $ for some future observation a single step ahead. In such a way eventually in the belief tree will be solely the actions satisfying the probabilistic constraint \eqref{eq:SampleOuterConstraint} with $\epsilon =0$.

In the next section, we present the algorithms. 
\subsection{The Algorithms} \label{sec:PCSS}
\input{./floats/constrainedSS}
\subsubsection{Probabilistically Constrained  Sparse Sampling ($\epsilon =0$)}
In particular, inspired by SS \cite{Kearns02jml} and adaptivity aspects in \cite{Barenboim22ijcai}, we present an algorithm (Alg.~\ref{alg:constrainedSS}) for the general form of inner constraint of \eqref{eq:InnerConstr2}. Notably, to our knowledge, it is the first algorithm in the continuous domain dealing with probabilistic constraints in a POMDP setting. 
In this paper we are focusing on safety aspects. We allow ourselves to assume that the inner constraint is fulfilled on actual belief $b_k$ we acquired from the inference, namely $ \mathbf{1}_{\left\{\phi(b_{k}) \geq \delta \right\}} =1 $.  To keep presentation clear we also simplify the reward the be dependent solely on the current belief and denote $\rho(b)$. 

Alg.~\ref{alg:constrainedSS} is presented for $\epsilon = 0$.  It uses the Bellman optimality criterion while traversing the tree from the bottom up. However, we employ pruning actions violating the constraint on the way forward (down the tree). This happens in line $14$ of the Alg.~\ref{alg:constrainedSS}. Because we actually check the inner constraint on the way forward when the algorithm hits the bottom of the tree, we are left solely with actions fulfilling the outer constraint with $\epsilon=0$, so we do not need any additional checking on the way up at all. 

It shall be noted that the presented algorithm is heavy from the computational point of view due to the inability to trade the number of observation laces to the larger horizon. We introduce the Probabilistically Constrained Forward Search Sparse Sampling (PCFSSS)  approach to fix this issue and relax the assumption that $\epsilon = 0$. One also can easily relax the assumption that $\epsilon =0$ in Alg.~\ref{alg:constrainedSS}. 
\subsubsection{Probabilistically Constrained Forward Search Sparse Sampling (arbitrary $\epsilon$, anytime algorithm)} \label{sec:fsss}
To increase the horizon of the planner on the expense of number of observation laces we can use a variation of Forward Search Sparse Sampling algorithm \cite{Barenboim22ijcai}. In this algorithm we do trials of observation laces. In each trial, we descend with the lace of observations and actions  intermittently, calculate the beliefs along the way and ascend back to the root of the belief tree.  For observation we can use a circular buffer of size $m_d$.  Let us do the following calculations. We know that the lower bound in \eqref{eq:Lower} advance in steps of $\frac{1}{m}$. We can define $n_{\text{accept}}$, the smaller number of laces required to fulfill the inner constraint, namely $c(b^{\text{lace}}_{k:k+L}; \phi,\delta)=1$ to accept an candidate policy. This would be 
\begin{align}
	n_{\text{accept}} = \lceil m(1 - \epsilon) \rceil.
\end{align}
Similarly from \eqref{eq:Upper} $n_{\text{reject}}$, the  minimal number of laces required to violate the inner constraint,  namely $c(b^{\text{lace}}_{k:k+L}; \phi,\delta)=0$ to discard a candidate policy.
This would be 
\begin{align}
	n_{\text{reject}} = \lfloor m\epsilon \rfloor.
\end{align}
Suppose we expand from each belief action node at most $m_d$ observations. Overall we will have  $(m_d)^L$  observation laces. Let us use Upper Confidence Bound (UCB) strategy to decide which action to go.  Let us pinpoint  the following facts.  
\begin{itemize}
	\item If the inequality $\phi(b_{\ell+1}) < \delta $ is fulfilled at some belief in the tree. It means we have a whole subtree violating the inner constraint. This is because the belief $b_{\ell}$ is a part of all laces stemming from it. Meaning we have  $(m_d)^{L-\ell}$ laces violating the constraint. If  $ (m_d)^{L-\ell} \geq n_{\text{reject}}$ we can discard an action $a_{\ell-1}$ which is ancestor to the belief $b_{\ell}$. 
	\item If $(m_d)^{L-\ell} < n_{\text{reject}}$, we permit this action to expand this belief node. However, we shall update $n_{\text{reject}}$, such that $n_{\text{reject}} \leftarrow n_{\text{reject}} - (m_d)^{L-\ell} $, descend until the bottom of the tree and ascend back.
\end{itemize}  
In such a manner when we stop and return an optimal action in the root. It is guaranteed that this action will be safe with respect to our formulation. 
\section{Relation to chance constraints for $\epsilon=0$} \label{sec:RelationToChance}
Although, our formulation is universal, in this paper we are focusing on safety of the agent. Therefore we shall thoroughly regard state-of-the art safety constraint under partial observability, which is chance constraint.  Our immediate extension is augmenting chance constrained formulation with general belief dependent rewards and  extending it to the continuous domain in terms of states and observations. As a consequence, we obtain the following objective  
	\begin{align}
	&a^*_{k} \in \arg\max_{a_k \in \mathcal{A}} \{ Q^{\pi^*}(b_k, a_k)\} \text{ \ \ subject to}  \label{eq:SotaObj}\\
	&\prob(\mathbf{1}_{\left\{\tau \in  \times_{i=0}^{L-\ell+k} \mathcal{X}_{i}^{\text{safe}}\right\}}| b_{\ell}, a^*_\ell, \pi^*_{\ell+1:k+L-1})  \geq   \delta_{\ell}, \quad \forall b_{\ell} \text{ such that } \ell \in k: k+L-1  \label{eq:SotaSafety}
\end{align}
where $\tau = x_{k:k+L} $ is the trajectory of the states. Observe that the chance-constraint in \eqref{eq:SotaSafety} is enforced from each non terminal belief as in\cite{Santana16aaai}.  Let us emphasize that the question of selecting the $\delta_{\ell}$ per depth $\ell$ of the belief tree requires clarification and has not been answered to the best of our knowledge.  For now let us set $\delta_{\ell}=\delta \ \forall \ell$. Further, we will carefully regard this question.
\subsection{Rigorous Analysis of Safety Constraint and Interrelation of Our Probabilistic Constraint to Chance Constraint} \label{sec:Safety}
Following the previous discussion, we now examine in detail  the chance-constrained formulation. As we will shortly see, our analysis unveils that our novel  safety constraint  \eqref{eq:ProbSafety}, \eqref{eq:ProbSafeTraj} with $\epsilon=0$ is similar to chance constraint. However, as we further show, the existing chance-constrained state-of-the-art formulation, e.g., \cite{Bry11icra}, \cite{Santana16aaai} have some drawbacks. Since the paper \cite{Bry11icra} presents a  parametric method for Gaussian beliefs, it is relevant for us solely from the constraint formulation perspective.
Let us recite that by setting $\epsilon > 0$, our probabilistic constraint is unmatched to chance constraint. Moreover, safety constraint as dependent belief operator is genuinely nothing more than expectation over the indicator of safe states \eqref{eq:ProbSafeTraj}, as such 
\begin{align}
	\phi(b_{\ell}) = \prob\Big(\mathbf{1}_{\left\{x_{\ell}\in \mathcal{X}^{\text{safe}}_{\ell}\right\}}| b_{\ell}\Big) = 	\mathbb{E}\bigg[\mathbf{1}_{\left\{x_{\ell}\in \mathcal{X}^{\text{safe}}_{\ell}\right\}}| b_{\ell}\bigg ]. \label{eq:SafetyOperator}
\end{align} 
Meaning we narrowed here the discussion to the operator $\phi$ being the first moment of state dependent function, while our formulation supports general belief dependent operators without any limitations. 

Let us now delve into the question why our approach \eqref{eq:ProbSafety} is more general , and its advantages over an existing state-of-the-art formulation \eqref{eq:SotaSafety}, e.g., \cite{Bry11icra}, \cite{Santana16aaai}. 

We, however, shall compare two formulations of the constraint, ours \eqref{eq:ProbSafety} and the conventional. 
Let us rigorously show why the conventional chance constraint is insufficient to account for risk. Further, we will present our adaption of RAO* algorithm presented in \cite{Santana16aaai} to continuous spaces, This is a contribution on its own to the best of our knowledge. Further, to examine performance of both formulations, we compare our  probabilistic constraint (Alg.~\ref{alg:constrainedSS}) against it. 

Our key insight is that a conventional chance constraint only partially depends on the observations since it averages state trajectories and is not explicitly formulated with respect to posterior beliefs. In the next section, we shed light on this aspect.

\subsubsection{Enforcing chance constraint at the root of the belief tree}
We face that in \cite{Bry11icra} the chance constraint imposed only in the root of the belief tree. With this motivation in mind we focus for the moment of the whole belief tree, namely, $\ell=k$ and inspect probability that the trajectory $\tau$ will be safe.  Note that from the properties if the indicator variable
\begin{align}
	&\mathbf{1}_{\left\{\tau(\omega) \in  \times_{i=0}^L \mathcal{X}_{k+i}^{\text{safe}}\right\}}  =  \mathbf{1}_{\left\{ \bigcap_{i=k}^{k+L} \left\{x_i(\omega) \in \mathcal{X}_{i}^{\text{safe}}\right\}\right\}} = \bigwedge_{i=k}^{k+L}\mathbf{1}_{\left\{ x_i(\omega) \in   \mathcal{X}^{\text{safe}}_i\right\}}=\prod_{i=k}^{k+L}\mathbf{1}_{\left\{ x_i(\omega) \in   \mathcal{X}^{\text{safe}}_i\right\}} \quad \forall \omega \in \Omega, \nonumber
\end{align}
where $\Omega$ is the space of the outcomes.
 
Meaning, the safe trajectory is the trajectory comprised of safe states. Another property of the indicator variable is
\begin{align}
	&\prob(\mathbf{1}_{\left\{\tau \in  \times_{i=0}^L \mathcal{X}^{\mathrm{safe}}_i\right\}}| b_k, \pi) = \int_{\tau} \textstyle\prod_{i=k}^{k+L} \mathbf{1}\left\{ x_i\!\in\!  \mathcal{X}^{\mathrm{safe}}_i\right\} \probd(\tau |b_k, \pi_{k+1:k+L-1}, a_k)\mathrm{d} \tau. \label{eq:ExpectationTraj} 
\end{align}
Crucially, as seen in \eqref{eq:ExpectationTraj}, and in contrast to \eqref{eq:ProbSafety} and \eqref{eq:ProbSafeTraj}, such an approach is not distribution-aware since it sees the posterior solely through the lens of expectation.

In such a formulation the observations are required merely to decide which action to take alongside the trajectory. As  visualized in  Fig.~\ref{fig:ChanceConstraintRoot}, we regard in that formulation trajectories coming from different posterior beliefs without a distinction from which belief they arrived. In contrast, we will see further our approach is truly distribution-aware since it accounts for the number of safe posteriors. Let us show the following  lemma. 
\begin{lem}[Probability of trajectory] \label{lem:TrajProbab}
	\begin{align}
		&\probd(\tau |b_k, \pi_{k+1:k+L-1}, a_k) = \probd_T(x_{k+1}|x_{k}, a_k )b_k(x_k)\cdot   \nonumber\\
		& \!\!\!\!\!\!\!\!\!\!\!\underset{z_{k+1:k+L-1}}{\int} \!\!\prod_{i=k+1}^{k+L-1} \! \probd_T(x_{i+1} | x_{i}, \pi(b_i(b_{i-1}, a_{i -1}, z_{i}))) \probd_Z(z_{i}|x_i) \mathrm{d}z_{k+1:k+L-1}. \label{eq:DistributionConstr}
	\end{align}
\end{lem}
\noindent We provide the proof in Appendix \ref{proof:TrajProbab}. As we see, the observations are required solely for decision which action to take. In particular, when we deal with static action sequences $a_{k:k+L-1}$ the observations cancel out.  According to the former explanation, we conclude the following.
 \begin{figure}
 	\centering
 	 \begin{minipage}[t]{0.45\textwidth}
 		\centering
 		\includegraphics[width=\textwidth]{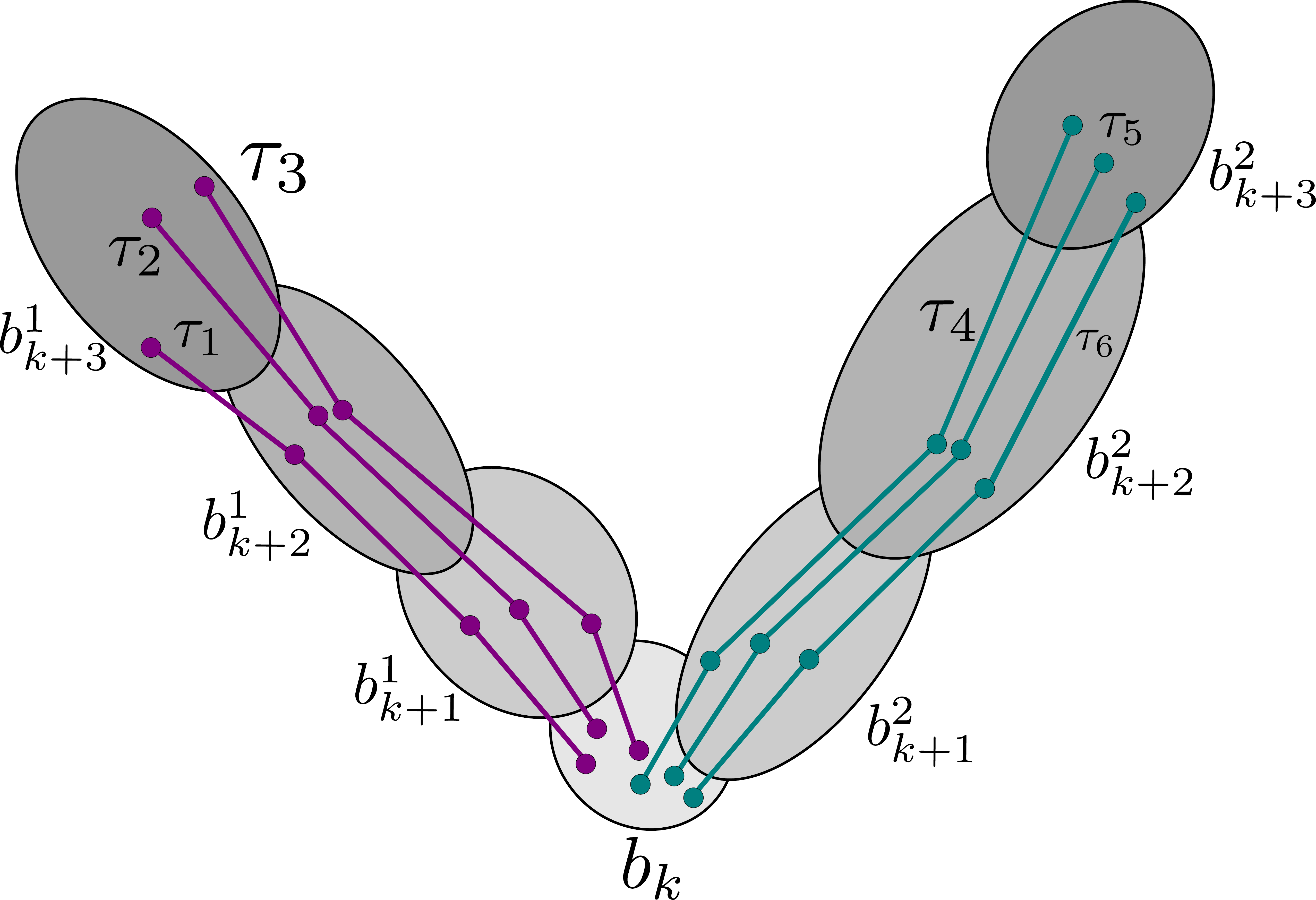}
 		\subcaption{}
 		\label{fig:ChanceConstraintRoot}
 	\end{minipage}%
 	\hfill
 	\begin{minipage}[t]{0.45\textwidth}
 		\centering 
 		\includegraphics[width=0.6\textwidth]{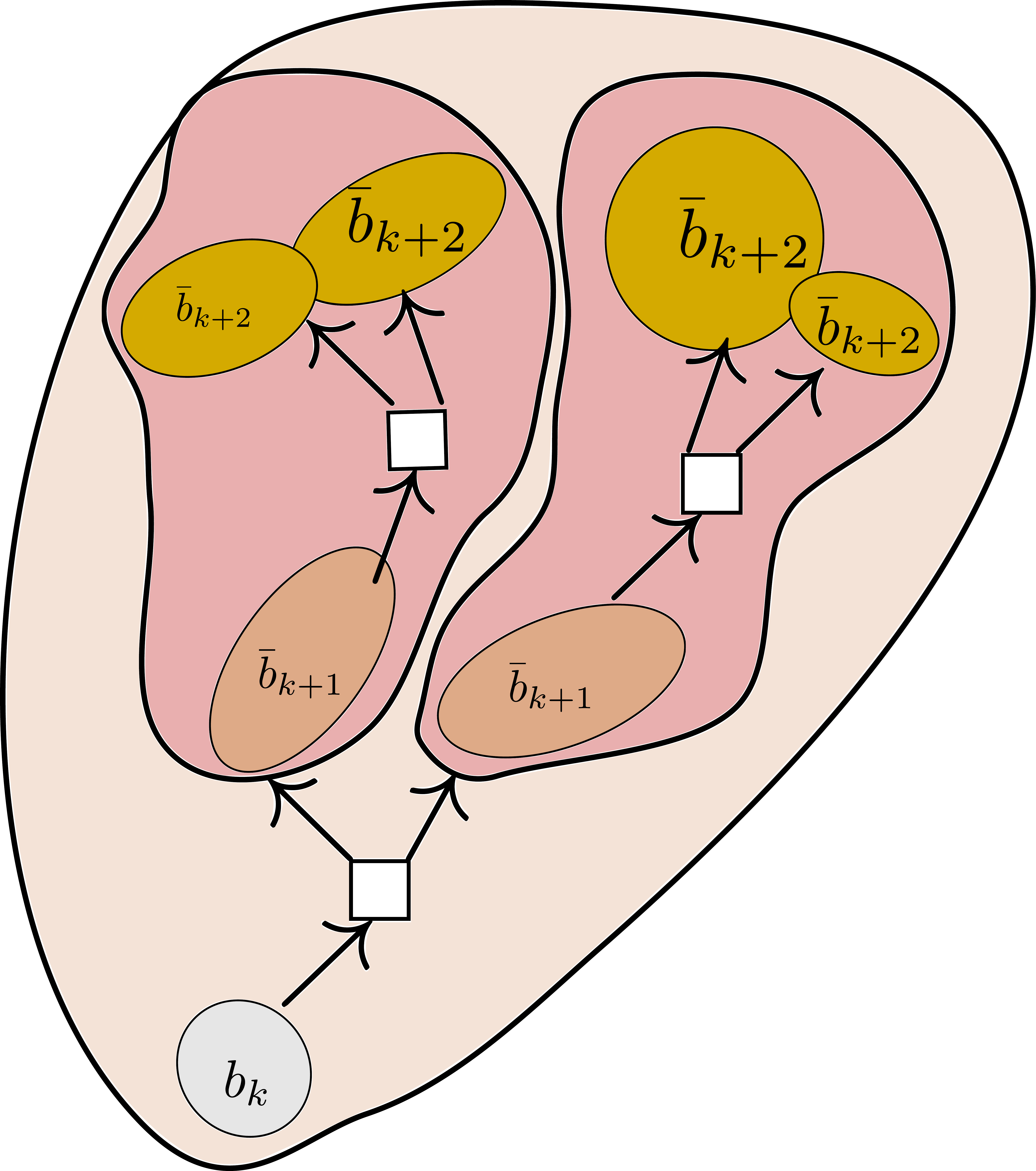}
 		\subcaption{}
 		\label{fig:ChanceConstraint}
 	\end{minipage}%
	\caption{\textbf{(a)}Visualization of the conventional chance constraint enforced from the root of the belief tree. By gray ovals illustrate the posterior beliefs. The indicators over the {\color{teal}{\bf teal}} and {\color{inkscapePurple}{\bf purple}} trajectories are averaged without any distinction (Section \ref{sec:Safety}). However, they a parts of different sequences of the posterior beliefs. \textbf{(b)}Visualization of sub-tree resolution of chance constraints. When enforced from each non terminal belief, its resolution is solely sub-trees and not the beliefs themselves (Section \ref{sec:SubtreeRes}). Our method does not suffer from the limitations presented in {\bf both} figures.} 
\end{figure}
First, the chance constraint by formulation is oblivious to the distribution of future beliefs (Fig.~\ref{fig:ChanceConstraintRoot}). Second, enforcing the chance constraints from each not terminal belief (Fig.~\ref{fig:ChanceConstraint}) ameliorate the formulation.  This brings us to the next section.     

\subsubsection{Investigating Chance Constrained Belief MDP} \label{sec:SubtreeRes}
A particularly interesting question, is how enforcing the chance constraint from each non terminal belief patches the formulation to be partially risk aware (Fig.~\ref{fig:ChanceConstraint}).  
As mentioned, this paper focuses on belief-dependent constraints and rewards. Therefore, our novel Probabilistic Constraint (PC) given by \eqref{eq:ProbSafeTraj} and Alg.~\ref{alg:constrainedSS}, which is based on  Belief-MDP (BMDP) and formulated in terms of beliefs and not the state trajectories. In particular,  BMDP formulation will permit us to deal with belief-dependent rewards such as information gain and differential entropy. All in all, we conclude that to compare with our formulation, we shall analyze the chance constraint on Belief-MDP (BMDP) level. 

In order to benchmark our approach \eqref{eq:OuterConstr} to the conventional Chance Constraint (CC) formulation \eqref{eq:SotaSafety}, we now scrutinize the latter from another angle and reformulate   it in the context of  posterior beliefs as it cannot be used in belief-dependent solvers in its current form. Such an extension has not been done previously, to the best of our knowledge. 
Let us present a lemma which will shed light on the relation between the conventional formulation of safety constraint and the posterior beliefs.  To improve readability let us introduce yet another Bernoulli variable $\iota_i(\omega) \bydef \mathbf{1}_{\left\{ x_i(\omega) \in   \mathcal{X}^{\mathrm{safe}}_i\right\}}$. We also indicated the constraint to be met in \eqref{eq:ChanceSafety} with underbrace. 
\begin{lem}[Average over the posteriors obtained from the safe priors] \label{lem:AvSafePosteriors}
	\begin{align}
		&\prob\Bigg(\bigwedge_{i=k}^{k+L} \iota_i| b_k, \pi\Bigg) = \underbrace{\prob( \iota_k | b_k) \underset{{\bar{z}_{k+1}}}{\mathbb{E}}\Big[ \underbrace{\prob(\iota_{k+1} |\bar{b}_{k+1}) \underset{{\bar{z}_{k+2}}}{\mathbb{E}}\Big[ \underbrace{\prob(\iota_{k+2} |\bar{b}_{k+2}) \dots}_{\geq \delta} \Big|a_{k+1},\bar{b}_{k+1}^{\mathrm{safe}} \Big]}_{\geq \delta }  \Big|a_{k},\bar{b}_{k}^{\mathrm{safe}} \Big]}_{\geq \delta}, \label{eq:ChanceSafety}
	\end{align}
\end{lem}
\noindent where $\bar{b}_i = \psi(b_{i-1}^{\mathrm{safe}},a_{i-1},\bar{z}_{i})$ which is different than $b_i=\psi(b_{i-1},a_{i-1},z_{i}))$ (even if the realizations of the observations are equal $\bar{z}_i = z_i$) used in \eqref{eq:ProbSafeTraj}, similarly $\bar{z}_i$ is the observation given safe belief in previous time instance.
We provide  the proof in Appendix \ref{proof:AvSafePosteriors}.  Here,  $\psi$ is a method for Bayesian belief update.  Further details can be found in Appendix \ref{sec:PostFromSafePrior}.  Note that we added the bar to highlight that this is the posterior belief obtained by pushing forward in time (with an action and the observation) the safe portion of prior belief. 
 
Also, note that since $b_k$ is actual belief and we know that the agent is safe $\prob( \iota_k | b_k)=1$; and therefore can be omitted. The statement \eqref{eq:ChanceSafety} is very close to equation $(12)$ of \cite{Santana16aaai}. 
Let us elaborate on the notation of the safe belief. We define the safe belief as 
\begin{align}
	b^{\text{safe}}(x) \triangleq \frac{\mathbf{1}\left\{ x\!\in\!  \mathcal{X}^{\text{safe}}\right\}b(x)}{\int_{\xi \in \mathcal{X}} \mathbf{1}\left\{ \xi\!\in\!  \mathcal{X}^{\text{safe}}\right\}b(\xi)\mathrm{d}\xi}, \label{eq:SafeBelief}
\end{align}
i.e., we nullify the unsafe portion of the belief and re-normalize.   

We emphasize  
that  Lemma \ref{lem:AvSafePosteriors} is just a reformulation of the conventional constraint \eqref{eq:SotaSafety}, and now we proceed to analyzing the difference between that formulation and our constraint \eqref{eq:ProbSafety}. 
\subsection{Similarities and differences between our formulation and chance constraint} \label{sec:SimAndDiff}
We  rearrange our formulation  \eqref{eq:ProbSafety} to arrive to an expression similar in its form to \eqref{eq:ChanceSafety}. To that end, we use \eqref{eq:ConstraintExp} alongside  with \eqref{eq:ProbSafety} and \eqref{eq:ProbSafeTraj}, and get 
\begin{equation}
	\label{eq:OurSafety} 
	\begin{gathered}
		\underset{z_{k+1:k+L}}{\mathbb{E}}\left[\prod_{i=k}^{k+L}\mathbf{1}_{\left\{ \prob(\iota_i | b_{i}) \geq \delta\right\}}\bigg|b_k, \pi_{k+1:k+L-1}, a_k\!\right] =  \\
		\mathbf{1}_{\left\{ \prob(\iota_k | b_{k}) \geq \delta \right\}} \cdot \underset{z_{k+1:k+L}}{\mathbb{E}}\left[\prod_{i=k+1}^{k+L}\mathbf{1}_{\left\{ \prob(\iota_i | b_{i}) \geq \delta\right\}}\bigg|b_k, \pi_{k+1:k+L-1}, a_k\!\right] = \\
		\mathbf{1}_{\left\{ \prob(\iota_k | b_{k}) \geq \delta \right\}} \cdot\underset{{z_{k+1}}}{\mathbb{E}} \!\Big[\! \mathbf{1}_{\left\{\prob(\iota_{k+1}  |b_{k+1}) \geq \delta\right\}} \cdot 
		\underset{{z_{k+2}}}{\mathbb{E}} \!\Big[ \! \mathbf{1}_{\left \{\!\prob(\iota_{k+2} |b_{k+2}) \geq \delta \right\} } \cdots \\ \cdots \mathbf{1}_{\left \{\prob(\iota_{k+L-1} |b_{k+L-1})\geq \delta \right\}}\underset{{z_{k+L}}}{\mathbb{E}} \!\Big[ \! \mathbf{1}_{\left \{\!\prob(\iota_{k+L} |b_{k+L}) \geq \delta \right\} }   \Big|a_{k+L-1}\!,b_{k+L-1}\! \Big] \cdots \Big|a_{k+1}\!,b_{k+1}\! \Big] \!\Big|a_{k},b_{k}\!\Big]. 
	\end{gathered}
\end{equation}
While appearing similar, the two formulations are genuinely different.  Let us for clarity rewrite again  \eqref{eq:ChanceSafety} including the horizon this time 
\begin{equation}
	\label{eq:ChanceContraintSafety} 
	\begin{gathered}
		\prob( \iota_k | b_k) \underset{{\bar{z}_{k+1}}}{\mathbb{E}}\Big[ \prob(\iota_{k+1} |\bar{b}_{k+1}) \underset{{\bar{z}_{k+2}}}{\mathbb{E}}\Big[ \prob(\iota_{k+2} |\bar{b}_{k+2}) \cdots \\
		\prob(\iota_{k+L-1} |\bar{b}_{k+L-1})\underset{{\bar{z}_{k+L}}}{\mathbb{E}} \!\Big[ \prob(\iota_{k+L} |\bar{b}_{k+L})   \Big|a_{k+L-1}\!,\bar{b}^{\mathrm{safe}}_{k+L-1}\! \Big]  \cdots\Big|a_{k+1},\bar{b}_{k+1}^{\mathrm{safe}} \Big]  \Big|a_{k},\bar{b}_{k}^{\mathrm{safe}} \Big].
	\end{gathered}	 
\end{equation}
 
\subsubsection{Undesirable Behavior due to Subtree resolution and inability to control on the level of future beliefs} \label{sec:SubtreeExplanation}
Comparing \eqref{eq:OurSafety} and \eqref{eq:ChanceContraintSafety}, we elicit that our formulation \eqref{eq:OurSafety} is actually thresholding each posterior belief in the tree while the \eqref{eq:ChanceContraintSafety} solely the subtrees. 
For clearly seeing the differences let us focus on the myopic setting ($L=1$).	Let us present our formulation (Probabilistic Constraint) versus the conventional formulation (Chance Constraint):
\begin{align}
	&\text{\texttt{PC}:  } \mathbf{1}_{\left\{ \prob(\iota_k | b_{k}) \geq \delta \right\}} \cdot\underset{z_{k+1}}{\mathbb{E}}\Big[\mathbf{1}_{\left\{ \prob(\iota_{k+1}| \psi(b_k, a_k,z_{k+1})) \geq \delta\right\}}\Big|b_k, a_k\! \Big] \!\! \geq 1- \epsilon , \label{eq:OurSafetyMyopic} \\
	& \text{\texttt{CC}:    }  \prob(\iota_k | b_{k}) \cdot \underset{\bar{z}_{k+1}}{\mathbb{E}}\Big[ \prob(\iota_{k+1} | \psi(b_{k}^{\mathrm{safe}},a_{k},\bar{z}_{k+1}))\Big|b_{k}^{\mathrm{safe}}, a_{k} \Big] \geq \delta. \label{eq:ChanceSafetyMyopic}
\end{align}
From the above we immediately  see that our approach (\texttt{PC}) is truly distribution aware as it counts the number of safe posteriors because of the indicator outside the inequality involving the probability value. In contrast, the conventional formulation (\texttt{CC}) merely averages the posterior probabilities and asks if on average they are larger than or equal to $\delta$. 

Let us give a specific example. Suppose the belief $b_k$ is safe. 
Assume that $\delta = 0.7$ and we have three equiprobable observations  in a myopic setting such that $\prob(\mathbf{1}\left\{ x_{k+1}\!\in\!  \mathcal{X}^{\mathrm{safe}}_{k+1}\right\} | \psi(b_{k}^{\mathrm{safe}},a_{k},z^j_{k+1}))$ equals $0.1$, $1.0$, $1.0$ for $j=1, 2, 3$ respectively. On average we have  exactly $0.7$ such that \eqref{eq:ChanceSafetyMyopic}  is fulfilled.  However one belief is extremely unsafe. In contrast, as our formulation \eqref{eq:OurSafetyMyopic} is distribution-aware, it is aware that only two out of the three observation sequences satisfy the constraint. For example, it will declare (the sampling-based approximation of) \eqref{eq:OurSafetyMyopic} is not satisfied  
if (e.g.) we select $\epsilon = 0$ and $\delta=0.7$. 
\subsubsection{Undesirable Behavior due to growing horizon}	\label{sec:ScalingExplanation}
The reformulation \eqref{eq:ChanceSafety} of the conventional constraint allows to reveal its another undesired characteristic.   
With a growing horizon $L$, the probability $\prob(\bigwedge_{i=k}^{k+L}\mathbf{1}_{\left\{ x_i \in   \mathcal{X}^{\mathrm{safe}}_i\right\}}| b_k, \pi)$ will decay. 
To see that explicitly let us start from the terminal time index $k+L$. We observe that
\begin{align}
	\prob(\iota_{k+L-1} |\bar{b}_{k+L-1})\max_{z_{k+L}}\!\prob(\iota_{k+L} |\bar{b}_{k+L}) \geq \prob(\iota_{k+L-1} |\bar{b}_{k+L-1})\!\underset{{\bar{z}_{k+L}}}{\mathbb{E}}\!\Big[\! \prob(\iota_{k+L} |\bar{b}_{k+L})\Big|a_{k+L-1},\bar{b}_{k+L-1}^{\mathrm{safe}}\! \Big] \geq \delta  . \nonumber
\end{align} 
Continuing until the present time we obtain the product of probabilities of multiplicands, making it harder and harder to fulfill the constraint \eqref{eq:SotaSafety} with growing horizon.   This requirement making each belief in the belief tree from which emanates an action to fillfill much harder constraint (the threshold is much larger than $\delta$).  In fact, we never now which actual threshold each posterior belief shall be larger from. 
In contrast, in our formulation \eqref{eq:ProbSafeTraj} we require being larger or equal to $\delta$ solely from the multiplicands of \eqref{eq:ProbSafeTraj}, so  it does not suffer from such a problem and scales to the growing horizon.
\subsubsection{Summary}
To summarize, the conventional collision avoidance constraint \eqref{eq:ChanceSafety} has several key differences versus ours \eqref{eq:OurSafety}.    
\begin{enumerate}
	\item Instead of looking into safe state trajectories we are dealing with safe posterior beliefs trajectories. Our approach uses same  distribution of the observations each step ahead as for the reward $p(z'|b, a)$ while the chance constraint builds upon  $\probd(z'|a,b,\mathbf{1}\left\{ x\!\in\!  \mathcal{X}^{\text{safe}}\right\})$. The advantage of chance constraint is that only the safe portion of the belief is propagated into future (with an action and the observation). Meaning the faulty trajectory is not continued after the unsafe event.
	\item Our constraint formulation is distribution-aware, while the conventional is addressing this aspect in a limiting way. To repharse that, chance constraint operates with the resolution of the subtrees and it is oblivious to the posterior beliefs themselves (Fig.~\ref{fig:ChanceConstraint}). In addition the chance constraint is lacking the control capabilities. Increasing $\epsilon>0$ is not possible in chance constrained setting.   
	\item When the horizon grows the conventional constraint becomes more and more conservative up to violation. In other words the formulation does not scale with horizon, unless $\delta$ is appropriately adjusted. Our formulation does not suffer from this limitation.  For proper comparison with our formulation, we propose to do scaling to $\delta$. Our baseline is Alg.~\ref{alg:constrainedSSbaseline} which will become apparent shortly. It can be applied with scaling or without. To the best of our knowledge it is a novel formulation on its own. 
	\item In our formulation the distribution of future beliefs/observations, and belief update are matched in objective \eqref{eq:ConstrObj} and in the outer constraint \eqref{eq:OuterConstr}, as opposed to chance constrained formulation \eqref{eq:SotaObj} and \eqref{eq:SotaSafety}.  
\end{enumerate}

Before we move to the next section let us mention that a specific variation of \eqref{eq:SampleOuterConstraint} in the context of safety with $\epsilon=0$ is 
\begin{align}
	\forall j = 1 \dots m  \quad  \bigg(\prod_{\ell=k}^{k+L} \mathbf{1}\left\{ \prob(x_{\ell} \in \mathcal{X}^{\mathrm{safe}}_{\ell}| b^j_{\ell}) \geq \delta\right\} \bigg)= 1,
\end{align} 
where $m=m_d^L$, namely, $m_d$ to the power of $L$, and where $m_d$ is the number of observations expanded from each node of the belief tree as in Alg. \ref{alg:constrainedSS}.

\section{Approach to Chance-constrained Continuous $\rho$-POMDP}\label{sec:ChanceConstrainedContPOMDP}
The investigation of chance constraints merged with a general belief-dependent reward has led us to need an algorithmic extension. 
As mentioned, there are two prominent online approaches for solving a continuous POMDP with belief-dependent rewards in a nonparametric domain: SS \cite{Kearns02jml}, and PFT-DPW \cite{Sunberg17arxiv}. In continuous domains, it is unclear how to apply heuristics guided forward search described by \cite{Santana16aaai}. Instead of using the heuristics, we utilize the Bellman principle to resolve that issue. Moreover, as we observe from equation \eqref{eq:ChanceSafety} the observations for the objective \eqref{eq:SotaObj} and the chance constraint \eqref{eq:ChanceSafety} have different distributions as well as the belief updates.

In \cite{Santana16aaai}, this is addressed by considering a discrete and finite observation space and exhaustively expanding all the observations. Such an approach is not possible in a continuous setting. To tackle this issue, we resort to importance sampling such that only a single set of observations is maintained. Let us emphasize that such a problem was not addressed so far. Note that this issue does not exist in our probabilistic approach Alg.~\ref{alg:constrainedSS}. Moreover, the disparity of the belief updates was not addressed at all.

Next, we contribute an {\bf importance sampling} based approach for chance-constrained {\bf continuous} POMDP and its efficient approximation.

\subsection{Importance Sampling Approach for Chance-constrained Continuous POMDP}\label{sec:Importance}
As we have seen in  Lemma \ref{lem:AvSafePosteriors} and equation \eqref{eq:Qfunc}  the distributions of the observations of the chance constraint and the action value function are different, as well as the belief update.  

Let us observe a myopic setting. Since  we draw observations sequentially the extension to an arbitrary horizon is straightforward. In case of the objective function, the desired probability density is $\probd(z_{i+1} | b_i, a_i)$, whereas for the chance constraint we are dealing with $\probd(\bar{z}_{i+1} | b_i, \mathbf{1}\{x_i \in \mathcal{X}_i^{\mathrm{safe}}\}, a_i)$.  As a result of this discrepancy there are two different distributions of future observations. We note that our probabilistic constraint formulation (see Section \ref{sec:SimAndDiff}) does not exhibit this discrepancy.
With chance constraints, we should have thus sampled from each and effectively constructed two belief trees. Putting aside that it would be an enormous computational burden, the question how to apply a consistent policy in both trees requires clarification.  

To avoid this, we suggest to construct a single belief tree where observations are sampled from  $\probd(z_{i+1} | b_i, a_i)$  and properly re-weighted via Importance Sampling (IS) for the evaluation of the chance constraint.  

Importantly, $\probd(\bar{z}_{i+1} | b_i, \mathbf{1}\{x_i \in \mathcal{X}_i^{\mathrm{safe}}\}, a_i)$ is {\bf absolutely continuous} with respect to $\probd(z_{i+1} | b_i, a_i)$. 
\begin{lem}[Absolute Continuity] \label{lem:absCont}
	$\probd(\bar{z}_{i+1} | b_i, \mathbf{1}\{x_i \in \mathcal{X}_i^{\mathrm{safe}}\}, a_i) \ll \probd(z_{i+1} | b_i, a_i)$.	
\end{lem}
\noindent We provide the proof in Appendix~\ref{proof:absCont}. 
  Since the absolute continuity holds, we can safely use IS.

Specifically, suppose we sampled $m_d$ samples $\{z^j_{i+1}\}_{j=1}^{m_d} \sim \probd(z_{i+1} | b_i, a_i)$. From now on, we can think about
\begin{align}
	\hat{\probd}_{(m_d)}(z_{i+1} | b_i, a_i) = \frac{1}{m^d}\sum_{j=1}^{m_d} \delta(z_{i+1}- z^j_{i+1}),
\end{align}
as the density of the discrete probability (Fig.~\ref{fig:CCImportance}).   Leveraging IS, we obtain the desired probability density utilizing the same samples through the following  manipulation 
\begin{align}
	&\hat{\probd}_{(m_d)}(\bar{z}_{i+1} | b_i, \mathbf{1}_{\{x_i \in \mathcal{X}_i^{\mathrm{safe}}\}}, a_i) =\frac{1}{\sum_{j=1}^{m_d} w^{\bar{z},j}_{i+1}} \sum_{j=1}^{m_d} w^{\bar{z},j}_{i+1} \delta(\bar{z}_{i+1}- z^j_{i+1}), \label{eq:ImportanceApprox}
\end{align}
where the $j$-th weight  is given by
\begin{align}	
	w^{\bar{z},j}_{i+1} =\frac{1}{m_d} \frac{\probd(\bar{z}_{i+1} =  z^j_{i+1}| b_i, \mathbf{1}\{x_i \in \mathcal{X}_i^{\mathrm{safe}}\}, a_i)}{\probd(z_{i+1} =  z^j_{i+1}| b_i, a_i)}.   
\end{align}
In Appendix~\ref{sec:ImportanceWeights} we specify expressions for the nominator and denominator. 

Let us clarify again that with the proposed IS-based approach, the sampled observations are used for both the objective and the chance constraint (Fig.~\ref{fig:CCImportance}). However, for the chance constraint we re-weight the samples using Importance weight to obtain the correct expected value according to \eqref{eq:ChanceSafety}.

\begin{figure}[t] 
	\centering
	\begin{minipage}[t]{0.4\textwidth}
		\centering 
		\includegraphics[width=\textwidth]{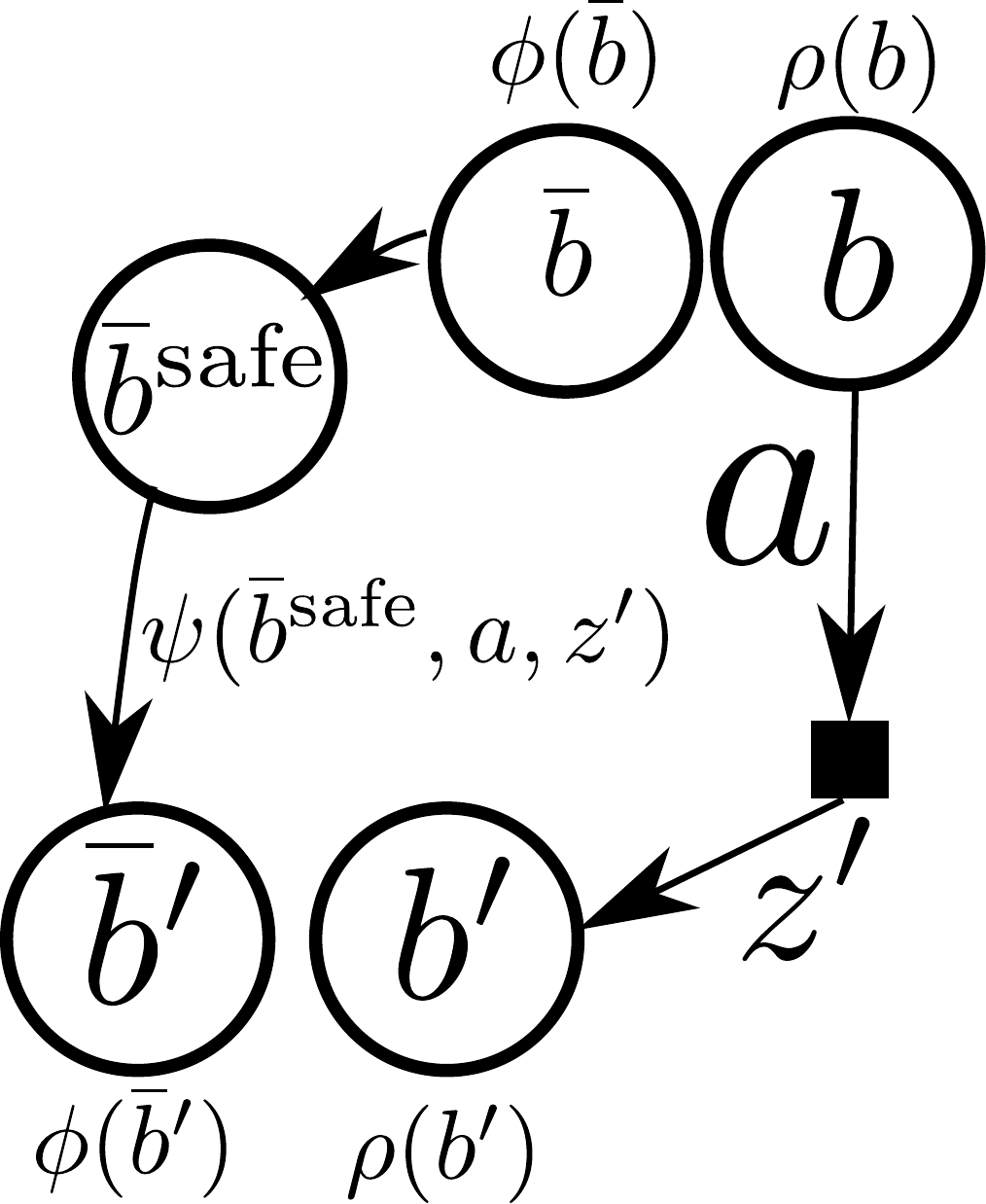}
		\subcaption{}
		\label{fig:CCImportance}
	\end{minipage}%
	\hfill
	\begin{minipage}[t]{0.4\textwidth}
		\centering 
		\includegraphics[width=\textwidth]{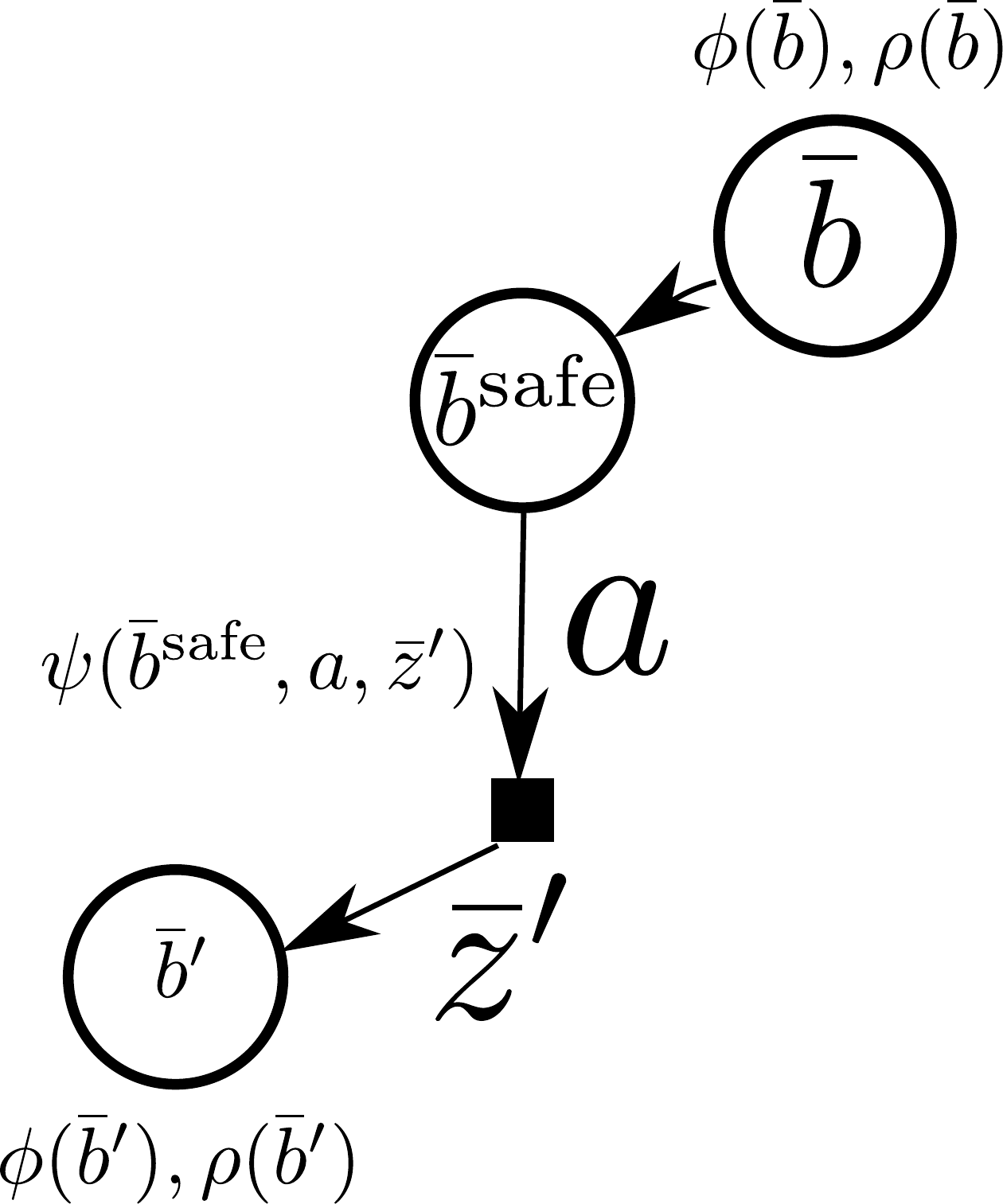}
		\subcaption{}
		\label{fig:CCNoImportance}
	\end{minipage}  
	\caption{On both illustrations the operator $\phi$ is expectation as showed in eq \eqref{eq:SafetyOperator}. \textbf{(a)} Visualization of the belief tree obtained by the importance sampling (Section \ref{sec:Importance}). For actual belief from real world inference in time $k$, we set $b_{k} = \bar{b}_k$.\textbf{(b)} No importance sampling. Same distribution of the observations and belief update is  used for belief dependent rewards and chance constraint.}
\end{figure} 

\subsection{Necessary Condition proposed by \cite{Santana16aaai} for Feasibility of Chance Constraint} \label{sec:NecessaryPruningCond}
Moreover, we also extend the necessary condition for feasibility of chance constraint proposed by \cite{Santana16aaai} to continuous spaces through importance sampling. This is one of the building blocks of our approach to solve continuous belief dependent chance constrained POMDP, see Alg.~\ref{alg:constrainedSSbaselineImportance} and Alg.~\ref{alg:constrainedSSbaseline}. We endow our IS approach with a pruning mechanism.

The paper \cite{Santana16aaai} utilizes the necessary condition for feasibility of an action. Let us restate their future execution risk definition $\mathrm{er}(b_{k+1},\pi)$ and the constraint
\begin{align}
	&\mathrm{er}(b_{k+1},\pi) \bydef 1 -\prob(\mathbf{1}_{\left\{\tau \in  \times_{i=1}^L \mathcal{X}_{k+i}^{\text{safe}}\right\}}| b_{k+1}, \pi) \leq \Delta \bydef 1-\delta. \label{eq:FutureRisk}
\end{align}
The risk at the $k$-th time step $r_b(b_k)$ is defined by 
\begin{align}
	\prob\Bigg(\mathbf{1}\left\{ x_k\!\in\!  \mathcal{X}^{\text{safe}}_k\right\} | b_k \Bigg) = \int_{x_k} b_k(x_k) \mathbf{1}\left\{ x_k\!\in\!  \mathcal{X}^{\text{safe}}_k\right\}\mathrm{d} x_k  = 1-r_b(b_k)
\end{align}
For completeness let us present the following Lemma, which is merely rewriting with our notations the pruning condition from \cite{Santana16aaai}, extended to the continuous spaces with importance sampling.
\begin{lem}[Necessary Condition for Feasibility of Chance Constrained action] \label{lem:NecessaryChance}
	Fix $0 \leq \Delta \leq 1$ and suppose that $\mathrm{er}(b_{k},\pi) \leq \Delta$.  The following holds $\forall i \in 1:m$
	\begin{align}
		&    r_b(\bar{b}^i_{k+1}) \leq \frac{1}{w^{\bar{z},i}_{k+1}}\Bigg(\frac{\Delta - r_b(\bar{b}_k)}{(1 -r_{b}(\bar{b}_k) )} - \sum_{\substack{{j=1} \\ {j\neq i}}}^{m} w^{\bar{z},j}_{k+1} r_b(\bar{b}^j_{k+1})\Bigg),  \label{eq:ChanceConstrPruning}
	\end{align}
\end{lem}	
We provide the proof in  Appendix~\ref{sec:ExecutionRiskBound}. Let us emphasize that our contribution here is solely the extension of this condition to continuous spaces via importance sampling.

Eq.\eqref{eq:ChanceConstrPruning} is used in our Alg.~\ref{alg:constrainedSSbaselineImportance} and further in Alg.~\ref{alg:constrainedSSbaseline}, specific lines $24$ and $21$ respectively.
 
We conclude this section by observing that it is possible that the action violates the chance constraint but  \eqref{eq:ChanceConstrPruning} still holds. This is in strike contrast to our probabilistic constraint pruning, as we proved in Theorem~\ref{thm:pruning}.

\subsection{Approximation} \label{sec:ChanceApprox}
The IS approach introduced in section \ref{sec:Importance} converges to the theoretical solution when $m_d \to \infty$. However, the mechanics of IS introduces a computational burden. To ameliorate the situation from the computational point of view let  us ask another question. Can we relinquish the requirement of IS? 

Specifically, say, we are using $\probd(\bar{z}_{i+1} | b_i, \mathbf{1}\{x_i \in \mathcal{X}_i^{\mathrm{safe}}\}, a_i) \quad i=k:k+L-1$ and corresponding beliefs $\bar{b}_{i+1}$ for the calculation of belief dependent reward. In other words we change the conventional objective as such.  Instead of using 
\begin{align}
	V^L(b_k, \pi)=\rho(b_k) + \underset{{\color{inkscapePurple} z_{k+1}}}{\mathbb{E}}\Big[ \rho(b_{k+1}) + \underset{{\color{inkscapePurple}z_{k+2}}}{\mathbb{E}}\Big[ \rho(b_{k+2}) \dots \Big|a_{k+1},b_{k+1} \Big]  \Big|a_{k},b_{k} \Big],		\label{eq:ValueBellman}
\end{align}
we use the distribution of the observations and the belief update from the chance constraints 
\begin{align}
	U^L(b_k, \pi)=\rho(b_k) + \underset{{\color{inkscapePurple}{\bar{z}_{k+1}}}}{\mathbb{E}}\Big[ \rho({\color{emeraldgreen}\bar{b}_{k+1}}) + \underset{{\color{inkscapePurple}{\bar{z}_{k+2}}}}{\mathbb{E}}\Big[ \rho({\color{emeraldgreen} \bar{b}_{k+2}}) \dots \Big|a_{k+1},{\color{emeraldgreen} \bar{b}_{k+1}^{\mathrm{safe}}} \Big]  \Big|a_{k},{\color{emeraldgreen}\bar{b}_{k}^{\mathrm{safe}}} \Big]. \label{eq:MatchedChanceObj}
\end{align}
The above approximation can be interpreted as follows. Although we calculate the belief dependent rewards on the entire belief, following the belief dependent  reward calculation only the safe state particles of the posterior belief are pushed forward in time with action and the observation. This behavior is identical to that we obtained in the chance constraint \eqref{eq:ChanceSafety}. We face matched distribution of future beliefs in chance constraints and in the objective \eqref{eq:MatchedChanceObj}.   

The benefit of such an approximation with significantly faster decision making. Further we  demonstrate empirically the substantial acceleration  with good performance quality. 

What will be the impact of this approximation on decision making?  This question has not been addressed to the best of our knowledge. Next, we analyze the above approximation.

\subsubsection*{Analysis}

To properly analyze the situation let us focus on the single future step ahead.  Namely, in the proposed objective \eqref{eq:MatchedChanceObj} we have 
\begin{align}
\underset{\bar{z}_{\ell+1}}{\int} \rho(\psi(\bar{b}^{\mathrm{safe}}_{\ell}, a_{\ell},  \bar{z}_{\ell+1} ))
{\color{inkscapePurple}{\probd(\bar{z}_{\ell+1}| \bar{b}_{\ell}, \mathbf{1}\{x_{\ell} \in \mathcal{X}^{\mathrm{safe}}\}, a_{\ell})}} \mathrm{d} \bar{z}_{\ell+1}, \label{eq:ChanceObjAnalysis}
\end{align} 
versus the conventional \eqref{eq:ValueBellman} in accordance to    
\begin{align}
\underset{z_{\ell+1}}{\int}\rho(\psi(b_{\ell}, a_{\ell},  z_{\ell+1} )) {\color{inkscapePurple}{\probd(z_{\ell+1}| b_{\ell}, a_{\ell})}} \mathrm{d} z_{\ell+1}.	\label{eq:ObjAnalysis}
\end{align}
For a rigorous derivation please refer to Appendix~\ref{sec:ObjectivesBMDP}.  This discrepancy is a result of the separation, in the planning phase, of the chance constraint satisfaction (falling trajectories are not pushed forward in time) and the future return maximization (regular belief/observations in the belief tree).   

Our vision is that on the small extent of the quality of the optimal solution, utilization of the distribution \eqref{eq:ChanceObjAnalysis} for the reward will accelerate the performance by avoiding the need for IS. In other words, we suggest, as an approximation, to utilize \eqref{eq:ChanceObjAnalysis} \emph{also} for the reward.

Let us now turn to the analysis of such an approximation. In fact, the dependence on different beliefs $\bar{b}_{\ell}$ in \eqref{eq:ChanceObjAnalysis} and $b_{\ell}$ in \eqref{eq:ObjAnalysis} complicates the situation. We conduct an analysis for horizon two ($L=2$). 

We pinpoint that we start from the actual safe belief $b_{k}= \bar{b}^{\mathrm{safe}}_k$. Therefore, in the first step ahead, the way of pushing forward the belief with an action and observation is identical in two cases. This motivates us to assume $b_{\ell} = \bar{b_{\ell}}$, and instead of 	$
\underset{\bar{z}_{\ell+1}}{\mathbb{E}}\Big[ \rho(\bar{b}_{\ell+1})  \Big|a_{\ell},{\color{red}{\bar{b}_{\ell}}}, \mathbf{1}\{x_{\ell} \in \mathcal{X}^{\mathrm{safe}}\} \Big] $ we shall examine $	\underset{\bar{z}_{\ell+1}}{\mathbb{E}}\Big[ \rho(\bar{b}_{\ell+1})  \Big|a_{\ell},b_{\ell}, \mathbf{1}\{x_{\ell} \in \mathcal{X}^{\mathrm{safe}}\} \Big]$ versus $\underset{z_{\ell+1}}{\mathbb{E}}\Big[ \rho(b_{\ell+1})  \Big|a_{\ell},b_{\ell} \Big]$. 

We need to analyze the influence of different observation likelihood (marked by {\color{inkscapePurple}{purple}} color) as well as the fact    $\psi(b_{\ell}, a_{\ell},  z_{\ell+1})  \neq \psi(b^{\mathrm{safe}}_{\ell}, a_{\ell},  \bar{z}_{\ell+1}) $ even if $\bar{z}_{\ell+1} = z_{\ell+1}$ (integration variable).

We now analyze the above under several assumptions; analysis of the full case, as well as of an arbitrary horizon, is left for future work. We are going to show that under maximum likelihood observation assumption the impact of our approximation is negligible up until horizon $L=2$.

\begin{lem}[Maximum likelihood observation] \label{lem:ml}
	Assume that $\argmax_{x_{\ell}} b_{\ell}(x_{\ell}) \in \mathcal{X}^{\mathrm{safe}}$. We obtain the same maximum likelihood observation $z^{\mathrm{ML}}_{\ell+1}$ in both cases.
\end{lem}
The reader can find the proof in Appendix~\ref{proof:ml}.	
We shall proceed now to comparing  $\rho(\psi(b_{\ell}, a_{\ell},z^{\mathrm{ML}}_{\ell+1}))$ versus $\rho(\psi(b^{\mathrm{safe}}_{\ell},  a_{\ell},z^{\mathrm{ML}}_{\ell+1}))$. Suppose further that $\rho(b) = \mathbb{E}[r^x(x)|b]$. 

\begin{lem}[Objectives comparison] \label{lem:objComp}
	\begin{align}
		&\mathbb{E}[r^x(x)|\psi(b_{\ell},  a_{\ell},z^{\mathrm{ML}}_{\ell+1})\big] = \nonumber\\
		&\frac{\probd\Bigg(z^{\mathrm{ML}}_{\ell+1}\Bigg|b_{\ell}, a_{\ell}, \mathbf{1}\left\{ x_{\ell} \in  \mathcal{X}^{\mathrm{safe}}\right\} \Bigg){ \color{purple} \prob\big(\mathbf{1}\left\{ x_{\ell} \in  \mathcal{X}^{\mathrm{safe}}\right\} \big)|b_{\ell}, a_{\ell}\big)}}{\probd(z^{\mathrm{ML}}_{\ell+1}|b_{\ell}, a_{\ell} )}\mathbb{E}[r^x(x)|\psi(b^{\mathrm{safe}}_{\ell},  a_{\ell},z^{\mathrm{ML}}_{\ell+1})\big]+\\
		&\frac{\probd\Bigg(z^{\mathrm{ML}}_{\ell+1}\Bigg|b_{\ell}, a_{\ell}, \mathbf{1}\left\{ x_{\ell} \notin  \mathcal{X}^{\mathrm{safe}}\right\} \Bigg)\prob\big(\mathbf{1}\left\{ x_{\ell} \notin  \mathcal{X}^{\mathrm{safe}}\right\} \big)|b_{\ell}, a_{\ell}\big)}{\probd(z^{\mathrm{ML}}_{\ell+1}|b_{\ell}, a_{\ell} )}\mathbb{E}[r^x(x)|\psi(b_{\ell}, \mathbf{1}\left\{ x_{\ell} \notin  \mathcal{X}^{\mathrm{safe}}\right\}, a_{\ell},z^{\mathrm{ML}}_{\ell+1})\big]
	\end{align}	
\end{lem}
The reader can find the proof in Appendix~\ref{proof:ObjComp}.
As $\delta$ (from section \ref{sec:RelationToChance})  approaches one from the left  the first summand of {\bf feasible} action converges to $\mathbb{E}[r^x(x)|\psi(b_{\ell},  a_{\ell},z^{\mathrm{ML}}_{\ell+1})\big]$ and second to zero. This is happening due to the requirement $\prob\big(\mathbf{1}\left\{ x_{\ell} \in  \mathcal{X}^{\mathrm{safe}}\right\} \big)|b_{\ell}, a_{\ell}\big) \geq \delta$ for nonterminal beliefs. Actually as we explained in section \ref{sec:ScalingExplanation} this probability have to be much larger than $\delta$ 
	
Since in practice $\delta$ is close to one from the left due to requirement to be safe with high probability, our approximation is not too harmful as we empirically observe further. 

\subsection{The Algorithms}
\input{./floats/constrainedSSbaselineImportance}
\input{./floats/constrainedSSbaselineNoImportance}
\input{./floats/PruneChance}
Our solver for the conventional Chance-constrained POMDP (Section \ref{sec:Importance}) is formulated as Alg.~\ref{alg:constrainedSSbaselineImportance}. Our efficient variant from Section \ref{sec:ChanceApprox} is summarized in Alg.~\ref{alg:constrainedSSbaseline}. Similar to \cite{Santana16aaai}, the chance constraint in Alg.~\ref{alg:constrainedSSbaselineImportance} and Alg.~\ref{alg:constrainedSSbaseline} is assured  to be fulfilled from any nonterminal belief node in the belief tree until the bottom. Note that since we are using only necessary condition for pruning on the way down the tree (line $24$ in Alg.~\ref{alg:constrainedSSbaselineImportance} and line $21$ Alg.~\ref{alg:constrainedSSbaseline}), we still need to verify the constraint on the way up (line $33$ in Alg.~\ref{alg:constrainedSSbaselineImportance} and line $30$ Alg.~\ref{alg:constrainedSSbaseline})      
Both Algorithms \ref{alg:constrainedSSbaselineImportance} and  \ref{alg:constrainedSSbaseline} utilize the pruning technique described in section \ref{sec:NecessaryPruningCond} and have boolean switch ``sc" which  selects between the scaled versions and the unscaled as explained in section \ref{sec:ScalingExplanation}.

\section{Simulations and Results} \label{sec:Results}
We demonstrate theoretical findings on two problems, navigation to the static goal and target tracking. Both problems are under the umbrella of Belief Space Planning with a given map. We simulate our probabilistically constrained Alg.~\ref{alg:constrainedSS} (PCSS) versus two chance constrained algorithms (CCSS-IS) with importance sampling (Alg.~\ref{alg:constrainedSSbaselineImportance}) and without  (FastCCSS Alg.~\ref{alg:constrainedSSbaseline})  in continuous domains. Our simulation is in an MPC framework, that is ~re-planning after each step. We will compare the running times of the planning sessions, the number of actions that have not been pruned, and cumulative rewards along the simulated execution of the selected online policy which is in fact the algorithm itself (All of our algorithms in this paper are online policies). We also shall report the influence of the scaling parameter on chance-constrained algorithms.  
Our action space is the space of motion primitives of unit vectors $\mathcal{A} = \{  \rightarrow, \nearrow, \uparrow, \nwarrow,\leftarrow,  \swarrow,   \downarrow, \searrow,  Null\}$. For simplicity in both problems our belief dependent reward  is 
\begin{align}
	\rho(b, a, z', b') =\frac{1}{m_x}\sum_{i=1}^{m_x} r^x(x^i, a) \quad x^i \sim b,  
\end{align}
where $m_x$ is the number of the belief particles. However, as we further prove we still can account for uncertainty even with reward being the first moment of the state dependent reward. 

\subsection{Navigation to static Goal}
\begin{figure}[t] 
	\centering
	\begin{minipage}[t]{0.3\textwidth}
		\centering 
		\includegraphics[width=\textwidth]{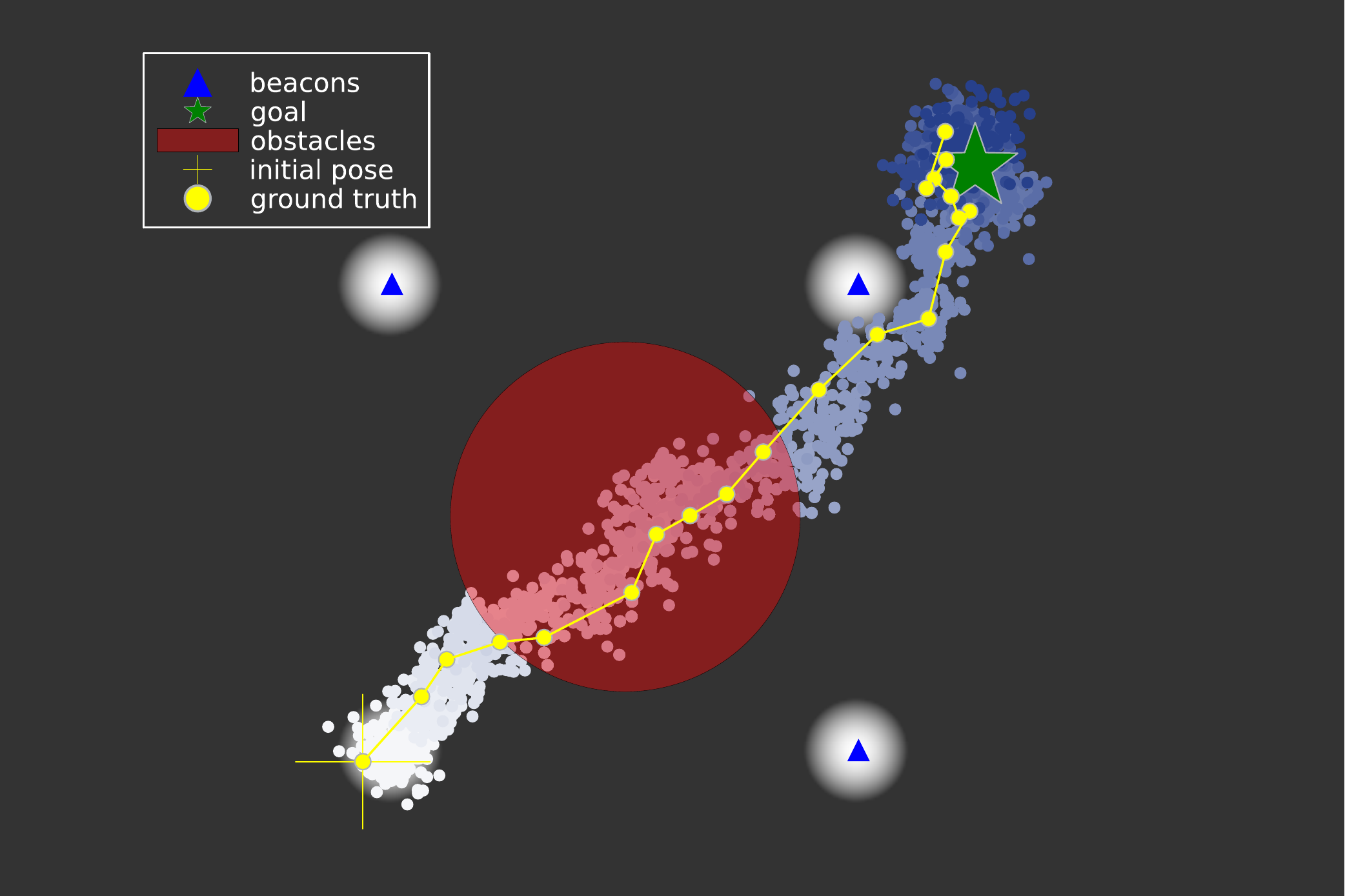}
		\subcaption{}
	\end{minipage}%
	\hfill
	\begin{minipage}[t]{0.3\textwidth}
		\centering 
		\includegraphics[width=\textwidth]{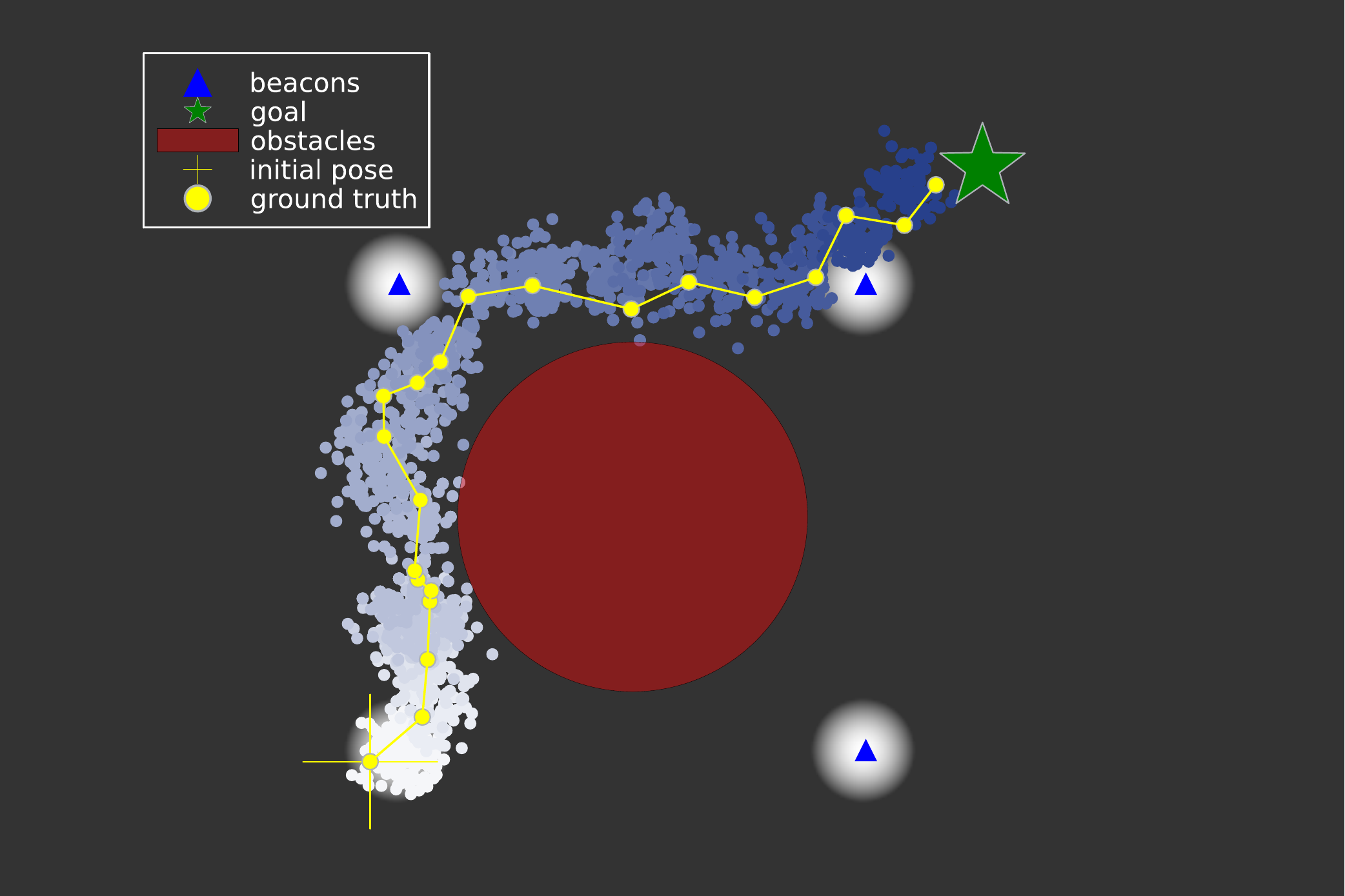}
		\subcaption{}
	\end{minipage} 
	\hfill
	\begin{minipage}[t]{0.3\textwidth}
		\centering 
		\includegraphics[width=\textwidth]{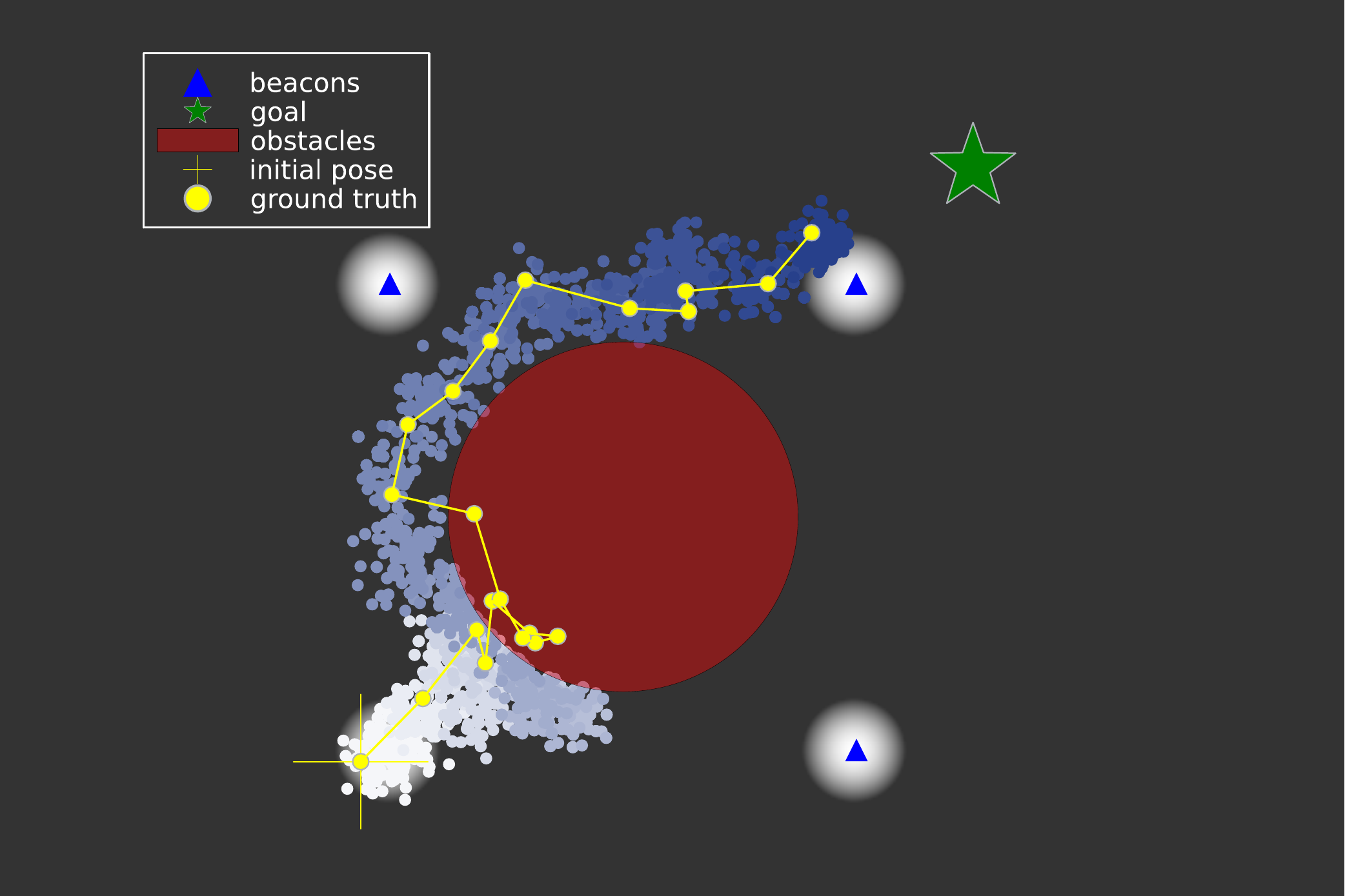}
		\subcaption{}
	\end{minipage}  
	\caption{Visualization of a single trial from $50$ in myopic setting of {\bf navigation to the static goal problem} at our {\bf first map}.The hyperparameters are $m_1=100, L=1,m_x=150$, $\delta=0.8$ : \textbf{(a)} Constraints are deactivated. The robot solves solely the objective \eqref{eq:ConstrObj}, actual belief $b_k$ is not made safe. \textbf{(b)} The PCSS algorithm succesfully avoided the obstacle.  \textbf{(c)} CCSS  with the same seed collided with the obstacle. } \label{fig:MyopicStaticGoal} 
\end{figure} 
Following the previous discussion we are ready to delve into the details of our first problem. We adopt the well known problem of navigation to the goal with collision avoidance. The belief in this problem is maintained over the robot's pose which is $2$-dimensional vector. Our state dependent reward is $r^x(x, a) = -\|x - x^g\|_2^2$.
By $x^g$ we denote the location of the goal. Note that such $\rho(b, a, z', b')$  accounts for belief uncertainty as we show in Appendix \ref{sec:DistanceToGoal}. 

Our obstacle has a circular shape with a center at $x^o$ and radius $r^o$. We approximate the probability of not having the collision by $$\probd\big(\mathbf{1}\left\{x_{k}\!\in\! \mathcal{X}^{\text{safe}}\right\} | b_{k}\big) = 1 - \frac{1}{m_x}\sum_{i=1}^{m_x} \mathbf{1}\left\{ \|x_k^i - x^o\|_2 \leq r^o\right\}.$$ 
 
Motion and observation models, and the initial belief are $\probd_T(\cdot|x, a) = \mathcal{N}(x + a, \Sigma_T)$, $\probd_Z (\cdot | x ; \{x^{b,i}\}_{i=1} )=\mathcal{N}(x, \Sigma_O)$, $b_0 = \mathcal{N}(x_0, \Sigma_0)$ respectively. The robot obtains an observation from the closest beacon.  
The covariance matrices are  diagonal $\Sigma_T = I\cdot \sigma^2_w$ and 
\begin{align}
	\!\Sigma_{O}(x; \{x^{b,i}\}_{i=1})\!\! = \!\! \begin{cases}
		\!\sigma^2_{w}I   \underset{i}{\min }\ d_i, \text{ if } \underset{i}{\min} \ d_i \geq \! r_{\text{min}} \\
		\!\sigma^2_{v}I, \text{else}  \end{cases}  \label{eq:BeaconsObsCov}
\end{align}
where $d_i=\|x\!-\!x^{b,i}\|_2$,  $x^{b,i}$ is the 2D location of the beacon $i$. We set the parameters to be $r_{\text{min}}=0.01$, $\sigma_w^2=0.1$ and $\sigma_v^2=0.01$, $\gamma=0.99$. The initial belief admits Gaussian distribution $\mathcal{N}(x_k; \mu, \Sigma)$ with covariance  $\Sigma = \sigma I = 0.1 \cdot I$ and mean $\mu = (0.0, 0.0)^T$. Initial ground truth state of the robot was set to $x^{\mathrm{gt}}_k = (-0.5, -0.2)^T$.
\subsection{Target Tracking} \label{sec:TargetTrack}
Now we describe the second problem. In this problem we have a moving target in addition to the agent. In this problem the belief is maintained over both positions, the agent and the target. The state dependent reward in this problem  is $r^x(x, a) = -\|x^{\text{agent}} - x^{\text{target}}\|_2^2$. It accounts for the uncertainty of both the target and the agent in similar manner as in previous problem.
Moreover, now we have a squared obstacle. We check collision now according to
$$\probd(\mathbf{1}\left\{x^{\mathrm{agent}}_{k}\!\in\! \mathcal{X}^{\text{safe}}\right\} | b_{k}) = 1 - \frac{1}{m_x}\sum_{i=1}^{m_x} \mathbf{1}\left\{ \|x_k^{\mathrm{agent},i} - x^o\|_\infty \leq r^o\right\},$$
where the $\|\xi\|_{\infty} =\max_{i}|\xi_i| $, where $\xi_i$ is the coordinate $i$ of vector $\xi_i$.
The motion model of the target is identical to the motion model of the agent and follows 
\begin{align}
	\probd_T(\cdot|x, a) = \mathcal{N}(x^{\text{agent}} + a^{\text{agent}}, \Sigma_T) \cdot \mathcal{N}(x^{\text{target}} + a^{\text{target}}, \Sigma_T),
\end{align}	  
where by $x$ we denote the concatenated $\{x^{\text{agent}}, x^{\text{target}}\}$.   
For target action we use a circular buffer with  $\{  \leftarrow, \uparrow\}$ action sequence.  We maintain belief over the agent and the target. For simplicity we assume that in inference as well as in planning session we know the target action sequence.   The observation model is also the multiplication of the observation model from the previous section with the additional observation model due to moving target.  Such as the overall observation model is 
\begin{align}
	\probd_Z (\cdot | x ; \{x^{b,i}\}_{i=1} )=\mathcal{N}(x^{\text{agent}}, \Sigma_O(x^{\text{agent}}; \{x^{b,i}\}_{i=1}))\cdot \mathcal{N}(x^{\text{agent}} -  x^{\text{target}},\Sigma_O(x^{\text{agent}},x^{\text{target}})),
\end{align}
where $\Sigma_O(x^{\text{agent}}; \{x^{b,i}\}_{i=1})$ conforms to \eqref{eq:BeaconsObsCov} and 
\begin{align}
	\Sigma_O(x^{\text{agent}},x^{\text{target}}) = \!\! \begin{cases}
		\!\sigma^2_{w}I   \| x^{\text{agent}} - x^{\text{target}}  \|_2, \text{ if } \| x^{\text{agent}} - x^{\text{target}}  \|_2 \geq \! r_{\text{min}} \\
		\!\sigma^2_{v}I, \text{else}  \end{cases} 
\end{align}

Importantly,  the target does not collides with obstacles, it can fly above.  In this problem we selected the parameters to be $r_{\text{min}}=0.01$, $\sigma_w^2=0.1$ and $\sigma_v^2=0.01$, $\gamma=0.99$. The initial belief admits Gaussian distribution $\mathcal{N}(x_k; \mu, \Sigma)$ with covariance  $\Sigma = \sigma I = 0.01 \cdot I$ and mean $\mu = (\underbrace{0, 0}_{\text{agent}}, \underbrace{10, 0}_{\text{target}} )^T$. Initial ground truth state of the robot and the target was set to $x^{\mathrm{gt}}_k = (\underbrace{-0.5, -0.2}_{\text{agent}}, \underbrace{10, 0}_{\text{target}})^T$.
\subsection{Discussion} 
\begin{figure}[t] 
	\centering
	\begin{minipage}[t]{0.45\textwidth}
		\centering 
		\includegraphics[width=\textwidth]{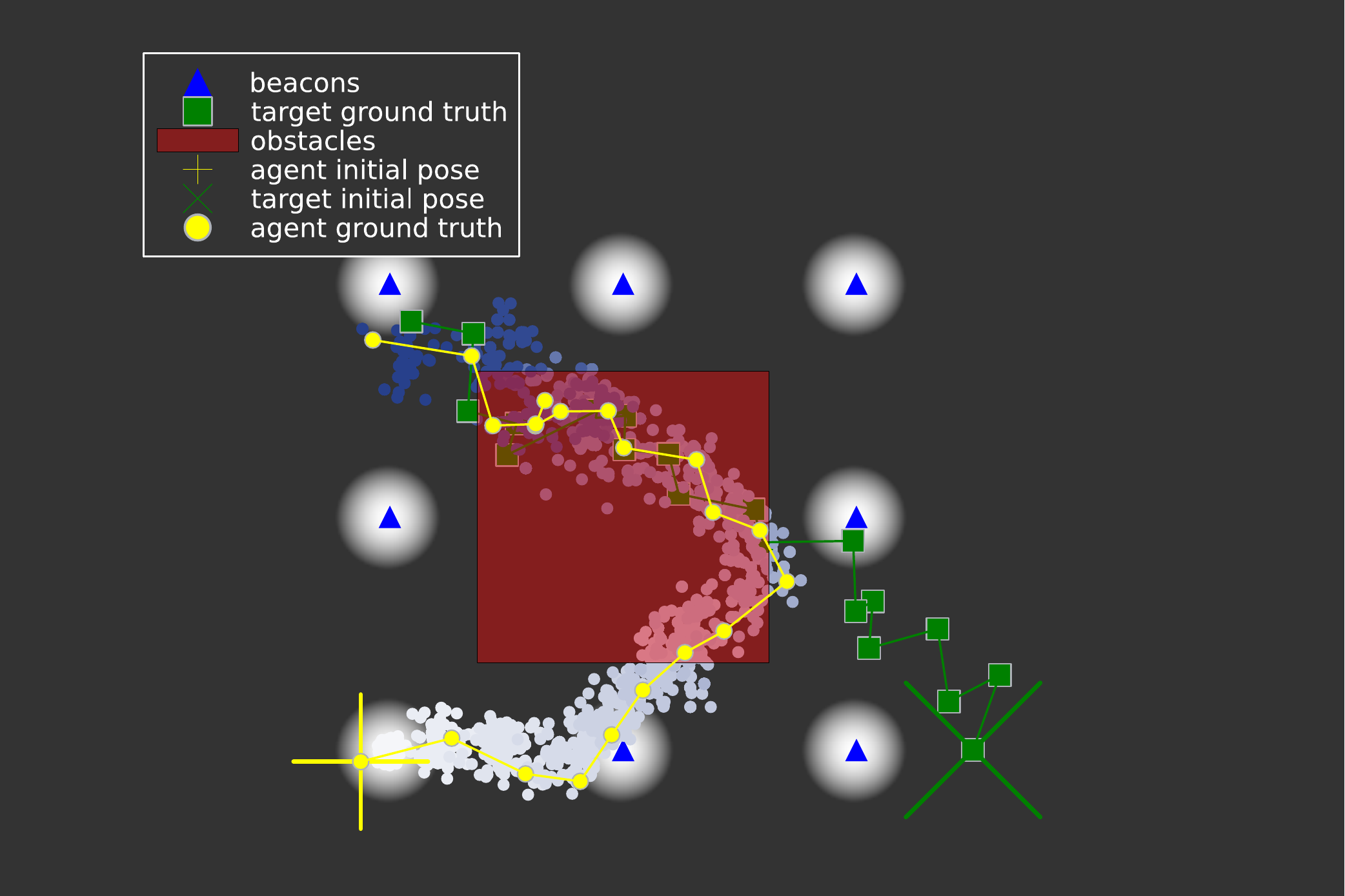}
		\subcaption{}
	\end{minipage}%
	\hfill
	\begin{minipage}[t]{0.45\textwidth}
		\centering 
		\includegraphics[width=\textwidth]{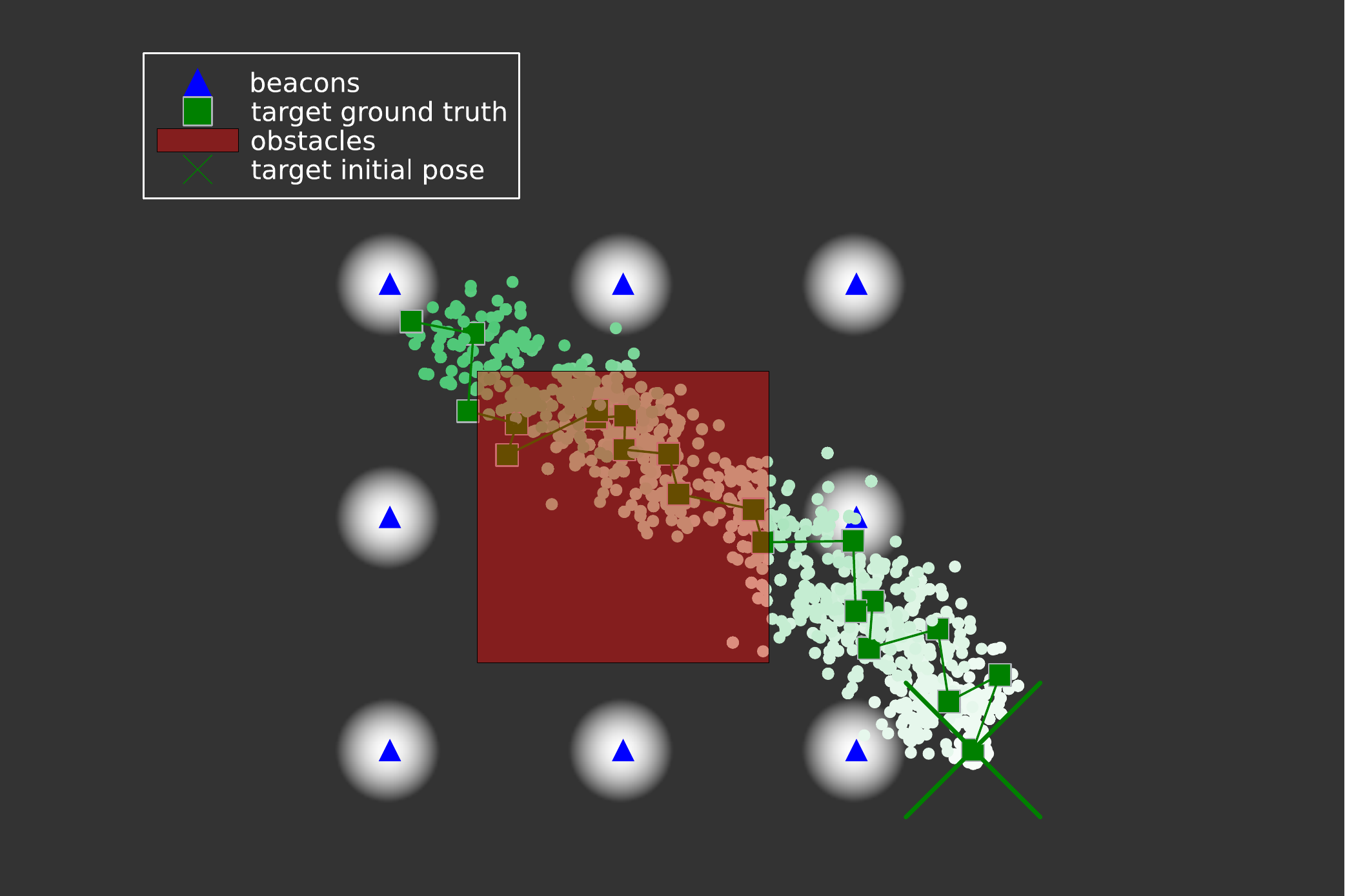}
		\subcaption{}
	\end{minipage} 
	\vfill
	\begin{minipage}[t]{0.45\textwidth}
		\centering 
		\includegraphics[width=\textwidth]{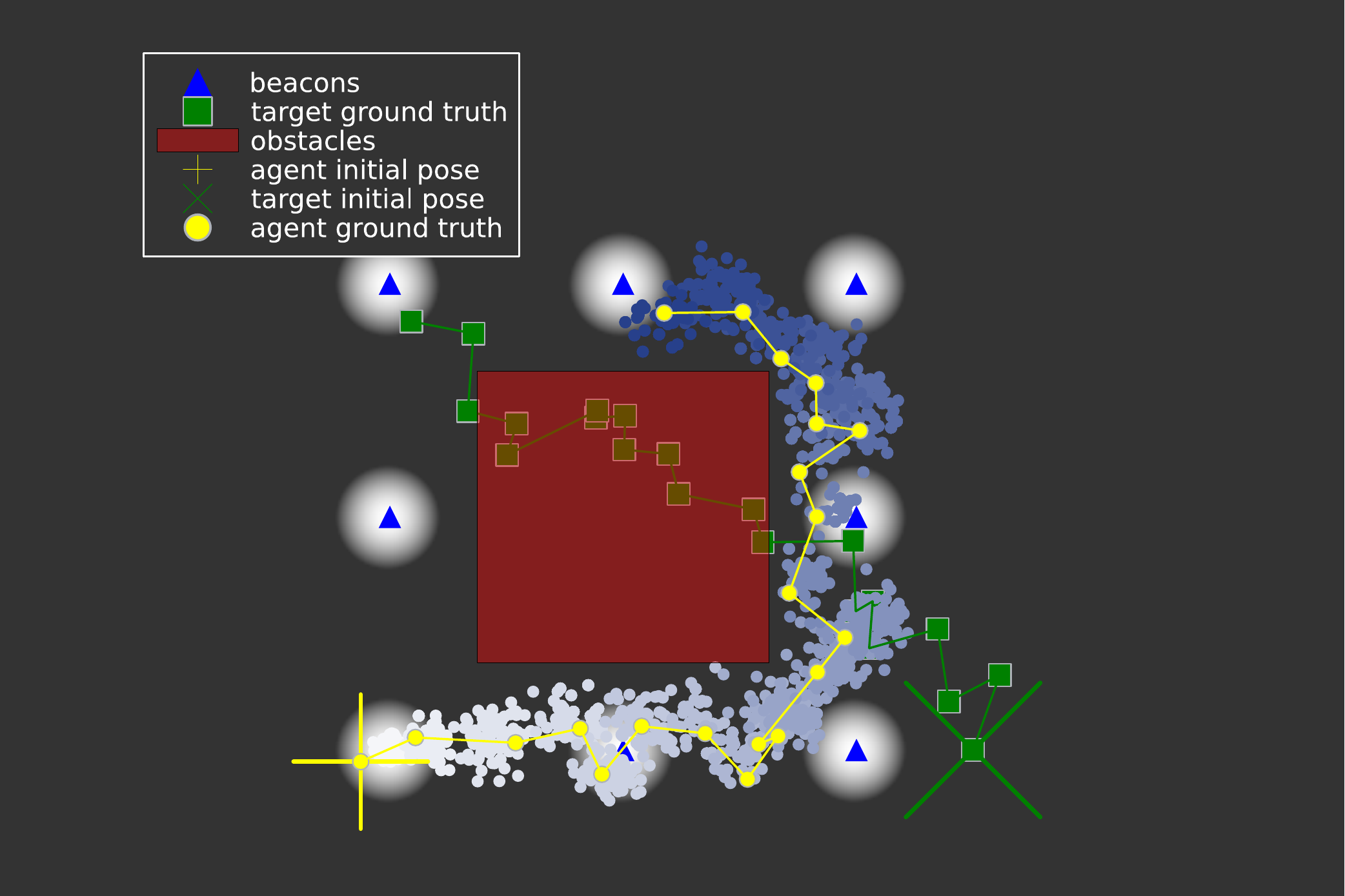}
		\subcaption{}
	\end{minipage}  
		\hfill
	\begin{minipage}[t]{0.45\textwidth}
		\centering 
		\includegraphics[width=\textwidth]{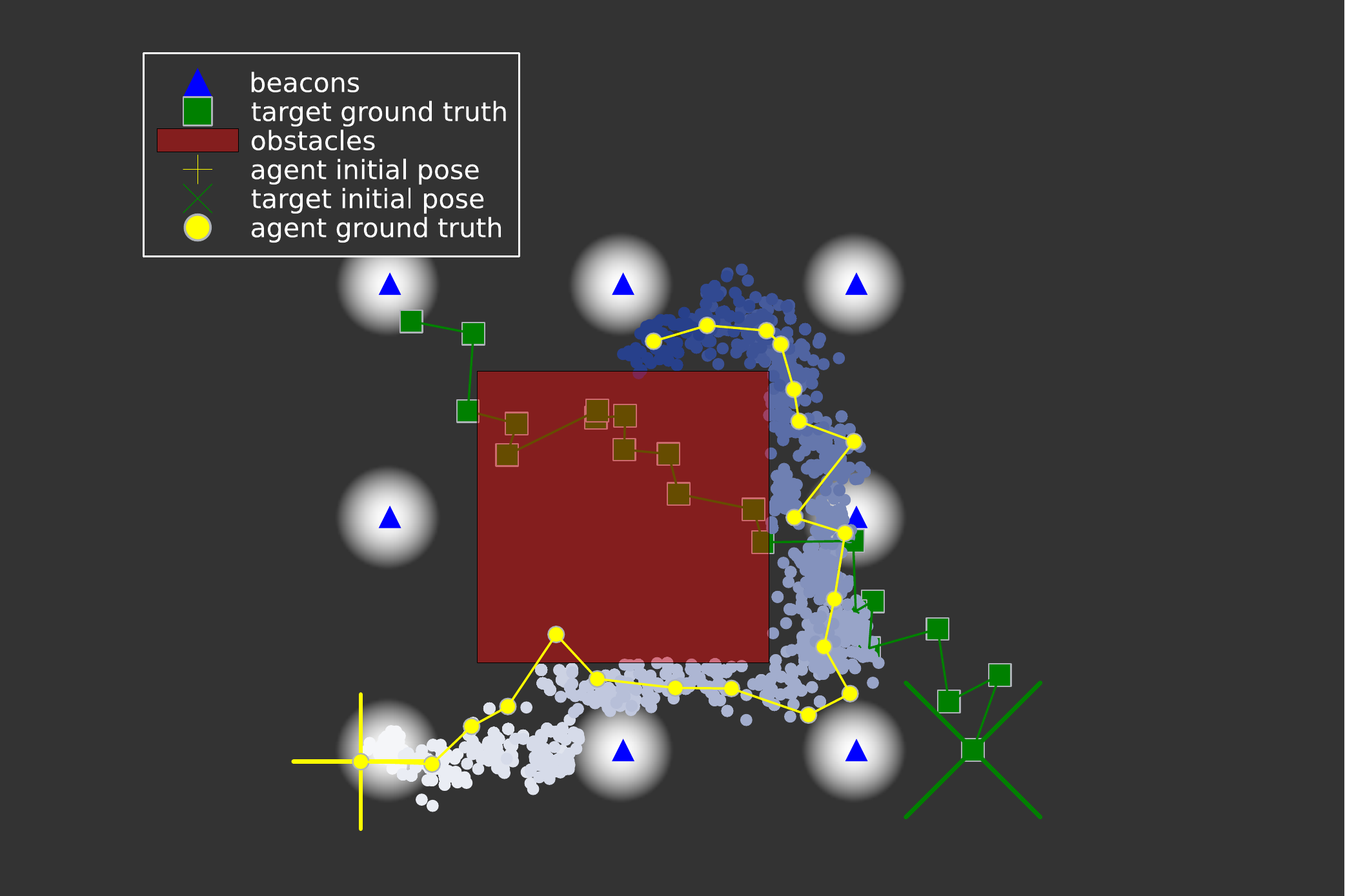}
		\subcaption{}
	\end{minipage}
	\caption{Visualization of a single trial from $50$ in myopic setting of {\bf target tracking problem} at our {\bf second map}. The hyperparameters are $m_1=100, L=1,m_x=150$, $\delta=0.8$ with scaling set off: \textbf{(a)} Constraints are deactivated. The robot solves solely the objective \eqref{eq:ConstrObj}, actual belief $b_k$ is not made safe. \textbf{(b)} Particles of the target corresponding to particles of the (a) illustration. \textbf{(c)} The PCSS algorithm successfully avoided the obstacle.  \textbf{(d)} FastCCSS  with the same seed collided with the obstacle. Note that final belief shown before it has been made safe. Note, the target flies above the the obstacle, as we described in section \ref{sec:TargetTrack}. } \label{fig:MyopicMovingTarget} 
\end{figure} 
\input{./floats/tableAblationStaticGoalMyopicMap1.tex}
For each pair of algorithms we calculate the speedup according to 
\begin{align}
	\frac{t^{\mathrm{baseline}}- t^{\mathrm{algorithm}}}{t^{\mathrm{baseline}}}. \label{eq:speedup}
\end{align} 
The \eqref{eq:speedup} measure saved time relative to the baseline running time.  In the similar manner the number of expanded and not pruned actions $N$ fraction is 
\begin{align}
	\frac{N^{\mathrm{baseline}}- N^{\mathrm{algorithm}}}{N^{\mathrm{baseline}}}.  \label{eq:actions}
\end{align}
Note also that it is possible that algorithms declare that no feasible solution exists or the actual belief can not be made safe, if all samples fall inside the obstacle.

Before we proceed, note that the collision avoidance mechanism is identical in FastCCSS Alg.~\ref{alg:constrainedSSbaseline} and Alg.~\ref{alg:constrainedSSbaselineImportance}.  
Further, we present in depth study of the presented above algorithms with two forms of obstacles, the circle form and the square.

\subsubsection{Myopic Planning}
Let us start from the myopic setting. When the agent plans myopically the FastCCSS and CCSS-IS are identical algorithms, since we make actual belief $b_k$ always safe (See Fig.~\ref{fig:CCImportance}). So there is no point in running the CCSS-IS in myopic setting.  Due to fact that the belief tree is  myopic we observe much more collisions with chance constrained formulation due to averaging of the leaves in shallow myopic belief tree as explained in section \ref{sec:SubtreeExplanation}. For actual results see Table~\ref{tbl:AblationStudyMyopicStaticGoal} and \ref{tbl:AblationStudyMyopicTargetTracking}) . Here we observe substantial speedup of our Alg.~\ref{alg:constrainedSS} versus Alg.~\ref{alg:constrainedSSbaseline}. We demonstrate a single trial of navigation to static goal and target tracking respectively at Fig.~\ref{fig:MyopicStaticGoal}  and Fig.~\ref{fig:MyopicMovingTarget}. In both problems when $\delta =0$, the agent crashed in all $50$ trials. 
\subsubsection{Larger Horizons}
\input{./floats/tableAblationStaticGoalL2Map1.tex}
\input{./floats/tableAblationStaticGoalL2Map1SpeedupActions.tex}
\input{./floats/tableAblationTargetTrackingMyopicMap2.tex}
%
We study three presented algorithms. Both chance constrained approaches have scaling parameter. We first investigate the number of collisions and the actual value function obtained by $50$ executions of $21$ planning sessions and inference steps. In both problems, we face the number of collisions of Alg.~\ref{alg:constrainedSS} similar to unscaled versions of Alg.~\ref{alg:constrainedSSbaseline} and Alg.~\ref{alg:constrainedSSbaselineImportance}, see Table~\ref{tbl:AblationStudyL2StaticGoal} and  Table~\ref{tbl:AblationStudyL2TargetTracking}.   We then examine the running time and expanded actions of three algorithms in both problems 
Table~\ref{tbl:AblationStudyL2SpeedupActionsStaticGoal}
and Table~\ref{tbl:AblationStudyL2SpeedupActionsTargetTracking}.  As we see, a larger speedup in Alg.~\ref{alg:constrainedSS} relative to Alg.~\ref{alg:constrainedSSbaselineImportance}, as well as Alg.~\ref{alg:constrainedSSbaseline} relative to Alg.~\ref{alg:constrainedSSbaselineImportance}. We do not observe a substantial decrease in expected cumulative reward in   Alg.~\ref{alg:constrainedSSbaseline} with any checked $\delta$. The number of expanded actions is always smaller with our pruning in Alg.~\ref{alg:constrainedSS} (Table~\ref{tbl:AblationStudyL2SpeedupActionsStaticGoal}
and Table~\ref{tbl:AblationStudyL2SpeedupActionsTargetTracking}). 

\input{./floats/tableAblationTargetTrackingL2Map2.tex}
\input{./floats/tableAblationTargetTrackingL2Map2SpeedupActions.tex}
\section{Conclusions} \label{sec:Conclusions}
We proposed a novel formulation of belief-dependent probabilistically constrained continuous POMDP. Our formulation allows us, adaptively with respect to observation laces, to accept or reject the candidate policy satisfying or violating the probabilistic constraint. We also uplifted chance-constrained POMDP to continuous domains in terms of states and observations and general belief-dependent rewards. Our simulations corroborate the superiority of our efficient algorithms in terms of celerity. In all simulations we obtained typical speedup of $60\%$. We intend to continue investigating the proposed formulation towards larger horizons using FSSS, as explained in the paper. 
	
\begin{appendices}
\section{Proofs} \label{sec:Proofs}
\subsection{Proof of Theorem~\ref{thm:pruning} Necessary and sufficient condition for feasibility of probabilistic constraint} \label{proof:pruning}
Before we start let us state that by definition, using $c(b^i_{k:k+L};\phi,\delta)=\bigg(\prod_{\ell=k}^{k+L} \mathbf{1}_{\left\{\phi(b^i_{\ell}) \geq \delta \right\}}\bigg)$ holds
\begin{align}
	1 \geq  \frac{1}{m} \sum_{i=1}^{m} c(b^i_{k:k+L};\phi, \delta)= \frac{1}{m} \sum_{i=1}^{m}\bigg(\prod_{\ell=k}^{k+L} \mathbf{1}_{\left\{\phi(b^i_{\ell}) \geq \delta\right\}}\bigg)
\end{align}
Suppose that 
\begin{align}
	1 \geq \frac{1}{m} \sum_{i=1}^{m} c(b^i_{k:k+L};\phi, \delta) \geq 1,
\end{align}
so 
\begin{align}
 \sum_{i=1}^{m}\bigg(\prod_{\ell=k}^{k+L} \mathbf{1}_{\left\{\phi(b^i_{\ell}) \geq \delta\right\}}\bigg) =m \label{eq:keyobservation}
\end{align}

Suppose in contradiction that $\exists i,\ell $ such that $\mathbf{1}_{\left\{\phi(b^i_{\ell}) \geq \delta\right\}} = 0$. We have that 
\begin{align}
	\sum_{i=1}^{m}\bigg(\prod_{\ell=k}^{k+L} \mathbf{1}_{\left\{\phi(b^i_{\ell}) \geq \delta\right\}}\bigg) < m
\end{align}
This proves the first statement. For the second statement we prove the reciprocal implication. Assume that   $\forall i,\ell $ holds $\mathbf{1}_{\left\{\phi(b^i_{\ell}) \geq \delta\right\}} = 1$, we arrived at the fulfilling equation \eqref{eq:keyobservation}. 
\qed
\subsection{Proof of Lemma~\ref{lem:Represent} (representation of our outer constraint).}  \label{proof:Represent}
 To verify the inner constraint $c(b_{k:k+L};\phi, \delta)$ we need to know lace of the beliefs, operator $\phi$ and $\delta$. To rephrase  that   
 \begin{align}
 	\prob\left(c(b_{k:k+L};\phi, \delta) | b_k, \pi, a_k, b_{k+1:k+L}\right) = c(b_{k:k+L}; \phi, \delta).
 \end{align} 
In addition let us state the fact that $p(b_{k+1:k+L}|b_k, \pi, a_k, z_{k+1:k+L})$ is Dirac's delta function. All in all, we can write  
\begin{align}
	&\prob\left(c(b_{k:k+L};\phi,\delta)|b_k, \pi, a_k  \right) =  \int_{\substack{b_{k+1:k+L} \\ z_{k+1:k+L}}} \prob\left(c(b_{k:k+L};\phi,\delta) | b_k, \pi, a_k, b_{k+1:k+L}\right) \cdot \nonumber\\
	& \probd(b_{k+1:k+L}|b_k, \pi, a_k, z_{k+1:k+L})p(z_{k+1:k+L}| b_k, \pi,a_k)\mathrm{d}b_{k+1:k+L}\mathrm{d}z_{k+1:k+L}=\\
	&\mathbb{E}_{z_{k+1:k+L}}\left[c(b_{k:k+L};\phi,\delta)\bigg|b_k, \pi, a_k\right].
\end{align}
\qed
\subsection{Proof of Lemma~\ref{lem:TrajProbab} (probability of the trajectory)}  \label{proof:TrajProbab}
\begin{align}
	&\probd(x_{k:k+L} |b_k, \pi_{k+1:k+L-1}, a_k)= \!\!\!\!\!\!\!\!\!\!\int\limits_{z_{k+1:k+L-1}}\!\!\!\!\!\! {\color{teal}\probd(x_{k:k+L} , z_{k+1:k+L -1 } |  b_k, \pi)}\mathrm{d}z_{k+1:k+L-1}= \\
	& \!\!\!\!\!\!\int\limits_{z_{k+1:k+L-1}}\!\!\!\!\!\! \probd(x_{k+L} | z_{k: k+L -1 }, x_{k:k+L-1},  b_k, \pi) \probd(z_{k:k+L -1 },x_{k:k+L-1} |   b_k, \pi)\mathrm{d}z_{k+1:k+L-1} =\\
	&\!\!\!\!\!\!\int\limits_{z_{k+1:k+L-1}}\!\!\!\!\!\! \probd_T(x_{k+L} | x_{k+L-1},  a_{k+L-1}) \probd(z_{k+L -1} | x_{k:k+L-1}, z_{k+1:k+L-2} ,b_k, \pi) \nonumber \\
	&\probd(x_{k:k+L-1}, z_{k+1:k+L-2} |b_k, \pi)\mathrm{d}z_{k+1:k+L-1}=\\
	&\!\!\!\!\!\!\int\limits_{z_{k+1:k+L-1}}\!\!\!\!\!\! \probd_T(x_{k+L} | x_{k+L-1},  a_{k+L-1}) \probd_Z(z_{k+L -1} | x_{k+L-1}) \nonumber \\
	&{\color{teal}\probd(x_{k:k+L-1}, z_{k+1:k+L-2} |b_k, \pi)}\mathrm{d}z_{k+1:k+L-1}.
\end{align}
We observe the recurrence relation. Overall
\begin{align}
	&\probd(\tau |b_k, \pi_{k+1:k+L-1}, a_k) = \probd_T(x_{k+1}|x_{k}, a_k )b_k(x_k)\nonumber\\
	& \int_{z_{k+1: k+L-1}} \prod_{i=k+1}^{k+L-1} \probd_T(x_{i+1} | x_{i}, \pi(b_i(b_{i-1}, a_{i -1}, z_{i}))) \probd_Z(z_{i}|x_i)\mathrm{d}z_{k+1:k+L-1}
\end{align}
\qed
\subsection{Proof of Lemma~\ref{lem:AvSafePosteriors} (average over the safe posteriors)} \label{proof:AvSafePosteriors}
\begin{align}
	&{\color{teal}\underbrace{\prob\Bigg(\bigwedge_{i=k}^{k+L}\mathbf{1}\left\{ x_i\!\in\!  \mathcal{X}^{\text{safe}}_i\right\}| b_k, \pi\Bigg)}_{(a)}} = \prob\Bigg(\mathbf{1}\left\{ x_k\!\in\!  \mathcal{X}^{\text{safe}}_k\right\} | b_k\Bigg) {\color{purple}\underbrace{\prob\Bigg(\bigwedge_{i=k+1}^{k+L}\mathbf{1}\left\{ x_i\!\in\!  \mathcal{X}^{\text{safe}}_i\right\}| \mathbf{1}\left\{ x_k\!\in\!  \mathcal{X}^{\text{safe}}_k\right\}, b_k, \pi\Bigg)}_{(b)}}
\end{align}
Let us focus on the expression we marked by $(b)$ 
\begin{align}
	&\prob\Bigg(\bigwedge_{i=k+1}^{k+L}\mathbf{1}\left\{ x_i\!\in\!  \mathcal{X}^{\text{safe}}_i\right\}| \mathbf{1}\left\{ x_k\!\in\!  \mathcal{X}^{\text{safe}}_k\right\} b_k, \pi\Bigg) = \nonumber\\
	&\int_{b_{k+1}}\prob\Bigg(\bigwedge_{i=k+1}^{k+L}\mathbf{1}\left\{ x_i\!\in\!  \mathcal{X}^{\text{safe}}_i\right\}| b_{k+1}, \mathbf{1}\left\{ x_k\!\in\!  \mathcal{X}^{\text{safe}}_k\right\}, b_k, \pi\Bigg)\prob\Bigg(b_{k+1} |  \mathbf{1}\left\{ x_k\!\in\!  \mathcal{X}^{\text{safe}}_k\right\} b_k, \pi\Bigg)\mathrm{d}b_{k+1}=\\
	&\int_{b_{k+1}} \prob\Bigg(b_{k+1} |  \mathbf{1}\left\{ x_k\!\in\!  \mathcal{X}^{\text{safe}}_k\right\}, b_k, \pi\Bigg) \prob\Bigg(\bigwedge_{i=k+1}^{k+L}\mathbf{1}\left\{ x_i\!\in\!  \mathcal{X}^{\text{safe}}_i\right\}| b_{k+1}, \pi\Bigg)\mathrm{d}b_{k+1}
\end{align}

Merging the two expressions we obtain 
\begin{align}
	&\prob\Bigg(\bigwedge_{i=k}^{k+L}\mathbf{1}\left\{ x_i\!\in\!  \mathcal{X}^{\text{safe}}_i\right\}| b_k, \pi\Bigg) = \prob\Bigg(\mathbf{1}\left\{ x_k\!\in\!  \mathcal{X}^{\text{safe}}_k\right\} | b_k\Bigg)\cdot \nonumber \\
	&\int_{b_{k+1}}\prob(b_{k+1} |  \mathbf{1}\left\{ x_k\!\in\!  \mathcal{X}^{\text{safe}}_k\right\}, b_k, \pi)\underbrace{{\color{teal} \prob\Bigg(\bigwedge_{i=k+1}^{k+L}\mathbf{1}\left\{ x_i\!\in\!  \mathcal{X}^{\text{safe}}_i\right\}| b_{k+1},\pi\Bigg)}}_{(c)}\mathrm{d}b_{k+1}
\end{align}
We observe that expression $(a)$ is very similar to $(c)$, namely
\begin{align}
	&\prob\Bigg(\bigwedge_{i=k+1}^{k+L}\mathbf{1}\left\{ x_i\!\in\!  \mathcal{X}^{\text{safe}}_i\right\}| b_{k+1}, \pi\Bigg) = \prob\Bigg(\mathbf{1}\left\{ x_{k+1}\!\in\!  \mathcal{X}^{\text{safe}}_{k+1}\right\} | b_{k+1}\Bigg)\cdot \nonumber \\
	&\int_{b_{k+2}}\prob(b_{k+2} |  \mathbf{1}\left\{ x_{k+1}\!\in\!  \mathcal{X}^{\text{safe}}_{k+1}\right\}, b_{k+1}, \pi)\underbrace{{\color{teal}\prob \Bigg(\bigwedge_{i=k+2}^{k+L}\mathbf{1}\left\{ x_i\!\in\!  \mathcal{X}^{\text{safe}}_i\right\}| b_{k+2}, \pi\Bigg)}}_{(d)}\mathrm{d}b_{k+2}
\end{align}
Merging the two we got 
\begin{align}
	&\prob\Bigg(\bigwedge_{i=k}^{k+L}\mathbf{1}\left\{ x_i\!\in\!  \mathcal{X}^{\text{safe}}_i\right\}| b_k, \pi\Bigg) = \prob(\mathbf{1}\left\{ x_k\!\in\!  \mathcal{X}^{\text{safe}}_k\right\} | b_k) \cdot\nonumber \\
	&\int_{b_{k+1}}\prob(b_{k+1} |  \mathbf{1}\left\{ x_k\!\in\!  \mathcal{X}^{\text{safe}}_k\right\}, b_k, \pi)\prob(\mathbf{1}\left\{ x_{k+1}\!\in\!  \mathcal{X}^{\text{safe}}_{k+1}\right\} | b_{k+1})\cdot \nonumber\\
	&\int_{b_{k+2}}\prob(b_{k+2} |  \mathbf{1}\left\{ x_{k+1}\!\in\!  \mathcal{X}^{\text{safe}}_{k+1}\right\}, b_{k+1}, \pi) \prob\Bigg(\bigwedge_{i=k+2}^{k+L}\mathbf{1}\left\{ x_i\!\in\!  \mathcal{X}^{\text{safe}}_i\right\}| b_{k+2}, \pi\Bigg)\mathrm{d}b_{k+2}\mathrm{d}b_{k+1}.
\end{align}
We behold the recurrence relation.

Now we show that marginalization can be done with respect to the observations. Let us assume that $i$ is the last index ($i=k+L$)
\begin{align}
	&\int\limits_{b_{i}} \prob(\mathbf{1}\left\{ x_{i}\!\in\!  \mathcal{X}^{\mathrm{safe}}_{i}\right\} | b_{i})\prob(b_{i} |  \mathbf{1}\left\{ x_{i-1}\!\in\!  \mathcal{X}^{\text{safe}}_{i-1}\right\}, b_{i-1}, \pi)\mathrm{d}b_i=\\
	&\int\limits_{b_{i}} \int_{z_i\in \mathcal{Z} } \prob(\mathbf{1}\left\{ x_{i}\!\in\!  \mathcal{X}^{\mathrm{safe}}_{i}\right\} | b_{i})\delta(b_i - \psi(b_{i-1},\mathbf{1}\left\{ x_{i-1}\!\in\!  \mathcal{X}^{\text{safe}}_{i-1}\right\},a_{i-1},z_i))\cdot \nonumber\\
	& p(z_i|a_{i-1},b_{i-1},\mathbf{1}\left\{ x_{i-1}\!\in\!  \mathcal{X}^{\text{safe}}_{i-1}\right\})\mathrm{d}z_i\mathrm{d}b_i=\\
	&\int_{z_i\in \mathcal{Z} }\int\limits_{b_{i}}  \prob(\mathbf{1}\left\{ x_{i}\!\in\!  \mathcal{X}^{\mathrm{safe}}_{i}\right\} | b_{i})\delta(b_i - \psi(b_{i-1},\mathbf{1}\left\{ x_{i-1}\!\in\!  \mathcal{X}^{\text{safe}}_{i-1}\right\},a_{i-1},z_i))\cdot \nonumber\\
	& p(z_i|a_{i-1},b_{i-1},\mathbf{1}\left\{ x_{i-1}\!\in\!  \mathcal{X}^{\text{safe}}_{i-1}\right\})\mathrm{d}b_i\mathrm{d}z_i=\\
	&\int_{z_i\in \mathcal{Z} } \prob(\mathbf{1}\left\{ x_{i}\!\in\!  \mathcal{X}^{\mathrm{safe}}_{i}\right\} | \psi(b_{i-1},\mathbf{1}\left\{ x_{i-1}\!\in\!  \mathcal{X}^{\text{safe}}_{i-1}\right\},a_{i-1},z_i)) \cdot \nonumber\\
	& p(z_i|a_{i-1},b_{i-1},\mathbf{1}\left\{ x_{i-1}\!\in\!  \mathcal{X}^{\text{safe}}_{i-1}\right\})\mathrm{d}z_i=\\
	& \mathbb{E}_{z_i}\Bigg[ \prob(\mathbf{1}\left\{ x_{i}\!\in\!  \mathcal{X}^{\mathrm{safe}}_{i}\right\} | \psi(b_{i-1},\mathbf{1}\left\{ x_{i-1}\!\in\!  \mathcal{X}^{\text{safe}}_{i-1}\right\},a_{i-1},z_i))\Bigg|a_{i-1},b_{i-1},\mathbf{1}\left\{ x_{i-1}\!\in\!  \mathcal{X}^{\text{safe}}_{i-1}\right\} \Bigg]
\end{align}
We plug this result into expression for $i-1$ and do the same trick to $b_{i-1}$
\qed
\subsection{Proof of Lemma~\ref{lem:absCont} (Absolute continuity of observation likelihoods)}  \label{proof:absCont}
Assume in contradiction that absolute continuity does not hold. That is, there exists observation $z_{k+1} = \bar{z}_{i+1} = \zeta$ such that  $\probd(z_{i+1} =\zeta | b_i, a_i) = 0 $ and $\probd(\bar{z}_{i+1}=\zeta | b_i, \mathbf{1}\{x_i \in \mathcal{X}_i^{\mathrm{safe}}\}, a_i) > 0 $. 
\begin{align}
\probd(z' = \zeta|a,b) = \int_{x' \in \mathcal{X}'} \probd_Z(z'=\zeta|x') p(x'|b, a)\mathrm{d}x' = 0 
\end{align}
The above integral can be zero only if the integrand is zero for all $x'$  
\begin{align}
\probd_Z(z'=\zeta|x') p(x'|b, a)  = 0 
\end{align}
If $\probd_Z(z'=\zeta|x') =0$ we have a contradiction since 
\begin{align}
\probd(z'=\zeta|a,b,\mathbf{1}\left\{ x\!\in\!  \mathcal{X}^{\text{safe}}\right\}) = \int_{x' \in \mathcal{X}'} \probd_Z(z'=\zeta |x') \probd(x'|b, a,\mathbf{1}\left\{ x\!\in\!  \mathcal{X}^{\text{safe}}\right\})\mathrm{d}x' = 0
\end{align}
On the other hand, $p(x'|b, a) > 0$ for some $x'$ since it is a propagated belief. 
\qed
\subsection{Proof of Lemma~\ref{lem:ml} (Maximum likelihood observation)}  \label{proof:ml}
It is sufficient to prove that we will obtain same $x^{\mathrm{ML}}_{\ell+1}$ in two cases. We recall that  $x^{\mathrm{ML}}_{\ell+1}= g(x^{\mathrm{ML}}_{\ell},a_{\ell})$ when the motion model has Gaussian noise $ \probd_T(\cdot | x_{\ell}, a_{\ell})=\mathcal{N}(g(x_{\ell},a_{\ell}), \Sigma_w)$ attains maximum at $g(x_{\ell},a_{\ell})$. 
\begin{align}
	x^{\mathrm{ML}}_{\ell} = \argmax_{x_{\ell}} b_{\ell}(x_{\ell}) = \argmax_{x_{\ell}}\frac{\mathbf{1}\left\{ x\!\in\!  \mathcal{X}^{\text{safe}}\right\} b_{\ell}(x_{\ell})}{\int_{\xi \in \mathcal{X}} \mathbf{1}\left\{ \xi\!\in\!  \mathcal{X}^{\text{safe}}\right\}b(\xi)\mathrm{d}\xi}  = \frac{\argmax_{x_{\ell}} b_{\ell}(x_{\ell})}{\int_{\xi \in \mathcal{X}} \mathbf{1}\left\{ \xi\!\in\!  \mathcal{X}^{\text{safe}}\right\}b(\xi)\mathrm{d}\xi} 
\end{align}
Where in the last passage we used the assumption. Now we can apply motion and observation models to obtain $z^{\mathrm{ML}}_{\ell+1}$.
\qed

\subsection{Proof of Lemma~\ref{lem:objComp} (Objectives comparison)}  \label{proof:ObjComp}
\begin{align}
	&\mathbb{E}[r^x(x)|\psi(b_{\ell},  a_{\ell},z^{\mathrm{ML}}_{\ell+1})\big] = \int_{x_{\ell+1}} r^x(x_{\ell+1})\probd(x_{\ell+1}|b_{\ell},a_{\ell}, z^{\mathrm{ML}}_{\ell+1})\mathrm{d} x_{\ell+1}=\\
	&\int_{x_{\ell+1}} r^x(x_{\ell+1})\probd(\mathbf{1}\left\{ x_{\ell} \in  \mathcal{X}^{\text{safe}}\right\}, x_{\ell+1}|b_{\ell},a_{\ell}, z^{\mathrm{ML}}_{\ell+1})\mathrm{d} x_{\ell+1} + \nonumber\\
	&\int_{x_{\ell+1}} r^x(x_{\ell+1})\probd(\mathbf{1}\left\{ x_{\ell} \notin  \mathcal{X}^{\text{safe}}\right\}, x_{\ell+1}|b_{\ell},a_{\ell}, z^{\mathrm{ML}}_{\ell+1})\mathrm{d} x_{\ell+1} 
\end{align}

\begin{align}
	&\int_{x_{\ell+1}} r^x(x_{\ell+1})\probd(\mathbf{1}\left\{ x_{\ell} \in  \mathcal{X}^{\text{safe}}\right\}, x_{\ell+1}|b_{\ell},a_{\ell}, z^{\mathrm{ML}}_{\ell+1})\mathrm{d} x_{\ell+1} =\\
	&\probd(\mathbf{1}\left\{ x_{\ell} \in  \mathcal{X}^{\text{safe}}\right\}|b_{\ell}, a_{\ell}, z^{\mathrm{ML}}_{\ell+1})\int_{x_{\ell+1}} r^x(x_{\ell+1})\probd(x_{\ell+1}|b_{\ell},a_{\ell}, z^{\mathrm{ML}}_{\ell+1},\mathbf{1}\left\{ x_{\ell} \in  \mathcal{X}^{\text{safe}}\right\})\mathrm{d} x_{\ell+1}=\\
	&\probd(\mathbf{1}\left\{ x_{\ell} \in  \mathcal{X}^{\text{safe}}\right\}|b_{\ell}, a_{\ell}, z^{\mathrm{ML}}_{\ell+1})\mathbb{E}[r^x(x)|\psi(b^{\mathrm{safe}}_{\ell},  a_{\ell},z^{\mathrm{ML}}_{\ell+1})\big]=\\
	&\frac{\probd\Bigg(z^{\mathrm{ML}}_{\ell+1}\Bigg|b_{\ell}, a_{\ell}, \mathbf{1}\left\{ x_{\ell} \in  \mathcal{X}^{\text{safe}}\right\} \Bigg)\prob\big(\mathbf{1}\left\{ x_{\ell} \in  \mathcal{X}^{\text{safe}}\right\} \big)|b_{\ell}, a_{\ell}\big)}{\probd(z^{\mathrm{ML}}_{\ell+1}|b_{\ell}, a_{\ell} )}\mathbb{E}[r^x(x)|\psi(b^{\mathrm{safe}}_{\ell},  a_{\ell},z^{\mathrm{ML}}_{\ell+1})\big]
\end{align}

\begin{align}
	&\probd(\mathbf{1}\left\{ x_{\ell} \notin  \mathcal{X}^{\text{safe}}\right\}|b_{\ell}, a_{\ell}, z^{\mathrm{ML}}_{\ell+1}))\mathbb{E}[r^x(x)|\psi(b_{\ell}, \mathbf{1}\left\{ x_{\ell} \notin  \mathcal{X}^{\text{safe}}\right\}, a_{\ell},z^{\mathrm{ML}}_{\ell+1})\big]=\\
	&\frac{\probd\Bigg(z^{\mathrm{ML}}_{\ell+1}\Bigg|b_{\ell}, a_{\ell}, \mathbf{1}\left\{ x_{\ell} \notin  \mathcal{X}^{\text{safe}}\right\} \Bigg)\prob\big(\mathbf{1}\left\{ x_{\ell} \notin  \mathcal{X}^{\text{safe}}\right\} \big)|b_{\ell}, a_{\ell}\big)}{\probd(z^{\mathrm{ML}}_{\ell+1}|b_{\ell}, a_{\ell} )}\mathbb{E}[r^x(x)|\psi(b_{\ell}, \mathbf{1}\left\{ x_{\ell} \notin  \mathcal{X}^{\text{safe}}\right\}, a_{\ell},z^{\mathrm{ML}}_{\ell+1})\big]
\end{align}
We arrived at the desired result.
\qed
\section{Calculating the Posterior Conditioned on the Safe Prior (Section \ref{sec:SubtreeRes})}\label{sec:PostFromSafePrior}
The safe event influence Belief-MDP motion model in the following way 
\begin{align}
	&\probd(b'|b, a,\mathbf{1}\left\{ x\!\in\!  \mathcal{X}^{\text{safe}}\right\})=\int_{z'\in \mathcal{Z} }\probd(b'|b,a,z', \mathbf{1}\left\{ x\!\in\!  \mathcal{X}^{\text{safe}}\right\}) \probd(z'|a,b,\mathbf{1}\left\{ x\!\in\!  \mathcal{X}^{\text{safe}}\right\}) \mathrm{d}z' =\nonumber\\
	&\int_{z'\in \mathcal{Z} }\delta(b' - \psi(b^{\text{safe}},a,z')) \probd(z'|a,b,\mathbf{1}\left\{ x\!\in\!  \mathcal{X}^{\text{safe}}\right\})\mathrm{d}z'
\end{align}
We first calculate the propagated belief conditioned on the safe prior.
\begin{align}
	\probd(x'|b, a,\mathbf{1}\left\{ x \!\in\!  \mathcal{X}^{\text{safe}}\right\})= \frac{\int_{x \in \mathcal{X}}\mathbf{1}\left\{ x\!\in\!  \mathcal{X}^{\text{safe}}\right\}\probd_T(x'|x,a)b(x)\mathrm{d}x}{\int_{\xi \in \mathcal{X}} \mathbf{1}\left\{ \xi\!\in\!  \mathcal{X}^{\text{safe}}\right\}b(\xi)\mathrm{d}\xi}
\end{align}
$b$ and event safe, meaning that belief supposed to be zero at non safe places.
Finally,
\begin{align}
	\probd(z'|a,b,\mathbf{1}\left\{ x\!\in\!  \mathcal{X}^{\text{safe}}\right\}) = \int_{x' \in \mathcal{X}'} \probd_Z(z'|x') p(x'|b, a,\mathbf{1}\left\{ x\!\in\!  \mathcal{X}^{\text{safe}}\right\})\mathrm{d}x'.
\end{align}
We can also look at the above from slightly different angle. We define $b^{\text{safe}}$
\begin{align}
	b^{\text{safe}}(x)= \frac{\mathbf{1}\left\{ x\!\in\!  \mathcal{X}^{\text{safe}}\right\}b(x)}{\int_{\xi \in \mathcal{X}} \mathbf{1}\left\{ \xi\!\in\!  \mathcal{X}^{\text{safe}}\right\}b(\xi)\mathrm{d}\xi},
\end{align}
such that 
\begin{align}
		\probd(x'|b, a,\mathbf{1}\left\{ x \!\in\!  \mathcal{X}^{\text{safe}}\right\})= \int_{x \in \mathcal{X}}\probd_T(x'|x,a)b^{\mathrm{safe}}(x)\mathrm{d}x.
\end{align}
We can use the safe belief defined above in the belief update as follows 
\begin{align}
	&\probd (x'|b, a, z', \mathbf{1}\{x \in \mathcal{X}^{\mathrm{safe}}\}) = \frac{\probd (z'|b, a, x', \mathbf{1}\{x \in \mathcal{X}^{\mathrm{safe}}\})\probd (x'|b, a, \mathbf{1}\{x \in \mathcal{X}^{\mathrm{safe}}\}) }{	\probd (z'|b, a, \mathbf{1}\{x \in \mathcal{X}^{\mathrm{safe}}\})} = \\
	&\frac{\probd_{Z} (z'|x')\probd (x'|b, a, \mathbf{1}\{x \in \mathcal{X}^{\mathrm{safe}}\}) }{\underset{\xi'}{\int}	\probd_Z (z'|\xi') \probd (\xi'|b, a, \mathbf{1}\{x \in \mathcal{X}^{\mathrm{safe}}\})\mathrm{d}\xi'}
\end{align}		
Now the $\psi$ is conventional belief update operator receiving as input $\psi(b^{\text{safe}}, a, z')$.
\section{Derivation of the Importance Weights (Section \ref{sec:Importance})} \label{sec:ImportanceWeights}
To calculate the likelihoods of the observations we shall do the following. Suppose that the belief is represented by samples. 
\begin{align}
	b_{k}(x_{k}) \approx \sum_{i=1}^{N} w_k^i \delta(x_{k} - x^i_{k}) \label{eq:BeliefSampleApprox},  
\end{align}
Let us introduce another notation $\delta(x_{k} - x^i_{k})= \delta^{x_k^i}(x_{k}) $   
so 
\begin{align}
	\probd(x_{k+1}|b_k, a_k,\mathbf{1}_{\left\{ x_k  \in   \mathcal{X}^{\text{safe}}_k\right\}})= &\frac{\int_{x_k \in \mathcal{X}}\mathbf{1}_{\left\{ x_k \in   \mathcal{X}^{\text{safe}} \right\}}\probd_T(x_{k+1}|x_k,a_k)b_k(x_k)\mathrm{d}x_k}{\int_{\xi_k \in \mathcal{X}} \mathbf{1}_{\left\{ \xi_k  \in   \mathcal{X}^{\text{safe}}\right\}}b_k(\xi_k)\mathrm{d}\xi_k} \approx \\
	&\frac{\int_{x_k \in \mathcal{X}}\mathbf{1}_{\left\{ x_k \in   \mathcal{X}^{\text{safe}} \right\}}\probd_T(x_{k+1}|x_k,a_k) \Bigg(\sum_{i=1}^{N} w_k^i \delta^{x_k^i}(x_{k})\Bigg)\mathrm{d}x_k}{\int_{\xi_k \in \mathcal{X}} \mathbf{1}_{\left\{ \xi_k \in   \mathcal{X}^{\text{safe}}\right\}}\Bigg(\sum_{i=1}^{N} w_k^i  \delta^{x_k^i}(x_{k})\Bigg)\mathrm{d}\xi_k} = \\
	&\frac{\sum_{i=1}^{N} w_k^i\mathbf{1}_{\left\{ x^i_k \in   \mathcal{X}^{\text{safe}} \right\}}\probd_T(x_{k+1}|x^i_k,a_k)}{\sum_{i=1}^{N} w_k^i \mathbf{1}_{\left\{ x^i_k \in   \mathcal{X}^{\text{safe}} \right\}}} \approx \\
	&\frac{\sum_{i=1}^{N} w_k^i\mathbf{1}_{\left\{ x^i_k \in   \mathcal{X}^{\text{safe}} \right\}}\delta^{x^i_{k+1}}(x_{k+1})}{\sum_{i=1}^{N} w_k^i \mathbf{1}_{\left\{ x^i_k \in   \mathcal{X}^{\text{safe}} \right\}}}.
\end{align}  
We got that 
\begin{align}
	\probd(\bar{z}_{k+1} =  z^j_{k+1}| b_k, \mathbf{1}_{\{x_k \in \mathcal{X}_i^{\mathrm{safe}}\}}, a_k) \approx  \frac{\sum_{i=1}^{N} w_k^i\mathbf{1}_{\left\{ x^i_k \in   \mathcal{X}^{\text{safe}} \right\}}\probd_Z(z^j_{k+1}|x^i_{k+1})}{\sum_{i=1}^{N} w_k^i \mathbf{1}_{\left\{ x^i_k \in   \mathcal{X}^{\text{safe}} \right\}}}.
\end{align}
In case of the denominator we arrive to the same expression, only without the indicator. 
\begin{align}
	\probd(z_{k+1} =  z^j_{k+1}| b_k, a_k) \approx  \frac{\sum_{i=1}^{N} w_k^i \probd_Z(z^j_{k+1}|x^i_{k+1})}{\sum_{i=1}^{N} w_k^i }.
\end{align}
In reality, however, it is possible that after we discard all the samples of the belief which are not safe we are left with a very small set of samples or an empty set. To alleviate this issue we resample the safe particles to a constant number of samples $N$.
\section{Mean Distance to Goal Accountability for Uncertainty} \label{sec:DistanceToGoal}
\begin{figure}[t] 
	\centering
	\begin{minipage}[t]{0.2\textwidth}
		\centering 
		\includegraphics[width=\textwidth]{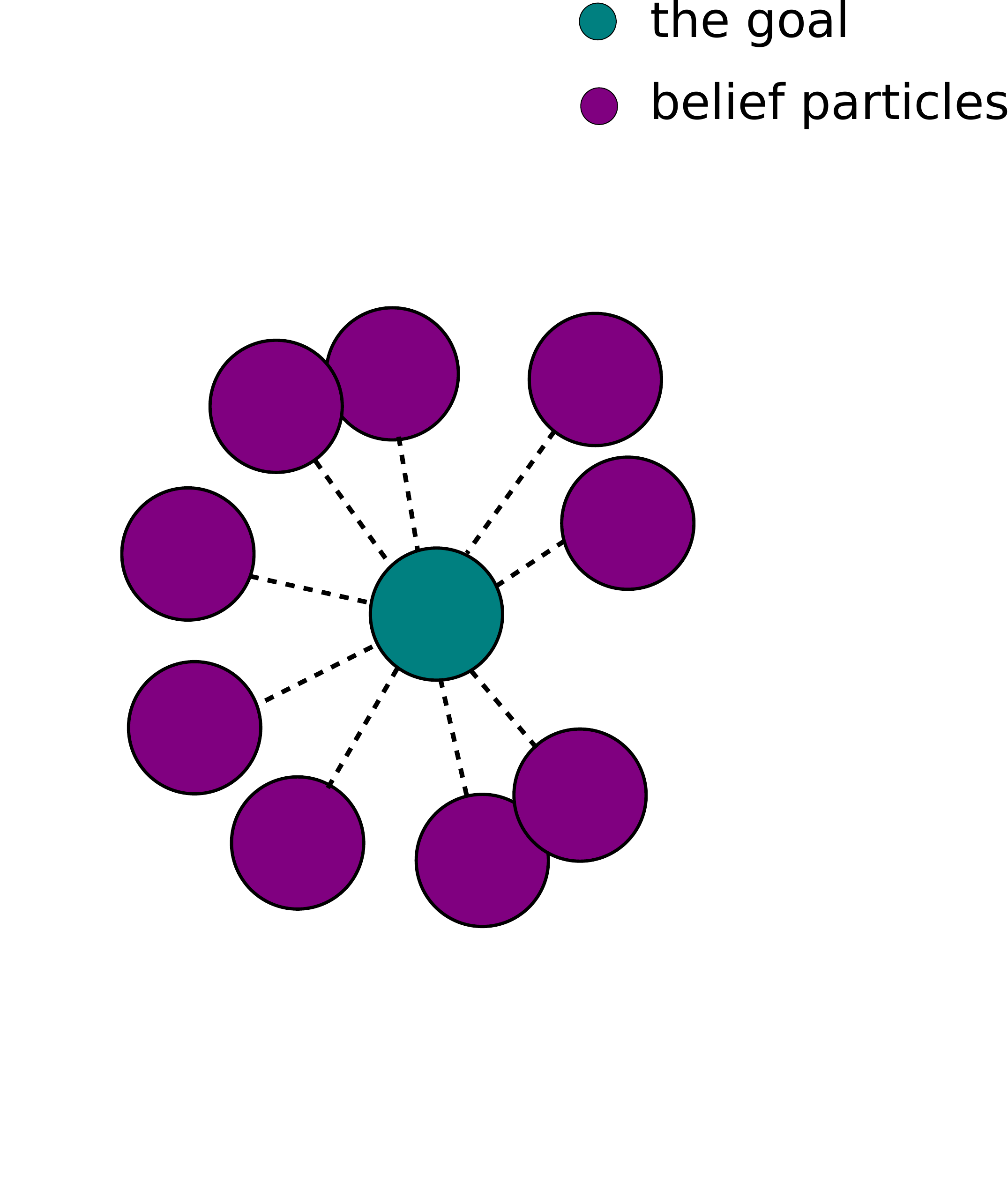}
		\subcaption{}
	\end{minipage}%
	\hfill
	\begin{minipage}[t]{0.2\textwidth}
		\centering 
		\includegraphics[width=\textwidth]{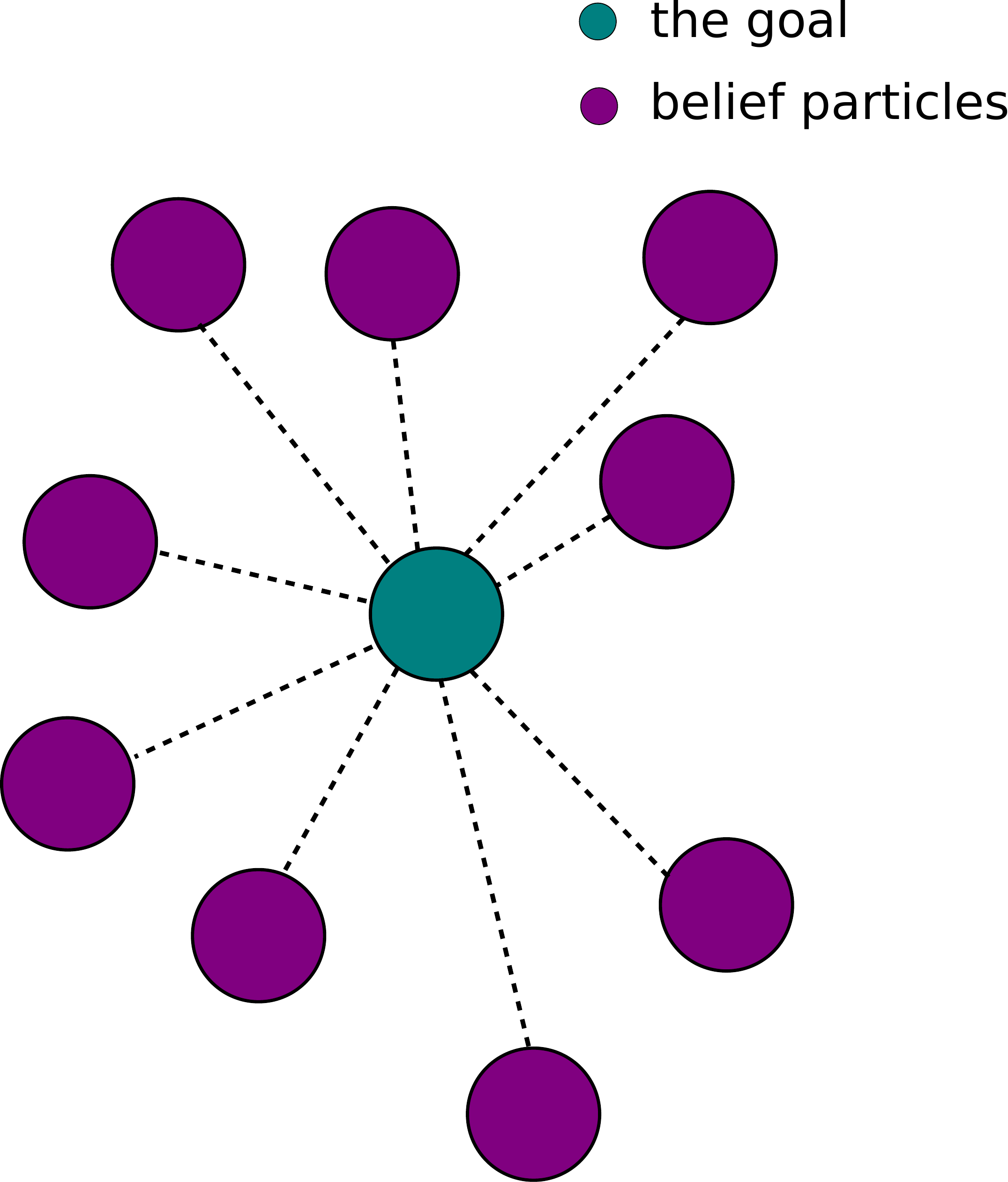}
		\subcaption{}
	\end{minipage}
	\hfill
	\begin{minipage}[t]{0.2\textwidth}
		\centering 
		\includegraphics[width=\textwidth]{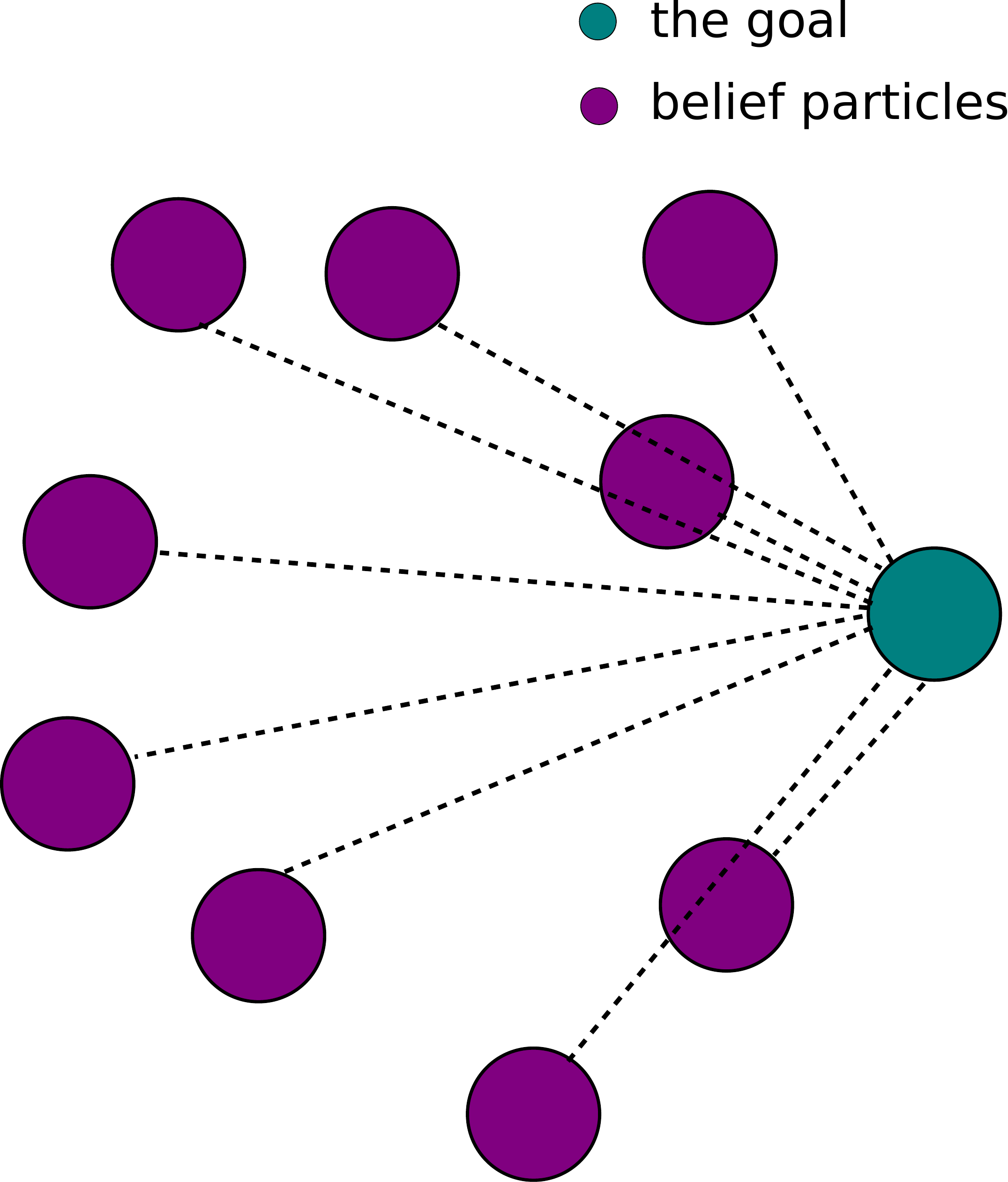}
		\subcaption{}
	\end{minipage}
		\hfill
	\begin{minipage}[t]{0.2\textwidth}
		\centering 
		\includegraphics[width=\textwidth]{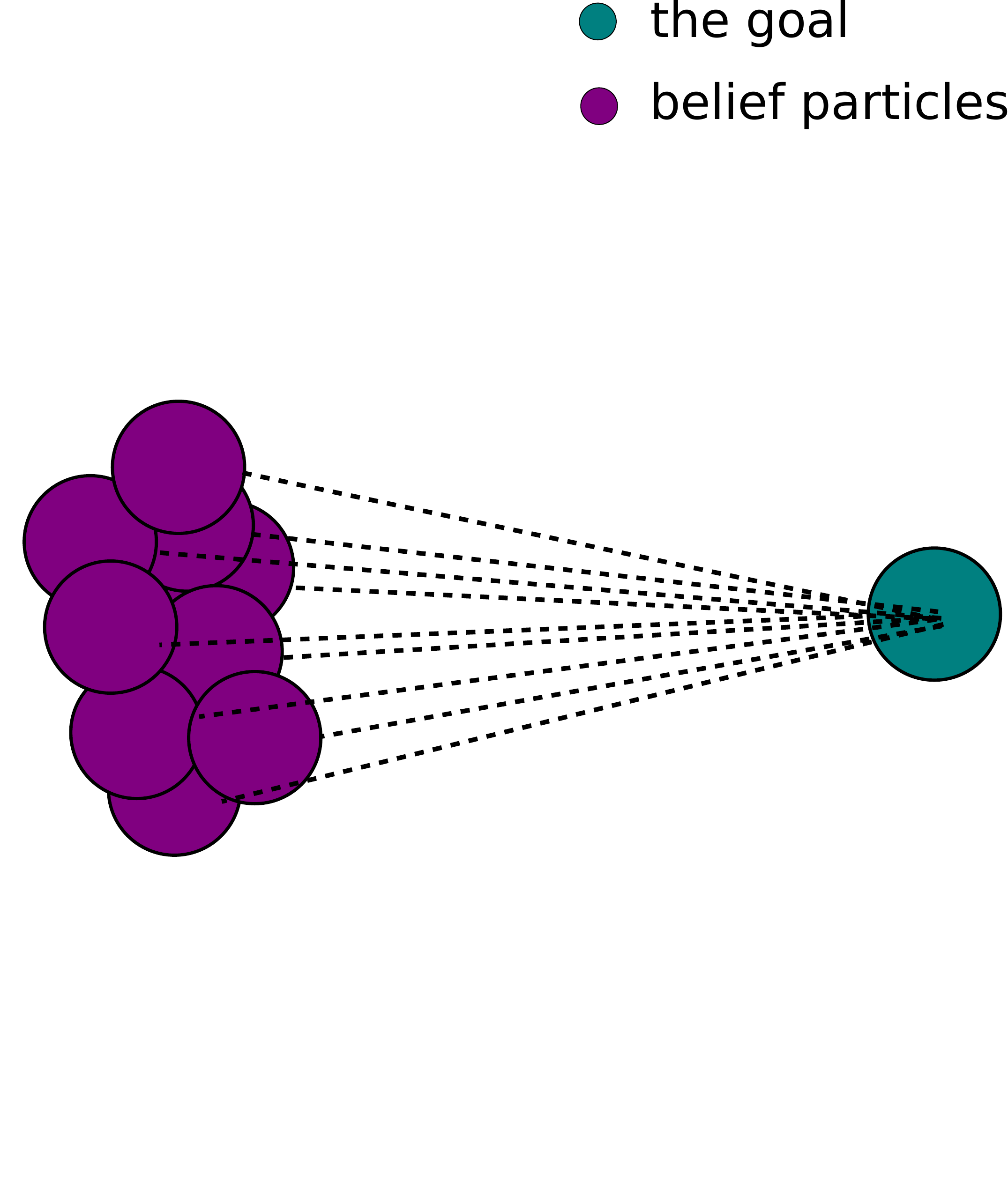}
		\subcaption{}
	\end{minipage}  
	\caption{Geometrical visualization of the natural belief uncertainty measure imprinted in the mean distance to the goal. \textbf{(a)} Less spread results in lowering all the distances, thereby the mean.  \textbf{(b)} The reciprocal situation. \textbf{(c)} Another situation, here, to decrease the mean distance to the goal, one has to reduce the spread and the distance between the expected value of the belief and the goal. \textbf{(d)} The spread is decreased, but the distance between the expected value of the belief and the goal is large (Appendix \ref{sec:DistanceToGoal}).} \label{fig:DistanceToGoal} 
\end{figure} 
In this section, we discuss in depth why the mean distance to goal intrinsically accounts for belief uncertainty. We show a geometrical visualization in Fig.~\ref{fig:DistanceToGoal}. Since the distance is no negative, the less spread of belief implies lower distances and vise versa.   Further, let us show that algebraically.

\begin{thm}
	Let $y$ be an arbitrary distributed random vector with $\mu_y$  and $\Sigma_y$  being the expected value and covariance matrix of $y$, respectively; and let $\Lambda$ be arbitrary matrix.  The following relation is correct
	\begin{align}
		\mathbb{E}\left[y^T \Lambda y\right] = \operatorname {tr}\left[\Lambda \Sigma_y \right]+\mu_y ^{T}\Lambda \mu_y, 
	\end{align} 
	where by $\operatorname {tr}$ we denote the trace operator. 
\end{thm}
\begin{proof}
Since the quadratic form is a scalar quantity,  $y^{T}\Lambda y =\operatorname {tr} (y ^{T}\Lambda y )$. Next, by the cyclic property of the trace operator,
\begin{align} 
	\mathbb{E}[\operatorname {tr} (y ^{T}\Lambda y )]=\mathbb{E} [\operatorname {tr} (\Lambda y y^{T})].
\end{align}
Since the trace operator is a linear combination of the components of the matrix, it therefore follows from the linearity of the expectation operator that
\begin{align}
	\mathbb{E} \left[\operatorname {tr} (\Lambda y y^{T})\right]=\operatorname {tr} (\Lambda \mathbb{E} (yy^{T})).
\end{align}
A standard property of variances then tells us that this is
\begin{align}
	\operatorname {tr} (\Lambda (\Sigma_x +\mu_x \mu_x^{T})).
\end{align}
Applying the cyclic property of the trace operator again, we get
\begin{align}
	\operatorname {tr} (\Lambda \Sigma_y )+\operatorname {tr} (\Lambda \mu_y \mu_y^{T})=\operatorname {tr} (\Lambda \Sigma_y )+\operatorname {tr} (\mu_y^{T}\Lambda \mu_y )=\operatorname {tr} (\Lambda \Sigma_y )+\mu_y^{T}\Lambda \mu_y .
\end{align}	
\end{proof}
\noindent Now, we set $y=x-x^g$, where $x \sim b$ with the mean $\mu_x$ and covariance matrix $\Sigma_x$;  $x^g$ is the deterministic goal location. Recall that covariance matrix is invariant to the deterministic translational shifts of a random vector, so $\Sigma_{x-x^g}=\Sigma_x$. Moreover, by setting $\Lambda = I$ we obtain

\begin{align}
	\mathbb{E}\left[y^T \Lambda y\right] = \mathbb{E}\left[(x-x^g)^T I (x-x^g)\right] = \mathbb{E}\left[\|x-x^g\|^2_2\right] =   \operatorname {tr}\left[\Sigma_x \right]+\|\mu_x-x^g\|^2_2. 
\end{align} 
We arrived at the desired result. As we observe in Fig.~\ref{fig:DistanceToGoal}, the trace of the covariance matrix controls the spread of the belief in the first summand; the second summand is the distance between the expected value of the belief.	

\section{Necessary Condition for Feasibility of Chance Constraint (Lemma \ref{lem:NecessaryChance})} \label{sec:ExecutionRiskBound}
In this section, we develop necessary condition for feasibility of chance constraint from \cite{Santana16aaai}.  Through importance sampling  we extend the condition presented in \cite{Santana16aaai} to continuous spaces in terms of states and the observations. Since all the beliefs in this section are obtained from safe belief with an action and observation, to remove clutter let us  relinquish the bar notation. 
\begin{align}
	\prob\Bigg(\mathbf{1}_{\left\{\tau \in  \times_{i=0}^L \mathcal{X}_{k+i}^{\text{safe}}\right\}}| b_k, \pi\Bigg) = \prob\Bigg(\bigwedge_{i=k+1}^{k+L}\mathbf{1}\left\{ x_i\!\in\!  \mathcal{X}^{\text{safe}}_i\right\}| \mathbf{1}\left\{ x_k\!\in\!  \mathcal{X}^{\text{safe}}_k\right\}, b_k, \pi\Bigg) \cdot \prob\Bigg(\mathbf{1}\left\{ x_k\!\in\!  \mathcal{X}^{\text{safe}}_k\right\} | b_k \Bigg)
\end{align}

\begin{align}
	& \prob\Bigg(\bigwedge_{i=k+1}^{k+L}\mathbf{1}\left\{ x_i\!\in\!  \mathcal{X}^{\text{safe}}_i\right\}| \mathbf{1}\left\{ x_k\!\in\!  \mathcal{X}^{\text{safe}}_k\right\}, b_k, \pi\Bigg) = \nonumber\\
	& \int_{b_{k+1}}\prob\Bigg(\bigwedge_{i=k+1}^{k+L}\mathbf{1}\left\{ x_i\!\in\!  \mathcal{X}^{\text{safe}}_i\right\}| b_{k+1}, \mathbf{1}\left\{ x_k\!\in\!  \mathcal{X}^{\text{safe}}_k\right\}, b_k, \pi\Bigg)\prob\Bigg(b_{k+1} |  \mathbf{1}\left\{ x_k\!\in\!  \mathcal{X}^{\text{safe}}_k\right\} b_k, \pi\Bigg)\mathrm{d}b_{k+1}=\\
	&\int_{b_{k+1}} \prob\Bigg(b_{k+1} |  \mathbf{1}\left\{ x_k\!\in\!  \mathcal{X}^{\text{safe}}_k\right\}, b_k, \pi\Bigg) \prob\Bigg(\bigwedge_{i=k+1}^{k+L}\mathbf{1}\left\{ x_i\!\in\!  \mathcal{X}^{\text{safe}}_i\right\}| b_{k+1}, \pi\Bigg)\mathrm{d}b_{k+1} = \\
	&\int_{b_{k+1}} \prob\Bigg(b_{k+1} |  \mathbf{1}\left\{ x_k\!\in\!  \mathcal{X}^{\text{safe}}_k\right\}, b_k, \pi\Bigg)\prob\Bigg(\mathbf{1}_{\left\{\tau \in  \times_{i=1}^L \mathcal{X}_{k+i}^{\text{safe}}\right\}}| b_{k+1}, \pi\Bigg) \mathrm{d}b_{k+1}
\end{align}
Following the notations in \cite{Santana16aaai} 
\begin{align}
	&\mathrm{er}(b_{k};\pi) \bydef 1 -\prob(\mathbf{1}_{\left\{\tau \in  \times_{i=0}^L \mathcal{X}_{k+i}^{\text{safe}}\right\}}| b_{k}, \pi) \leq \Delta \bydef 1-\delta, 
\end{align}
Similar 
\begin{align}
	\prob\Bigg(\mathbf{1}\left\{ x_k\!\in\!  \mathcal{X}^{\text{safe}}_k\right\} | b_k \Bigg) = \int_{x_k} b_k(x_k) \mathbf{1}\left\{ x_k\!\in\!  \mathcal{X}^{\text{safe}}_k\right\}\mathrm{d} x_k  = 1-r_b(b_k)
\end{align}
We have that 
\begin{align}
	&\mathrm{er}(b_{k};\pi) = 1 - (1- r_b(b_k))\cdot \int_{b_{k+1}} \prob\Bigg(b_{k+1} |  \mathbf{1}\left\{ x_k\!\in\!  \mathcal{X}^{\text{safe}}_k\right\}, b_k, \pi\Bigg)\prob\Bigg(\mathbf{1}_{\left\{\tau \in  \times_{i=1}^L \mathcal{X}_{k+i}^{\text{safe}}\right\}}| b_{k+1}, \pi\Bigg) \mathrm{d}b_{k+1} = \\
	&= 1 - (1- r_b(b_k))\cdot \int_{z_{k+1}} \prob\Bigg(z_{k+1} |  \mathbf{1}\left\{ x_k\!\in\!  \mathcal{X}^{\text{safe}}_k\right\}, b_k, \pi\Bigg)(1- \mathrm{er}(b_{k+1},\pi)) \mathrm{d}z_{k+1} = \\
	& 1 - (1- r_b(b_k))\cdot \Bigg(1-\int_{z_{k+1}} \prob\Bigg(z_{k+1} |  \mathbf{1}\left\{ x_k\!\in\!  \mathcal{X}^{\text{safe}}_k\right\}, b_k, \pi\Bigg) \mathrm{er}(b_{k+1},\pi) \mathrm{d}z_{k+1}\Bigg) = \\
	&r_b(b_k) + (1 -r_{b}(b_k) )\int_{z_{k+1}} \prob\Bigg(z_{k+1} |  \mathbf{1}\left\{ x_k\!\in\!  \mathcal{X}^{\text{safe}}_k\right\}, b_k, \pi\Bigg) \mathrm{er}(b_{k+1},\pi) \mathrm{d}z_{k+1}
\end{align}
Suppose we approximate the above integral by samples from 	$\probd(z_{k+1} =  z^j_{k+1}| b_k, a_k)$. We obtain, using \eqref{eq:ImportanceApprox}
\begin{align}
	&\mathrm{er}(b_{k};\pi) =  \nonumber\\
	&r_b(b_k) + (1 -r_{b}(b_k) )\int_{z_{k+1}}\frac{1}{\sum_{j=1}^{m} w^{\bar{z},j}_{k+1}}\sum_{j=1}^{m} w^{\bar{z},j}_{k+1} \delta(z_{k+1}- z^j_{k+1})  \mathrm{er}(b_{k+1},\pi) \mathrm{d}z_{k+1} = \\
	&r_b(b_k) + (1 -r_{b}(b_k) )\frac{1}{\sum_{j=1}^{m} w^{\bar{z},j}_{k+1}}\sum_{j=1}^{m} w^{\bar{z},j}_{k+1}   \mathrm{er}(b_{k+1}(z^j_{k+1}),\pi) 
\end{align}
Suppose that $0\leq \Delta \leq 1$ and $\mathrm{er}(b_{k};\pi) \leq \Delta$. From now on for clarity suppose the weights are already normalized. We choose some child $j = i$ and arrive at 

\begin{align}
	&r_b(b_k) + (1 -r_{b}(b_k) )\sum_{j=1}^{m} w^{\bar{z},j}_{k+1}   \mathrm{er}(b_{k+1}(z^j_{k+1}),\pi) \leq \Delta,  
\end{align}
that is 
\begin{align}
	&    \mathrm{er}(b_{k+1}(z^i_{k+1}),\pi) \leq \frac{1}{w^{\bar{z},i}_{k+1}}\Bigg(\frac{\Delta - r_b(b_k)}{(1 -r_{b}(b_k) )} - \sum_{\substack{{j=1} \\ {j\neq i}}}^{m} w^{\bar{z},j}_{k+1}   \mathrm{er}(b_{k+1}(z^j_{k+1}),\pi)\Bigg),  
\end{align} 
The existence of the above equation requires  if $r_{b}(b_k) <1$ and $w^{\bar{z},i}_{k+1} \neq 0$ whenever $\probd(z_{k+1} = z^i_{k+1}|b_k, a_k) > 0$.    

Let us show that these conditions are equivalent.  

If $r_b(b_k) = 1$, so $\prob(\mathbf{1}\left\{ x_k\!\in\!  \mathcal{X}^{\text{safe}}_k\right\}|b_k) = 0 $. Namely, at each $x_k$ or $b_k(x_k) =0 $ or $x_k \notin \mathcal{X}^{\text{safe}}_k$.  From 
\begin{align}
	\probd(x'|b, a,\mathbf{1}\left\{ x \!\in\!  \mathcal{X}^{\text{safe}}\right\})= \int_{x \in \mathcal{X}}\mathbf{1}\left\{ x\!\in\!  \mathcal{X}^{\text{safe}}\right\}\probd_T(x'|x,a)b(x)\mathrm{d}x  = 0
\end{align}
we observe that in this case the conditional $\probd(\bar{z}_{i+1} =  z^i_{k+1}| b_k, \mathbf{1}\{x_k \in \mathcal{X}_k^{\mathrm{safe}}\}, a_k) = 0$.  

Now we show the inverse relation. 
Suppose  $w^{\bar{z},i}_{k+1} = 0$, we have that $\probd(\bar{z}_{i+1} =  z^i_{k+1}| b_k, \mathbf{1}\{x_k \in \mathcal{X}_k^{\mathrm{safe}}\}, a_k) = 0$. This implies that  
 \begin{align}
 	\!\!\probd\bigg(\bar{z}_{i+1} =  z^i_{k+1}| b_k, \mathbf{1}\{x_k \in \mathcal{X}_k^{\mathrm{safe}}\}, a_k\bigg)\!\! = \!\!\!\!\!\!\int\limits_{x_{k+1} \in \mathcal{X}}  \!\!\!\!\!\!\probd_Z(\bar{z}_{i+1} =  z^i_{k+1}|x_{k+1}) \probd(x_{k+1}|b_k, a_k,\mathbf{1}\left\{ x_k\!\in\!  \mathcal{X}^{\text{safe}}\right\})\mathrm{d}x_{k+1} = 0 . \!
 \end{align}
We obtained again that   $\probd(x_{k+1}|b_k, a_k,\mathbf{1}\left\{ x_k\!\in\!  \mathcal{X}^{\text{safe}}\right\}) = 0 \quad \forall x_{k+1}$, so we arrived again that or $x_{k} \notin \mathcal{X}^{\text{safe}}_k$ or $b_k(x_k) = 0$.
Meaning that the system is guaranteed to be in a constraint-violating path
at time $k$, yielding $r_b(b_k)=1$. 

If $r_b(b_k) =1$ we prune the action resulting in this belief (See, line $7$ in Alg.~\ref{alg:constrainedSSbaselineImportance} and Alg.~\ref{alg:constrainedSSbaseline} ).
Note that in reality we approximate the belief by possibly weighted samples as in \eqref{eq:BeliefSampleApprox}. However the arguments presented above regarding the existence of the pruning condition are valid. 

To do pruning using technique from \cite{Santana16aaai} we note that $r_b(b_{k+1}) \leq \mathrm{er}(b_{k+1};\pi)$ and got that 
\begin{align}
	&    r_b(b^i_{k+1}) \leq \frac{1}{w^{\bar{z},i}_{k+1}}\Bigg(\frac{\Delta - r_b(b_k)}{(1 -r_{b}(b_k) )} - \sum_{\substack{{j=1} \\ {j\neq i}}}^{m} w^{\bar{z},j}_{k+1} r_b(b^j_{k+1})\Bigg).
\end{align}
If  $\mathrm{er}(b_{k};\pi) \leq \Delta$ the above shall hold for every child of $b_k$.
\section{Influence of Safety Event to the Objective (Section \ref{sec:ChanceApprox})} \label{sec:ObjectivesBMDP}
To rigorously derive the impact of safety event we shall start from Belief-MDP (BMDP). This is because the overall impact starts from BMDP motion model. 
	\begin{equation}
	\begin{gathered}
	\underset{\bar{b}_{\ell+1}}{\mathbb{E}}\Big[ \rho(\bar{b}_{\ell+1})  \Big|a_{\ell},b_{\ell}, \mathbf{1}\{x_{\ell} \in \mathcal{X}^{\mathrm{safe}}\} \Big] = \underset{\bar{z}_{\ell+1}}{\int} \rho(\psi(b^{\mathrm{safe}}_{\ell}, a_{\ell},  \bar{z}_{\ell+1} ))
	{\color{inkscapePurple}{\probd(\bar{z}_{\ell+1}| b_{\ell}, \mathbf{1}\{x_{\ell} \in \mathcal{X}^{\mathrm{safe}}\}, a_{\ell})}} \mathrm{d} \bar{z}_{\ell+1}
	\end{gathered}
	\end{equation}
	\begin{equation}
	\begin{gathered}
	\underset{b_{\ell+1}}{\mathbb{E}}\Big[ \rho(b_{\ell+1})  \Big|a_{\ell},b_{\ell} \Big] =
	\underset{z_{\ell+1}}{\int}\rho(\psi(b_{\ell}, a_{\ell},  z_{\ell+1} )) {\color{inkscapePurple}{\probd(z_{\ell+1}| b_{\ell}, a_{\ell})}} \mathrm{d} z_{\ell+1}
	\end{gathered}
	\end{equation}

Further we prove the above relations. 

\begin{equation}
	\begin{gathered}
		\underset{\bar{b}_{\ell+1}}{\mathbb{E}}\Big[ \rho(\bar{b}_{\ell+1})  \Big|a_{\ell},b_{\ell}, \mathbf{1}\{x_{\ell} \in \mathcal{X}^{\mathrm{safe}}\} \Big] = \int_{\bar{b}_{\ell+1}}\rho(\bar{b}_{\ell+1}) \probd(\bar{b}_{\ell+1}| b_{\ell},\mathbf{1}\{x_{\ell} \in \mathcal{X}^{\mathrm{safe}}\} , a_{\ell}) \mathrm{d} \bar{b}_{\ell+1}=\\
		\underset{\bar{b}_{\ell+1}}{\int}\rho(\bar{b}_{\ell+1}) \Bigg( \underset{\bar{z}_{\ell+1}}{\int}\probd(\bar{b}_{\ell+1}, \bar{z}_{\ell+1}| b_{\ell}, a_{\ell},\mathbf{1}\{x_{\ell} \in \mathcal{X}^{\mathrm{safe}}\} ) \mathrm{d} z_{\ell+1}\Bigg) \mathrm{d}b_{\ell+1}=\\
		\underset{\bar{z}_{\ell+1}}{\int}\Bigg(\underset{\bar{b}_{\ell+1}}{\int}\rho(\bar{b}_{\ell+1}) \probd(\bar{b}_{\ell+1}| b_{\ell}, \mathbf{1}\{x_{\ell} \in \mathcal{X}^{\mathrm{safe}}\}, a_{\ell},  \bar{z}_{\ell+1} )\mathrm{d} \bar{b}_{\ell+1}\Bigg)  {\color{inkscapePurple}{\probd(\bar{z}_{\ell+1}| b_{\ell}, \mathbf{1}\{x_{\ell} \in \mathcal{X}^{\mathrm{safe}}\}, a_{\ell})}} \mathrm{d} \bar{z}_{\ell+1}
	\end{gathered}
\end{equation}
with 
\begin{equation}
	\begin{gathered}
		\underset{\bar{b}_{\ell+1}}{\int}\rho(\bar{b}_{\ell+1}) \probd(\bar{b}_{\ell+1}| b_{\ell}, \mathbf{1}\{x_{\ell} \in \mathcal{X}^{\mathrm{safe}}\}, a_{\ell},  \bar{z}_{\ell+1} )\mathrm{d} \bar{b}_{\ell+1} = \\
		\int_{\bar{b}_{\ell+1}}  \rho(\bar{b}_{\ell+1}) \delta(\bar{b}_{\ell+1} - \psi(b^{\text{safe}}_{\ell}, a_{\ell},  \bar{z}_{\ell+1})) \mathrm{d} \bar{b}_{\ell+1} = \rho(\psi(b^{\mathrm{safe}}_{\ell}, a_{\ell},  \bar{z}_{\ell+1} ))
	\end{gathered}
\end{equation}

versus 
\begin{equation}
	\begin{gathered}
		\underset{b_{\ell+1}}{\mathbb{E}}\Big[ \rho(b_{\ell+1})  \Big|a_{\ell},b_{\ell} \Big] =\underset{b_{\ell+1}}{\int}\rho(b_{\ell+1}) \probd(b_{\ell+1}| b_{\ell}, a_{\ell}) \mathrm{d} b_{\ell+1}= \underset{b_{\ell+1}}{\int}\rho(b_{\ell+1})\Bigg(\underset{z_{\ell+1}}{\int} \probd(b_{\ell+1}, z_{\ell+1}| b_{\ell}, a_{\ell} ) \mathrm{d} z_{\ell+1}\Bigg)\mathrm{d} b_{\ell+1}=\\
		\underset{z_{\ell+1}}{\int}\Bigg(\underset{b_{\ell+1}}{\int}\rho(b_{\ell+1}) \probd(b_{\ell+1}| b_{\ell}, a_{\ell},  z_{\ell+1} )\mathrm{d} b_{\ell+1}\Bigg)  {\color{inkscapePurple}{\probd(z_{\ell+1}| b_{\ell}, a_{\ell})}} \mathrm{d} z_{\ell+1}
	\end{gathered}
\end{equation}
with 
\begin{equation}
	\begin{gathered}
		\underset{b_{\ell+1}}{\int}\rho(b_{\ell+1}) \probd(b_{\ell+1}| b_{\ell}, a_{\ell},  z_{\ell+1} )\mathrm{d} b_{\ell+1} =\underset{b_{\ell+1}}{\int}\rho(b_{\ell+1}) \delta(b_{\ell+1} - \psi(b_{\ell}, a_{\ell},  z_{\ell+1} )\mathrm{d} b_{\ell+1} = \rho(\psi(b_{\ell}, a_{\ell},  z_{\ell+1} )).
	\end{gathered}
\end{equation}
\qed

\end{appendices}	
\bibliographystyle{plain}	

\input{paper.bbl}
%
\end{document}

%% file: floats/ConstrainedSS.tex
\begin{algorithm}[h!]
\caption{Prob. Constrained BMDP  Sparse Sampling (PCSS)}
\begin{algorithmic}[1]
\Procedure{PCSS}{belief: $b$, depth: $d$}
\If { $d=0$ }
	\State \Return	(Null, $\rho(b)$)
\EndIf
\State $(a^{*}, v^*$) $\leftarrow$ (Null, $-\infty$)
\For  {$a \in  \mathcal{A}$}
\State $ v\leftarrow$ $0.0$ \Comment{Value function}
\State Calculate propagated belief $b'^{-}$ 
\State pruned  $\leftarrow$ false 
\For {\_ $\in 1:m_d$}
\State Sample $x^o \sim b'^{-}$ 
\State Sample $z' \sim \probd(z'|x^o)$
\State  $b^{\prime} \leftarrow \psi(b, a, z')$ 
\If {$\phi(b^{\prime} ) < \delta$ } \Comment{Prune}
\State pruned  $\leftarrow$ true
\State break from the observation loop \Comment{Go to line $21$}
\EndIf  
\State $\_$, $v^{\prime}$, $\leftarrow$ \Call{PCSS}{$b^{\prime}$,$d-1$}
\State $v+=(\rho(b) + \gamma \cdot v^{\prime})/m_d$ \Comment{calculate value fun.}
\EndFor
\If {$  \mathrm{not}(\mathrm{pruned}) \wedge v > v^{*} $}
\State $(a^{*}, v^*$) $\leftarrow$ ($a, v$)
\EndIf 
\EndFor
\State \Return $(a^{*}, v^*$)
\EndProcedure
\end{algorithmic}
\label{alg:constrainedSS}
\end{algorithm}			

%% file: floats/ConstrainedSSbaselineImportance.tex
\begin{algorithm}[h!]
\caption{Chance Constrained BMDP  Sparse Importance Sampling (CCSS-IS)}
\begin{algorithmic}[1]
\Procedure{CCSS}{belief: $b$, belief: $\bar{b}$  depth: $d$, scale: sc}
\State $\varphi \leftarrow \phi(\bar{b})$ 
\If { $d=0$ }
	\State \Return	(Null, $\rho(b)$, $\varphi$)
\EndIf
\State $(a^{*}, v^*$) $\leftarrow$ (Null, $-\infty$)
\If {$\varphi == 0$}  \Comment{No point to continue, the action resulting in this belief shall be pruned}
\State {\bf return}  $(a^{*}, v^*$, $\varphi$) 
\EndIf
\State $\bar{b}^{\text{safe}} \leftarrow$ Make $\bar{b}$ safe \Comment{as in equation \eqref{eq:SafeBelief}}
\For  {$a \in  \mathcal{A}$}
\State $ v \leftarrow$ $0.0$ \Comment{Value function}
\State Calculate propagated belief $b'^{-}$ form $b$
\State Calculate propagated belief $\bar{b}'^{-}$ form $\bar{b}^{\text{safe}}$
\State $\Phi(ba) \leftarrow \{\}$  
\State $\mathrm{status} \leftarrow \mathrm{true}$
\For { $ j \in 1:m_d$}
\State Sample $x^o \sim b'^{-}$ 
\State Sample $z' \sim \probd(z|x^o)$ 
\State Calculate  $w^{\bar{z},\prime, j}$
\State $b^{\prime,j} \leftarrow \psi(b, a, z')$
\State $\bar{b}^{\prime,j} \leftarrow \psi(\bar{b}^{\text{safe}}, a, z')$
\EndFor 
\If {\Call{PRUNEACTIONCHANCE}{$\{ \bar{b}^{\prime,j} \}^{m_d}_{j=1}$, $\{  w^{\bar{z},\prime, j} \}^{m_d}_{j=1}$, $\varphi$, $\delta^{(d+1)} \cdot \mathrm{sc}+ \delta  (1-\mathrm{sc})$} } \Comment{See Alg.~\ref{alg:PruningChance}}
\State next action 
\EndIf
\For {\_ $\in 1:m_d$}
\State $\_$, $v^{\prime}$, $\varphi'_{\mathrm{exp}}$ $\leftarrow$ \Call{CCSS}{$b^{\prime}$, $\bar{b}^{\prime}$, $d-1$, sc}
\State $v+=(\rho(b) + \gamma \cdot v^{\prime})/m_d$
\State $\Phi(ba) \leftarrow \Phi(ba) \cup \varphi'_{\mathrm{exp}}$
\EndFor
\State $\varphi_{\mathrm{exp}} \leftarrow \frac{1}{\sum_{j=1}^{m_d} w^{\bar{z},\prime, j}}\underset{\varphi\prime_{\mathrm{exp}} \in \Phi(ba)}{\sum} w^{\bar{z},\prime}\varphi\prime_{\mathrm{exp}}$
\If {$\varphi \cdot \varphi_{\mathrm{exp}} < \delta^{(d+1)} \cdot \mathrm{sc}+ \delta  (1-\mathrm{sc})  $}
	\State status $\leftarrow$ false	
\EndIf	
\If {$  \mathrm{status}  \wedge u > v^{*} $}
\State $(a^{*}, v^*$) $\leftarrow$ ($a, v$)
\EndIf 
\EndFor
\State \Return	($a^{*}$, $v^*$, $\varphi \cdot \varphi_{\mathrm{exp}}$)
\EndProcedure
\end{algorithmic}
\label{alg:constrainedSSbaselineImportance}
\end{algorithm}			

%% file: floats/ConstrainedSSbaselineNoImportance.tex
\begin{algorithm}[h!]
\caption{Fast Chance Constrained BMDP  Sparse Sampling (FastCCSS)}
\begin{algorithmic}[1]
\Procedure{CCSS}{belief: $\bar{b}$  depth: $d$, scale: sc}
\State $\varphi \leftarrow \phi(\bar{b})$ 
\If { $d=0$ }
	\State \Return	(Null, $\rho(\bar{b})$, $\varphi$)
\EndIf
\State $(a^{*}, v^*$) $\leftarrow$ (Null, $-\infty$)
\If {$\varphi == 0$}
\State {\bf return}  $(a^{*}, v^*$, $\varphi$) 
\EndIf
\State $\bar{b}^{\text{safe}} \leftarrow$ Make $\bar{b}$ safe \Comment{as in equation \eqref{eq:SafeBelief}}
\For  {$a \in  \mathcal{A}$}
\State $ v \leftarrow$ $0.0$ \Comment{Value function}
\State Calculate propagated belief $\bar{b}'^{-}$ form $\bar{b}^{\text{safe}}$
\State $\Phi(ba) \leftarrow \{\}$  
\State $\mathrm{status} \leftarrow \mathrm{true}$
\For { $ j \in 1:m_d$}
\State Sample $x^o \sim \bar{b}'^{-}$ 
\State Sample $\bar{z}' \sim \probd(z|x^o)$ 
\State $\bar{b}^{\prime,j} \leftarrow \psi(\bar{b}^{\text{safe}}, a, \bar{z}')$
\EndFor 
\If {\Call{PRUNEACTIONCHANCE}{$\{ \bar{b}^{\prime,j} \}^{m_d}_{j=1}$, $\{ \frac{1}{m_d} \}^{m_d}_{j=1}$, $\varphi$, $\delta^{(d+1)} \cdot \mathrm{sc}+ \delta  (1-\mathrm{sc})$} } \Comment{See Alg.~\ref{alg:PruningChance}}
\State next action 
\EndIf
\For {\_ $\in 1:m_d$}
\State $\_$, $v^{\prime}$, $\varphi'_{\mathrm{exp}}$ $\leftarrow$ \Call{CCSS}{$b^{\prime}$, $\bar{b}^{\prime}$, $d-1$, sc}
\State $v+=(\rho(\bar{b}) + \gamma \cdot v^{\prime})/m_d$
\State $\Phi(ba) \leftarrow \Phi(ba) \cup \varphi'_{\mathrm{exp}}$
\EndFor
\State $\varphi_{\mathrm{exp}} \leftarrow \frac{1}{m_d}\underset{\varphi\prime_{\mathrm{exp}} \in \Phi(ba)}{\sum} \varphi\prime_{\mathrm{exp}}$
\If {$\varphi \cdot \varphi_{\mathrm{exp}} < \delta^{(d+1)} \cdot \mathrm{sc}+ \delta  (1-\mathrm{sc})  $}
	\State status $\leftarrow$ false	
\EndIf	
\If {$  \mathrm{status}  \wedge u > v^{*} $}
\State $(a^{*}, v^*$) $\leftarrow$ ($a, v$)
\EndIf 
\EndFor
\State \Return	($a^{*}$, $v^*$, $\varphi \cdot \varphi_{\mathrm{exp}}$)
\EndProcedure
\end{algorithmic}
\label{alg:constrainedSSbaseline}
\end{algorithm}			

%% file: floats/PruneChance.tex
\begin{algorithm}[h!]
\caption{Necessary condition for Feasibility of Chance Constraint}
\begin{algorithmic}[1]
\Procedure{PRUNEACTIONCHANCE}{$\{ \bar{b}^{\prime,j} \}^{m_d}_{j=1}$, $\{ w^{\bar{z},\prime, j} \}^{m_d}_{j=1}$, $\varphi$, $\delta$}
\State $\Delta \leftarrow 1- \delta $, $r_b \leftarrow 1 - \varphi$
\For{each $\bar{b}^{\prime,i}$}
\If {\eqref{eq:ChanceConstrPruning} is not met}
\State return true \Comment{prune}
\EndIf  
\EndFor
\State return false
\EndProcedure
\end{algorithmic}
\label{alg:PruningChance}
\end{algorithm}			

%% file: floats/tableAblationStaticGoalMyopicMap1.tex
\begin{table}[t]
	\caption{ $50$ Trials of $21$ planning sessions and executions of optimal action of PCSS versus FastCCSS in myopic setting. Same seed in both algorithms. Scaling is false in the FastCCSS. This problem is the {\bf navigation to static goal} problem in  our {\bf first map}.}
	\centering
	\resizebox{\textwidth}{!}{
	\begin{tabular}{|c|c|c|c|c|c|c|c|c|c|c|c|c|c|c|}
		\hline
		\multicolumn{3}{|c|}{ Parameters } & \multicolumn{2}{|c|}{ num collisions } &  \multicolumn{2}{|c|}{ mean cum. reward (V) $\pm$ std} & \multicolumn{2}{|c|}{ mean cum. reward (V) no coll $\pm$ std} &  \multicolumn{2}{|c|}{cum. plan. time [sec]} & speedup &\multicolumn{2}{|c|}{Total expanded actions} & actions frac.\\
		\hline 
		 $m_x$ & $m_1$ & $\delta$ & PCSS  & FastCCSS & PCSS &  FastCCSS  & PCSS & FastCCSS   & PCSS & FastCCSS &  &PCSS & FastCCSS  &\\
		\hline 
		$150$ & $100$ & $0.9$ &  ${\bf 2}/50$  & ${\bf 19}/50$   & $  -192\pm 30.07$ & $ -176.26\pm 29.51$    &$ -192.01 \pm 28.85$ & $ -177.20\pm 31.13$ &  $80.24$& $105.78$ &${\bf 0.24}$& $6559$ & $6986$& $0.06$\\
		\hline 
		 $150$ & $100$ & $0.85$ &  ${\bf 3}/50$  & ${\bf 28}/50$   & $ -196.05 \pm 30.44$ & $-170.51 \pm 27.58 $    &$ -195.08\pm 30.47$ & $ -166.16\pm 25.25$ & $80.29$ & $102.43$ &${ \bf 0.21}$&$6431$& $7045$&$0.087$\\
		\hline 
		 $150$ & $100$ & $0.8$ &  ${\bf 5}/50$  & ${\bf 25}/50$   & $ -186.62 \pm 31.27$ & $  -166.55\pm  30.58$    &$ -86.47\pm 32.09$ & $ -159.98\pm 24.62$ & $80.09$& $103.91$ &${ \bf 0.23}$&$6444$& $7245$&$0.11$\\
		\hline 
		$150$ & $100$ & $0.75$ &  ${\bf 6}/50$  & ${\bf 39}/50$   & $ -188.22 \pm 29.95$ & $ -176.25\pm 34.97$    &$ -189.42\pm 29.76$ & $-169.74\pm 33.93$ & $89.34$ & $115.85$ &${\bf 0.23}$&$6569$& $7077$&$0.07$\\
		\hline
		 $150$ & $100$ & $0.7$ &  ${\bf 7}/50$  & ${\bf 39}/50$   & $ -176.87 \pm 27.81$ & $-172.03 \pm 36.82$    &$ -175.85\pm 27.33$ & $ -161.62\pm 34.66$ & $78.99$ & $103.75$  &${\bf 0.23}$&$6626$& $7182$&$0.08$\\
		\hline 
	\end{tabular}
}
	\label{tbl:AblationStudyMyopicStaticGoal} 
\end{table}

%% file: floats/tableAblationStaticGoalL2Map1.tex
\begin{table}[t]
	\caption{ $50$ Trials of $21$ planning sessions and executions of optimal action of PCSS versus CCSS-IS Alg.~\ref{alg:constrainedSSbaselineImportance} and FastCCSS Alg.~\ref{alg:constrainedSSbaseline}. Same seed in three  algorithms. This problem is the {\bf navigation to static goal} problem in  { \bf our first map} $L=2$. Here we study the number of collisions and the reward value.}
	\centering
	\resizebox{\textwidth}{!}{
	\begin{tabular}{|c|c|c|c|c|c|c|c|c|c|c|c|c|c|}
		\hline
		\multicolumn{5}{|c|}{ Parameters }  &\multicolumn{3}{|c|}{ num collisions } &  \multicolumn{3}{|c|}{ mean cum. rew. $\pm$ std} & \multicolumn{3}{|c|}{ mean cum. rew.  no coll $\pm$ std} \\
		\hline
		$m_x$ & $m_1$ & $m_2$  & $\delta$ & sc & PCSS  & FastCCSS & CCSS-IS & PCSS &  FastCCSS  & CCSS-IS & PCSS & FastCCSS   & CCSS-IS  \\
		\hline  
		$100$ & $10$ & $10$  & $0.9$ & $0$ & \multirow{2}{*}{${\bf 5}/50$}  & ${\bf 7 }/50$   & ${\bf 7}/50$ & \multirow{2}{*}{$ -156.43\pm 20.26$ }   & $ -155.30 \pm 16.94  $ & $ -154.90 \pm 15.85 $  & \multirow{2}{*}{$-157.68 \pm  20.86$} & $ -155.74 \pm 17.94$& $   -155.71 \pm 16.56$  \\
		\cline{1-5} \cline{7-8} \cline{10-11}\cline{13-14}
		$100$ & $10$ & $10$  & $0.9$ & $1$ &   & ${\bf 7}/50$   & ${\bf 12}/50$ &     & $ -153.70\pm 16.29$ &$-151.77 \pm 17.33$  & $  $ &$ -152.22\pm 14.16$ & $-152.32\pm 16.49$ \\
		\hline 
		$100$ & $10$ & $10$ & $0.8$ & $0$ & \multirow{2}{*}{${\bf 7}/50$}  & ${\bf 5 }/50$   & ${\bf 9}/50$ & \multirow{2}{*}{$ -149.68\pm 15.03$ }   & $ -153.18\pm 15.13$ & $ -151.44\pm 17.56$  & \multirow{2}{*}{$ -150.77\pm 14.20$} & $  -153.60\pm 15.70$& $  -151.97 \pm  17.36  $  \\
		\cline{1-5} \cline{7-8} \cline{10-11}\cline{13-14}
		$100$ & $10$ & $10$  & $0.8$ & $1$ &   & ${\bf 11}/50$   & ${\bf  13}/50$ &     &$ -149.32 \pm 17.64$ & $ 149.84\pm  17.28$ &   &$  -152.11\pm 17.77 $ & $ -150.03 \pm 17.36 $ \\
		\hline 
		$100$ & $10$ & $10$ & $0.7$ & $0$ & \multirow{2}{*}{${\bf 12}/50$}  & ${\bf  13 }/50$   & ${\bf 11}/50$ & \multirow{2}{*}{$  -148.24 \pm 14.60$ }   & $ -148.85\pm 17.69$ & $ -152.40\pm 14.97 $  & \multirow{2}{*}{$  -148.01\pm 15.70$} & $ -145.57 \pm 14.07$& $ -153.67\pm   15.21$  \\
		\cline{1-5} \cline{7-8} \cline{10-11}\cline{13-14}
		$100$ & $10$ & $10$ &  $0.7$ & $1$ &   & ${\bf 24 }/50$   & ${\bf 20 }/50$ &     &$ -147.86\pm 15.52$ & $-140.33 \pm 22.07$ & $ $ &$ -148.06\pm  14.04$ & $-147.20\pm 24.59$ \\
		\hline 
	\end{tabular}
}
	\label{tbl:AblationStudyL2StaticGoal} 
\end{table}

%% file: floats/tableAblationStaticGoalL2Map1SpeedupActions.tex
\begin{table}[t]
	\caption{ $50$ Trials of $21$ planning sessions and executions of optimal action of PCSS Alg.~\ref{alg:constrainedSS} versus CCSS-IS Alg.~\ref{alg:constrainedSSbaselineImportance} and FastCCSS Alg.~\ref{alg:constrainedSSbaseline}. Same seed in three  algorithms. This problem is the { \bf navigation to static goal} problem in  our { \bf first map} $L=2$. In this table we study speedup \eqref{eq:speedup} and the saved actions fraction \eqref{eq:actions}}
	\centering
	\resizebox{\textwidth}{!}{
	\begin{tabular}{|c|c|c|c|c|c|c|c|c|c|c|c|c|c|}
		\hline
		\multicolumn{5}{|c|}{ Parameters }   &  \multicolumn{3}{|c|}{cum. plan. time [sec]} & speedup Alg.~\ref{alg:constrainedSS} rel to \ref{alg:constrainedSSbaselineImportance}  &speedup Alg.~\ref{alg:constrainedSSbaseline} rel to \ref{alg:constrainedSSbaselineImportance}  &\multicolumn{3}{|c|}{Total expanded actions} & actions frac. Alg.~\ref{alg:constrainedSS} rel to \ref{alg:constrainedSSbaselineImportance}\\
		\hline
		$m_x$ & $m_1$ & $m_2$  & $\delta$ & sc & PCSS \ref{alg:constrainedSS} & FastCCSS \ref{alg:constrainedSSbaseline}& CCSS-IS \ref{alg:constrainedSSbaselineImportance}& & & PCSS \ref{alg:constrainedSS} & FastCCSS \ref{alg:constrainedSSbaseline}& CCSS-IS \ref{alg:constrainedSSbaselineImportance}&   \\
		\hline  
		$100$ & $10$ & $10$ &   $0.9$ & $0$ & \multirow{2}{*}{$359.89$}  & $444.87$   &  $1898.91$ & ${\bf 0.81}$&$ {\bf 0.77}$&\multirow{2}{*}{$602389$}&$645943$ &$649736$& $0.07$ \\
		\cline{1-5} \cline{7-8} \cline{9-10}\cline{12-14}
		$100$ & $10$ & $10$  & $0.9$ & $1$ &  & $518.86$& $1976.01$ &${\bf 0.82}$&${\bf 0.74}$& &$676711$&$671214$&$0.10$\\
		\hline 
		$100$ & $10$ & $10$  & $0.8$ & $0$ & \multirow{2}{*}{$393.12$}  & $490.63$   &  $ 1821.17$ & ${\bf 0.78}$&$ {\bf 0.73}$&\multirow{2}{*}{$618153 $}&$668893$ &$669365$& $0.08$\\
		\cline{1-5} \cline{7-8} \cline{9-10}\cline{12-14}
		$100$ & $10$ & $10$  & $0.8$ & $1$ &  & $519.08$&$2037.61$ &${\bf 0.81}$&${\bf 0.75 }$& &$695439$&$695758$&$0.11$\\
		\hline 
		$100$ & $10$ & $10$  & $0.7$ & $0$ & \multirow{2}{*}{$428.35$}  & $500.62$   &  $1833.90$ & ${\bf 0.77}$&$ {\bf 0.73}$&\multirow{2}{*}{$ 622865$}&$686610$ & $677433$&$0.08$\\
		\cline{1-5} \cline{7-8} \cline{9-10}\cline{12-14}
		$100$ & $10$ & $10$ & $0.7$ & $1$ &  & $532.14$ & $1914.24$ &${\bf 0.78}$&${\bf 0.72 }$& &$705501$ & $711430$&$0.12$\\
		\hline 
	\end{tabular}
}
	\label{tbl:AblationStudyL2SpeedupActionsStaticGoal} 
\end{table}

%% file: floats/tableAblationTargetTrackingMyopicMap2.tex
\begin{table}[t]
	\caption{ $50$ Trials of $21$ planning sessions and executions of optimal action of PCSS versus FastCCSS in myopic setting. Same seed in both algorithms. Scaling is false in the FastCCSS. This problem is the {\bf target tracking}  problem in  our { \bf second map}.}
	\centering
	\resizebox{\textwidth}{!}{
	\begin{tabular}{|c|c|c|c|c|c|c|c|c|c|c|c|c|c|c|}
		\hline
		\multicolumn{3}{|c|}{ Parameters } & \multicolumn{2}{|c|}{ num collisions } &  \multicolumn{2}{|c|}{ mean cum. reward (V) $\pm$ std} & \multicolumn{2}{|c|}{ mean cum. reward (V) no coll $\pm$ std} &  \multicolumn{2}{|c|}{cum. plan. time [sec]} & speedup &\multicolumn{2}{|c|}{Total expanded actions} & actions frac.\\
		\hline 
		 $m_x$ & $m_1$ & $\delta$ & PCSS  & FastCCSS & PCSS &  FastCCSS  & PCSS & FastCCSS   & PCSS & FastCCSS &  &PCSS & FastCCSS  &\\
		\hline 
		$150$ & $100$ & $0.9$ &  ${\bf 3}/50$  & ${\bf 23}/50$   & $   -108.94\pm 18.31$ & $ -90.17\pm 17.44$    &$  -108.54\pm 18.40$ & $ -89.82\pm 18.03$ &  $ 201.45$& $260.66$ &${\bf 0.22}$& $6360$ & $6944$& $0.08$\\
		\hline 
		 $150$ & $100$ & $0.85$ &  ${\bf 3}/50$  & ${\bf 28}/50$   & $-113.19\pm 17.76 $ & $ -85.29 \pm 15.99$    &$ -112.96\pm 18.13$ & $-83.49 \pm 17.84$ & $203.68$ & $268.47$ &${ \bf 0.24 }$&$6378$& $7010$&$0.09$\\
		\hline 
		 $150$ & $100$ & $0.8$ &  ${\bf 5}/50$  & ${\bf 37}/50$   & $ -107.78 \pm 18.74$ & $ -84.24\pm  19.85$     & $ -108.10\pm 19.08$&$ -93.11\pm 25.36$ & $202.88$& $262.41$ &${ \bf 0.22}$&$6432$& $7149$&$0.11$\\
		\hline 
		$150$ & $100$ & $0.75$ &  ${\bf 5}/50$  & ${\bf 41}/50$   & $ -109.28 \pm 20.37$ & $ -86.45 \pm 21.21$    &$ -109.53\pm 20.23$ & $-87.22\pm 18.08$ & $205.52$ & $269.91$ &${\bf 0.23}$&$6374$& $7143$&$0.11$\\
		\hline
		 $150$ & $100$ & $0.7$ &  ${\bf 5}/50$  & ${\bf 43}/50$   & $ -105.28 \pm 17.99$ & $-79.19 \pm17.40 $    &$ -106.59\pm 17.17$ & $ -85.05\pm 20.72$ & $214.00$ & $266.42$  &${\bf 0.19}$&$6546$& $7349$&$0.11$\\
		\hline 
	\end{tabular}
}
	\label{tbl:AblationStudyMyopicTargetTracking} 
\end{table}

%% file: floats/tableAblationTargetTrackingL2Map2.tex
\begin{table}[t]
	\caption{ $50$ Trials of $21$ planning sessions and executions of optimal action of PCSS versus CCSS-IS Alg.~\ref{alg:constrainedSSbaselineImportance} and FastCCSS Alg.~\ref{alg:constrainedSSbaseline}. Same seed in three  algorithms. This problem is the {\bf target tracking} in  our { \bf second map} $L=2$. Here we study the number of collisions and the reward value.}
	\centering
	\resizebox{\textwidth}{!}{
	\begin{tabular}{|c|c|c|c|c|c|c|c|c|c|c|c|c|c|}
		\hline
		\multicolumn{5}{|c|}{ Parameters }  &\multicolumn{3}{|c|}{ num collisions } &  \multicolumn{3}{|c|}{ mean cum. rew. $\pm$ std} & \multicolumn{3}{|c|}{ mean cum. rew.  no coll $\pm$ std} \\
		\hline
		$m_x$ & $m_1$ & $m_2$  & $\delta$ & sc & PCSS  & FastCCSS & CCSS-IS & PCSS &  FastCCSS  & CCSS-IS & PCSS & FastCCSS   & CCSS-IS  \\
		\hline  
		$100$ & $10$ & $10$  & $0.9$ & $0$ & \multirow{2}{*}{${\bf 14}/50$}  & ${\bf  6}/50$   & ${\bf 13}/50$ & \multirow{2}{*}{$ -98.99\pm  18.37$ }   & $ -100.91\pm  16.87$ & $ -100.14\pm 17.88$  & \multirow{2}{*}{$ -97.59\pm  16.85$} & $   -99.86\pm 16.49 $& $ -101.48 \pm 17.27$  \\
		\cline{1-5} \cline{7-8} \cline{10-11}\cline{13-14}
		$100$ & $10$ & $10$ &  $0.9$ & $1$ &   & ${\bf 17}/50$   & ${\bf 13}/50$ &     &$ -99.88 \pm  19.25 $ & $  -98.03\pm 17.61$ & $  $ &$  -103.65 \pm 17.34 $ & $-99.67\pm 18.14 $ \\
		\hline 
		$100$ & $10$ & $10$ &  $0.8$ & $0$ & \multirow{2}{*}{${\bf 14}/50$}  & ${\bf 10 }/50$   & ${\bf 13}/50$ & \multirow{2}{*}{$  -95.86 \pm  19.4 $ }   & $-96.19 \pm 17.03$ & $  -100.22\pm 18.75$  & \multirow{2}{*}{$  -96.88 \pm    19.14 $} & $ -95.64 \pm 17.48$& $-101.81\pm 19.08$  \\
		\cline{1-5} \cline{7-8} \cline{10-11}\cline{13-14}
		$100$ & $10$  & $10$ & $0.8$ & $1$ &   & ${\bf 25}/50$   & ${\bf 21 }/50$ &     &$  -91.39 \pm 17.27$ & $ -95.26 \pm  15.91$ &   &$ -91.96\pm 18.12$ & $  -97.36\pm  16.46$ \\
		\hline 
		$100$ & $10$ & $10$ &  $0.7$ & $0$ & \multirow{2}{*}{${\bf 20}/50$}  & ${\bf   23}/50$   & ${\bf 22}/50$ & \multirow{2}{*}{$ -92.93\pm 17.40$ }   & $-96.13 \pm 19.55 $ & $ -95.00\pm 17.31 $  & \multirow{2}{*}{$ -91.46 \pm 18.17$} & $    -98.77\pm 19.22 $& $ -96.20\pm 16.26 $  \\
		\cline{1-5} \cline{7-8} \cline{10-11}\cline{13-14}
		$100$ & $10$ & $10$  & $0.7$ & $1$ &   & ${\bf 32}/50$   & ${\bf  36}/50$ &     &$ -87.00\pm16.48 $ & $ -90.31 \pm  16.87$ & $ $ &$ -84.65\pm 15.72 $ & $-90.49\pm 16.27 $  \\
		\hline 
	\end{tabular}
}
	\label{tbl:AblationStudyL2TargetTracking} 
\end{table}

%% file: floats/tableAblationTargetTrackingL2Map2SpeedupActions.tex
\begin{table}[t]
	\caption{ $50$ Trials of $21$ planning sessions and executions of optimal action of PCSS Alg.~\ref{alg:constrainedSS} versus CCSS-IS Alg.~\ref{alg:constrainedSSbaselineImportance} and FastCCSS Alg.~\ref{alg:constrainedSSbaseline}. Same seed in three  algorithms. This problem is the {\bf target tracking} problem in  our { \bf second map} $L=2$. In this table we study speedup \eqref{eq:speedup} and the saved actions fraction \eqref{eq:actions}}
	\centering
	\resizebox{\textwidth}{!}{
	\begin{tabular}{|c|c|c|c|c|c|c|c|c|c|c|c|c|c|c|}
		\hline
		\multicolumn{5}{|c|}{ Parameters }   &  \multicolumn{3}{|c|}{cum. plan. time [sec]} & speedup Alg.~\ref{alg:constrainedSS} rel to \ref{alg:constrainedSSbaselineImportance}  &speedup Alg.~\ref{alg:constrainedSSbaseline} rel to \ref{alg:constrainedSSbaselineImportance}  &\multicolumn{3}{|c|}{Total expanded actions} & actions frac. Alg.~\ref{alg:constrainedSS} rel to \ref{alg:constrainedSSbaselineImportance}\\
		\hline
		$m_x$ & $m_1$ & $m_2$  & $\delta$ & sc & PCSS \ref{alg:constrainedSS} & FastCCSS \ref{alg:constrainedSSbaseline}& CCSS-IS \ref{alg:constrainedSSbaselineImportance}& & & PCSS \ref{alg:constrainedSS} & FastCCSS \ref{alg:constrainedSSbaseline}& CCSS-IS \ref{alg:constrainedSSbaselineImportance}&   \\
		\hline  
		$100$ & $10$ & $10$  & $0.9$ & $0$ & \multirow{2}{*}{$920.62$}  & $ 1110.45 $   &  $3450.07$ & ${\bf 0.73}$&$ {\bf 0.68 }$&\multirow{2}{*}{$534242$}&$588780$ &$590291$& $0.09$\\
		\cline{1-5} \cline{7-10} \cline{12-14}
		$100$ & $10$ & $10$  & $0.9$ & $1$ &  & $ 1183.08$& $3591.56$&${\bf 0.74}$&${\bf 0.67}$& &$624208$&$615302$&$0.13$\\
		\hline 
		$100$ & $10$ & $10$ &  $0.8$ & $0$ & \multirow{2}{*}{$ 944.38$}  & $1079.37$   &  $3544.20$ & ${\bf 0.73}$&$ {\bf 0.70 }$&\multirow{2}{*}{$539697$}&$612064$ &$613031$& $0.12$\\
		\cline{1-5} \cline{7-10} \cline{12-14}
		$100$ & $10$ & $10$  & $0.8$ & $1$ &  &$ 1153.06 $& $4287.72$ &${\bf 0.78}$&${\bf  0.73}$& &$643762$&$646185$&$0.16$\\
		\hline 
		$100$ & $10$ & $10$  & $0.7$ & $0$ & \multirow{2}{*}{$906.62$}    &  $1361.23$ & $4090.22$ & ${\bf 0.78}$&$ {\bf 0.67}$&\multirow{2}{*}{$563361$}&$627096$ &$627979$& $0.10$\\
		\cline{1-5} \cline{7-10} \cline{12-14}
		$100$ & $10$ & $10$  & $0.7$ & $1$ &  & $ 1456.03$ & $4310.20$ &${\bf 0.79 }$&${\bf  0.66}$& &$659623$ & $656564$&$0.14$\\
		\hline 
	\end{tabular}
}
	\label{tbl:AblationStudyL2SpeedupActionsTargetTracking} 
\end{table}